\documentclass[a4paper,11pt]{article}
\usepackage[margin=1in]{geometry}
\usepackage{import}
\usepackage{amsmath}
\usepackage{mathtools}
\usepackage{dsfont}
\usepackage[utf8]{inputenc} 
\usepackage{csquotes}

\makeatletter
\newcommand{\printfnsymbol}[1]{%
  \textsuperscript{\@fnsymbol{#1}}%
}
\makeatother

\usepackage[utf8]{inputenc} 
\usepackage[T1]{fontenc}    
\usepackage{hyperref}       
\usepackage{url}            
\usepackage{nicefrac}       
\usepackage{microtype}      
\usepackage{cite}
\usepackage{xcolor}
\usepackage{framed}
\colorlet{shadecolor}{pink}
\usepackage{authblk}
\usepackage{adjustbox}
\usepackage{bbm}

\usepackage{graphicx}
\usepackage{soul}
\usepackage{subcaption}
\usepackage{booktabs} 
\usepackage{tablefootnote}

\usepackage{amsmath,amsthm,amssymb,amsfonts}
\usepackage{algorithm}
\usepackage{algorithmic}
\usepackage{enumerate}
\usepackage{cleveref}
\usepackage{comment}
\usepackage{bm}
\usepackage{pifont}
\newcommand{\cmark}{\ding{51}}%
\newcommand{\xmark}{\ding{55}}%

\theoremstyle{plain}
\newtheorem{theorem}{Theorem}[]

\newtheorem{lemma}[]{Lemma}

\newtheorem{assumption}{Assumption}[]

\theoremstyle{definition}
\newtheorem{remark}{Remark}[]

\newcommand{\ra}[1]{\renewcommand{\arraystretch}{#1}}

\def\x{{\mathbf x}}

\def\O{{\mathcal O}}

\begin{document}
\title{\bf Faster Non-Convex Federated Learning via Global and Local Momentum}
\date{}
\author[$\bm{\dag}$]{Rudrajit Das}
\author[$\bm{\dag}$]{Anish Acharya \thanks{Equal Contribution}}
\author[$\bm{\ddag}$]{Abolfazl Hashemi \printfnsymbol{1}}
\author[$\bm{\dag}$]{Sujay Sanghavi} \author[$\bm{\dag}$,$\bm{\mathsection}$]{Inderjit S. Dhillon}
\author[$\bm{\dag}$]{Ufuk Topcu}
\affil[$\bm{\dag}$]{University of Texas at Austin}
\affil[$\bm{\ddag}$]{Purdue University}
\affil[$\bm{\mathsection}$]{Amazon}
\maketitle

\begin{abstract}
  We propose \texttt{FedGLOMO}, a novel federated learning (FL) algorithm with an iteration complexity of $\mathcal{O}(\epsilon^{-1.5})$ to converge to an $\epsilon$-stationary point (i.e., $\mathbb{E}[\|\nabla f(\bm{x})\|^2] \leq \epsilon$) for smooth non-convex functions --
  under arbitrary client heterogeneity and compressed communication -- compared to the $\mathcal{O}(\epsilon^{-2})$ complexity of most prior works. Our key algorithmic idea that enables achieving this improved complexity is based on the observation that the convergence in FL is hampered by two sources of high variance: (i) the global server aggregation step with multiple local updates, exacerbated by client heterogeneity, and (ii) the noise of the local client-level stochastic gradients. By modeling the server aggregation step as a generalized gradient-type update, we propose a variance-reducing momentum-based global update at the server, which when applied in conjunction with variance-reduced local updates at the clients, enables \texttt{FedGLOMO} to enjoy an improved convergence rate. Moreover, we derive our results under a novel and more realistic client-heterogeneity assumption which we verify empirically -- unlike prior assumptions that are hard to verify. Our experiments illustrate the intrinsic variance reduction effect of \texttt{FedGLOMO}, which implicitly suppresses client-drift in heterogeneous data distribution settings and promotes communication efficiency.
\end{abstract}

\section{Introduction}
\label{sec:intro}
Federated learning (FL) is a new edge-computing approach that advocates training statistical models directly on remote devices by leveraging enhanced local resources on each device \cite{mcmahan2017communication}. In a standard FL setting, there are $n$ clients, each having its own training data, and a central server that is trying to train a model, parameterized by $\bm{w} \in \mathbb{R}^d$, using the clients' data.
Suppose the data distribution of the $i^{\text{th}}$ client is $\mathcal{D}_i$.
Then the $i^{\text{th}}$ client has an objective function $f_i(\bm{w})$ which is the expected loss, with respect to some loss function $\ell$, over data drawn from $\mathcal{D}_i$, and the goal of the central server is to optimize the average \footnote{In general this may be a weighted average, but here we only consider uniform weights, i.e., each weight is ${1}/{n}$.} 
loss $f(\bm{w})$, over the $n$ clients, i.e.,
\begin{equation}
    \label{eq:fl-intro-1}
    f(\bm{w}) := \frac{1}{n} \sum_{i=1}^{n} {f_i}(\bm{w}) \text{ \& } f_i(\bm{w}) = 
    \mathbb{E}_{\bm{x} \sim \mathcal{D}_i}[\ell(\bm{x},\bm{w})].
\end{equation}
{The setting where the data distributions of all the clients are identical, i.e. $\mathcal{D}_1 = \ldots = \mathcal{D}_n$, is typically known as the \enquote{homogeneous} setting. Otherwise, the settings where the data distributions are \textit{not} identical are referred to as the  \enquote{heterogeneous} settings.}

The core algorithmic idea of FL -- in the form of \texttt{FedAvg} -- was introduced in \cite{mcmahan2017communication}. In \texttt{FedAvg} (summarized in \Cref{alg:fed-avg}), a subset of the clients perform \textit{multiple} steps of gradient descent based updates on their local data and then communicate back their respective updates to the server, which then averages them to update the global model (hence the name \texttt{FedAvg}). This idea of performing multiple local updates before averaging once mitigates the communication cost required for training. Another strategy to cut down the communication cost is to have the clients send compressed/quantized messages to the server in every round -- this is of particular significance for training deep learning models where the number of model parameters is in millions or more.

In practice however, performing multiple local updates on clients with \textit{heterogeneous} data distributions leads to the so-called phenomenon of \enquote{{client-drift}}, wherein the individual client updates do not align well (due to over-fitting on the local client data) inhibiting the convergence of \texttt{FedAvg} to the optimum of the average loss over all the clients. At the heart of this issue is the high variance associated with the simple averaging step of \texttt{FedAvg} for the global update.

Ever since the development of FL, significant attention has been devoted to analyzing \texttt{FedAvg} under different settings, modifying \texttt{FedAvg} using ideas from centralized optimization to accelerate the training or to reduce the communication cost; we discuss these works in \Cref{rel-work}. Compared to centralized optimization, a formidable challenge in the theoretical analysis of FL algorithms is the use of multiple local updates in the clients which is compounded by the \textit{heterogeneous} nature of data distribution among the clients. To limit the extent of client heterogeneity, a standard
assumption in FL theory is the \textit{bounded client dissimilarity (BCD) assumption}, i.e.,
\begin{equation}
    \label{eq:bcd}
    \mathbb{E}_{i}[\|\nabla f_i(\bm{w}) - \nabla f(\bm{w})\|^2] \leq G^2 \text{ } \forall \text{ } \bm{w}, \text{ or } 
    \|\nabla f_i(\bm{w}) - \nabla f(\bm{w})\|^2 \leq G^2 \text{ } \forall \text{ } \bm{w} \text{ and } i \in [n],
\end{equation}
for some large enough constant $G$ (e.g., see A1 in \cite{karimireddy2020mime}). However, it is hard to verify in practice and does not allow for \textit{arbitrary client heterogeneity}.

Recently, \cite{arjevani2019lower} showed that the {stochastic first-order complexity} of any
algorithm in the \textit{centralized setting} to reach an $\epsilon$-stationary point (i.e., $\mathbb{E}[\|\nabla f(\bm{x})\|^2] \leq \epsilon$) for \textit{smooth non-convex functions} is 
$\Omega(\epsilon^{-1.5})$. It is well known that vanilla SGD has a suboptimal complexity of $\mathcal{O}(\epsilon^{-2})$ as it cannot mitigate the high variance of the stochastic gradient noise. Recognizing this issue, \textit{variance-reducing} techniques for SGD \cite{fang2018spider,zhou2018stochastic,cutkosky2019momentum,liu2020optimal} have been proposed that attain the optimal complexity of $\mathcal{O}(\epsilon^{-1.5})$. 
Coming to the federated setting, in addition to the noise in the \textit{local} client-level stochastic gradients, one has to also contend with the high variance associated with the \textit{global} server aggregation step which depends on the 
client heterogeneity and the number of local update steps. In this case, applying only local client-level variance-reduction is not enough for improving the iteration complexity of vanilla \texttt{FedAvg}.

To that end, we propose a novel FL algorithm with \textit{compressed communication} called \texttt{FedGLOMO} (\Cref{alg:2} and \ref{alg:2-local}) which applies \texttt{G}\textit{lobal} as well as \texttt{LO}\textit{cal} \textit{variance-reducing} \texttt{MO}\textit{mentum} to the server update and client updates, respectively.
We prove that the iteration complexity of \texttt{FedGLOMO} is $\mathcal{O}(\epsilon^{-1.5})$ in the smooth non-convex case, which is better than the $\mathcal{O}(\epsilon^{-2})$ complexity of related works in the FL setting; see \Cref{tb:comp} and \Cref{nov4-thm1}.
Further, our theory does not use the BCD assumption, i.e. \cref{eq:bcd}, which is a standard assumption in related works. Instead, we propose and use \Cref{as-het}, which is more realistic and \textit{empirically verified}, allowing for arbitrary client heterogeneity. 
It is worth mentioning here that for FL, \cite{karimireddy2020mime} also propose an algorithm (\texttt{MimeMVR}) which is shown to attain this improved complexity of $\mathcal{O}(\epsilon^{-1.5})$ but \textit{with} the BCD assumption and \textit{no} compressed communication; we talk about this at the end of \Cref{rel-work}.
\\
\\
\noindent We summarize our \textbf{contributions} next:
\\
\\
\noindent \textbf{(a)} We propose \texttt{FedGLOMO} (Alg. \ref{alg:2} and \ref{alg:2-local}), in which we apply a \textit{novel global momentum term at the server} in addition to SVRG-style \textit{local momentum at the clients}. The design of \texttt{FedGLOMO} is motivated by two critical issues that need to be alleviated to accelerate convergence in FL; these are the high variances associated with: (i) the \textit{global} server aggregation step due to heterogeneity of clients when there are multiple local updates, and (ii) the noise of \textit{local} client-level stochastic gradients. Global and local momentum result in \textit{variance reduction} for the global server update and the local client updates, allowing us to tackle (i) and (ii), respectively. This enables \texttt{FedGLOMO} to converge to an $\epsilon$-stationary point (i.e., $\mathbb{E}[\|\nabla f(\bm{x})\|^2] \leq \epsilon$) for smooth non-convex functions in $\mathcal{O}(\epsilon^{-1.5})$ gradient-based updates, which is better than the $\mathcal{O}(\epsilon^{-2})$ complexity of most related works in the FL setting; see \Cref{tb:comp} and \Cref{nov4-thm1}.
\\
\\
\noindent \noindent \textbf{(b)} Unlike prior work, our theory does not use the hard to verify {bounded client dissimilarity assumption} (i.e., \cref{eq:bcd}). Instead, to tighten our convergence result, we propose and use \Cref{as-het} -- which is a novel and {empirically verified} assumption, even allowing for \textit{arbitrary client heterogeneity}. Moreover, this assumption is not specific to \texttt{FedGLOMO}; we empirically verify that it also holds for \texttt{FedAvg} and derive a novel convergence result for \texttt{FedAvg} using this assumption and without \cref{eq:bcd} (see \Cref{sec:fed_avg_conv}).Refer to the discussion after \Cref{as-het} and \Cref{rem-sep21-2} for details. 
\\
\\
\noindent \textbf{(c)} Further, \texttt{FedGLOMO} is the \textit{first FL algorithm} achieving $\mathcal{O}(\epsilon^{-1.5})$ complexity while allowing \textit{compressed client-to-server communication}. For theory, applying compression in \texttt{FedGLOMO} is not trivial and the most obvious approach to do so does not work; see \Cref{rem-sep21-3}. 
\\
\\
\textbf{(d)} In \Cref{sec:exp}, experiments with neural networks on CIFAR-10 and Fashion-MNIST \cite{xiao2017fashion} show that in a highly heterogeneous setting of at most two classes per client, \texttt{FedGLOMO} requires only about one-third the number of bits used by \texttt{FedAvg} with compressed communication, while this ratio further improves to one-fifth in the homogeneous setting; see \Cref{fig:1}.Our experiments also illustrate the variance reduction provided by our scheme which implicitly mitigates client-drift under heterogeneous data distribution and promotes communication-efficiency.

\section{Related Work}
\label{rel-work}
\textbf{\texttt{FedAvg} and related methods:}
\cite{reisizadeh2020fedpaq} propose \texttt{FedPAQ} which is basically \texttt{FedAvg} \cite{mcmahan2017communication} with quantized client-to-server communication, and establish its convergence for the homogeneous case. \cite{li2019convergence} establish the convergence of \texttt{FedAvg} for strongly convex functions with heterogeneity (assuming bounded client dissimilarity) but without any compressed communication. \cite{haddadpour2020federated} propose \texttt{FedCOMGATE} which incorporates gradient tracking \cite{pu2020distributed} and derive results with data heterogeneity and quantized communication. \cite{karimireddy2019scaffold} propose \texttt{SCAFFOLD} which uses control-variates to mitigate the client-drift owing to the heterogeneity of clients. \cite{li2018federated} present \texttt{FedProx} which adds a proximal term to control the deviation of the client parameters from the global server parameter in the previous round. \cite{reddi2020adaptive} propose federated versions of commonly used adaptive optimization methods and prove their convergence under heterogeneity. Local SGD \cite{zinkevich2010parallelized,stich2018local,yu2018parallel,wang2018cooperative,basu2019qsparse,stich2019error,patel2019communication,woodworth2020local,bayoumi2020tighter,liang2019variance,koloskova2020unified} is very similar to FL and is essentially based on the same principle as \texttt{FedAvg}. However, in local SGD, there is usually no data heterogeneity and all the clients participate in each round (known as \enquote{full device participation}), both of which do not hold in FL and simplify the derivation of convergence results.
\\
\\
\noindent \textbf{Momentum-based methods in FL:} 
\cite{wang2019slowmo, huo2020faster} present momentum-based updates at the server but without any improvement in the 
convergence rate as compared to momentum-free updates. \cite{qu2020federated} present Nesterov accelerated \texttt{FedAvg} for convex objectives. \cite{karimireddy2020mime} propose \texttt{Mime}(\texttt{MVR}) which applies momentum at the client-level based on globally computed statistics to control client-drift.
\\
\\
\noindent \noindent \textbf{Distributed optimization with compression:} There are several papers \cite{alistarh2017qsgd,suresh2017distributed,reisizadeh2020fedpaq,haddadpour2020federated,tang2018communication,wu2018error,bernstein2018signsgd,alistarh2018convergence,lin2017deep,stich2018sparsified,basu2019qsparse,hashemi2020delicoco,chen2020communication,horvath2019stochastic,gorbunov2021marina} aiming to minimize the communication bottleneck in distributed optimization by transmitting compressed messages to the central server and establishing their convergence. \cite{horvath2019stochastic,gorbunov2021marina} provide distributed algorithms with improved convergence rates by also applying variance reduction and periodically using full gradients; however, there are no multiple local updates in these works. In \Cref{sec:marina}, we compare our work's complexity against that of \cite{gorbunov2021marina}. In this work, we employ the quantization operator proposed in \cite{alistarh2017qsgd}.
\\
\\
\noindent \textbf{Optimal complexity/rate for smooth non-convex stochastic optimization:}  
\cite{arjevani2019lower} show that the optimal stochastic first-order complexity to reach an $\epsilon$-stationary point (i.e., $\mathbb{E}[\|\nabla f(\bm{x})\|^2] \leq \epsilon$) is $\mathcal{O}(\frac{\sigma}{\epsilon^{1.5}})$ where $\sigma^2$ is the variance of the stochastic gradients. SVRG-style algorithms such as \texttt{SPIDER} \cite{fang2018spider} and \texttt{SNVRG} \cite{zhou2018stochastic} attain this optimal complexity by periodically using giant batch sizes. \cite{cutkosky2019momentum} propose \texttt{STORM} which also attains this optimal complexity with adaptive learning rates, but without using any large batches. The key idea of \texttt{STORM} is momentum-based variance reduction, obtained by using the stochastic gradient at the previous point \textit{computed over the same batch} on which the stochastic gradient at the current point is computed. \cite{liu2020optimal} present a much simpler proof for essentially the same algorithm by employing a constant learning rate and requiring a large batch size only at the first iteration. Our key idea of global and local momentum is \texttt{STORM}-like \textit{variance-reducing} momentum applied to the aggregation step at the server, that we interpret as a generalized gradient-type update, and the local updates at the clients, respectively; see \Cref{sec:main}.
\\
\\
\Cref{tb:comp} compares the complexities of the most relevant related works in FL/local SGD with ours (on smooth non-convex functions). Note that only \texttt{FedGLOMO} 
and \texttt{MimeMVR} \cite{karimireddy2020mime} attain the improved iteration complexity of $\mathcal{O}(\epsilon^{-1.5})$ with respect to $\epsilon$. However, unlike \cite{karimireddy2020mime}, our work does not rely on the bounded client dissimilarity assumption (\cref{eq:bcd}) and allows for compressed client-to-server communication, in which case  maintaining the improved complexity is not trivial; for details, see \Cref{rem-sep21-2} and \Cref{rem-sep21-3}, respectively. There are meaningful algorithmic differences between our work and \cite{karimireddy2020mime} too. The most noteworthy one is that while we explicitly apply momentum in the server aggregation step (global momentum) as well as in the client updates (local momentum), 
\cite{karimireddy2020mime} only apply globally computed momentum in the local client updates and \textit{no} momentum at the server. For a detailed discussion of the differences of our work from \cite{karimireddy2020mime}, see \Cref{sec:disc}.
Since \texttt{Mime} is designed to deal with client drift, we also empirically compare it against \texttt{FedGLOMO} without compression in a highly heterogeneous setting in \Cref{sec:exp}.

\begin{table*}[t]\small
\caption{Number of gradient updates, i.e., $T$, required to achieve $\mathbb{E}[\|\nabla f(\bm{w})\|^2] \leq \epsilon$ on smooth non-convex functions. \enquote{BCD?} asks if the bounded client dissimilarity assumption (i.e., \cref{eq:bcd}) is used or not.
Here, $n$ is the total number of clients and $r$ is the number of clients participating in each round.
\\
$*1$: Results are under full device participation, i.e., $r=n$. 
\\
$*2$: Here, $\alpha \leq n$ is a problem-dependent quantity; in practice, \textbf{we expect $\alpha \ll n$} as confirmed in our experiments.
}
\label{tb:comp}
\ra{1}
\begin{adjustbox}{width=\textwidth}
\begin{tabular*}{\linewidth}{@{}cccc@{}}\toprule
Ref. & $T$ & Compressed Communication? & BCD? \\ \midrule
\texttt{FedCOMGATE} 
\cite{haddadpour2020federated} &$\mathcal{O}(\frac{1}{n\epsilon^2})^{*1}$&{\color{black}\cmark}&Yes
\\\midrule
Local SGD 
\cite{koloskova2020unified,wang2019slowmo} &$\mathcal{O}(\frac{1}{n\epsilon^2})^{*1}$&{\color{red}\xmark}&Yes
\\\midrule
\texttt{SCAFFOLD} 
\cite{karimireddy2019scaffold}&$\O(\frac{1}{r \epsilon^2})$&{\color{red}\xmark}&Yes
\\\midrule
\texttt{MimeMVR} 
\cite{karimireddy2020mime}&$\mathcal{O}\big(\frac{1}{\sqrt{r}\epsilon^{1.5}}\big)$&{\color{red}\xmark}&Yes
\\\midrule
\textbf{This work} (\texttt{FedGLOMO})
&$\mathcal{O}\Big(\max\Big(\sqrt{\frac{\alpha}{n}}, \sqrt{\frac{(n-r)}{r(n-1)}}\Big)\frac{1}{\epsilon^{1.5}}\Big)^{*2}$&{\color{black}\cmark}&\textbf{No}
\\
\bottomrule
\end{tabular*}
\end{adjustbox}
\end{table*}

\section{Preliminaries}
\label{sec:prelim}
Recall the setting and the optimization problem that the server is trying to solve as defined in \cref{eq:fl-intro-1}. We assume that the clients have access to unbiased stochastic gradients of their individual losses. We denote the stochastic gradient of $f_i$ at $\bm{w}$ computed over a batch of samples $\mathcal{B}$, by $\widetilde{\nabla} f_i(\bm{w};\mathcal{B})$. Also in this paper, $K$ is the number of communication rounds, $E$ is the number of local updates per round or the period, and $T = KE$ is the total number of local updates or the (order-wise) number of gradient-based updates. Further, $r$ is the number of clients that the server accesses in each communication round, i.e., the global batch size.

\section{\texttt{FedGLOMO}: \texttt{G}lobal and \texttt{LO}cal \texttt{MO}mentum-Based Variance Reduction}
\label{sec:main}
\begin{algorithm}[!htb]
	\caption{\texttt{FedGLOMO} - Server Update}
	\label{alg:2}
	\begin{algorithmic}[1]
		\STATE {\bfseries Input:} Initial point $\bm{w}_0$, \# of rounds of communication $K$, period $E$, learning rates  $\{\eta_{k}\}_{k=0}^{K-1}$
		and global batch size $r$. $Q_D$ is the quantization operator. Set $\bm{w}_{-1} = \bm{w}_0$.
		\FOR{$k =0,\dots, K-1$}
		\STATE 
		Server sends $\bm{w}_k$, $\bm{w}_{k-1}$ to a set $\mathcal{S}_k$ of $r$ clients chosen uniformly at random w/o replacement.
		\FOR{client $i \in \mathcal{S}_k$}
		\STATE Set $\bm{w}_{k,0}^{(i)} = \bm{w}_k$ and $\widehat{\bm{w}}_{k-1,0}^{(i)} = \bm{w}_{k-1}$. Run \Cref{alg:2-local} for client $i$.
		\ENDFOR
		\IF{$k = 0$}
		\label{step-0}
		\STATE Set
		$\bm{u}_{k} = \frac{1}{r}\sum_{i \in \mathcal{S}_k}Q_D({\bm{w}_{k} - \bm{w}_{k,E}^{(i)}})$.
		\label{glob-mom-0}
		\ELSE
		\vspace{0.1 cm}
		\STATE Set
		$\bm{u}_{k} = \frac{\beta_k}{r} \sum_{i \in \mathcal{S}_k}Q_{D}(\bm{w}_{k} - {\bm{w}_{k,E}^{(i)}}) + 
		(1-\beta_k)\bm{u}_{k-1} + \frac{(1-\beta_k)}{r} \sum_{i \in \mathcal{S}_k} Q_{D}((\bm{w}_{k} - {\bm{w}_{k,E}^{(i)}}) - ({\bm{w}_{k-1} - \widehat{\bm{w}}_{k-1,E}^{(i)}}))$.
		{\color{blue} // \texttt{(Global Momentum)}}\label{glob-mom}
		\vspace{0.1 cm}
		\ENDIF
		\STATE Update $\bm{w}_{k+1} = \bm{w}_{k} - \bm{u}_k$.
		\ENDFOR
	\end{algorithmic}
\end{algorithm}

\begin{algorithm}[!htb]
	\caption{\texttt{FedGLOMO} - Client Update}
	\label{alg:2-local}
	\begin{algorithmic}[1]
		\FOR{$\tau = 0,\ldots,E-1$}
		\vspace{0.1 cm}
		\IF{$\tau = 0$}
		\vspace{0.1 cm}
		\STATE Set $\bm{v}_{k,\tau}^{(i)} = {\nabla} f_i(\bm{w}_{k,\tau}^{(i)})$,  $\widehat{\bm{v}}_{k-1,\tau}^{(i)} = {\nabla} f_i(\widehat{\bm{w}}_{k-1,\tau}^{(i)})$. 
		\label{l1}
		\ELSE
		\vspace{0.1 cm}
		\STATE Pick a random batch of samples 
		in client $i$, say $\mathcal{B}_{k,\tau}^{(i)}$.
		Compute the stochastic gradients 
		of $f_i$ at $\bm{w}_{k,\tau}^{(i)}$, $\widehat{\bm{w}}_{k-1,\tau}^{(i)}$, $\bm{w}_{k,\tau-1}^{(i)}$ and $\widehat{\bm{w}}_{k-1,\tau-1}^{(i)}$ over $\mathcal{B}_{k,\tau}^{(i)}$ viz.
		$\widetilde{\nabla} f_i(\bm{w}_{k,\tau}^{(i)};\mathcal{B}_{k,\tau}^{(i)})$, $\widetilde{\nabla} f_i(\widehat{\bm{w}}_{k-1,\tau}^{(i)};\mathcal{B}_{k,\tau}^{(i)})$, $\widetilde{\nabla} f_i(\bm{w}_{k,\tau-1}^{(i)};\mathcal{B}_{k,\tau}^{(i)})$ and $\widetilde{\nabla} f_i(\widehat{\bm{w}}_{k-1,\tau-1}^{(i)};\mathcal{B}_{k,\tau}^{(i)})$. 
		\vspace{0.1 cm}
		\STATE 
		Update: $\bm{v}_{k,\tau}^{(i)} = \widetilde{\nabla} f_i(\bm{w}_{k,\tau}^{(i)};\mathcal{B}_{k,\tau}^{(i)}) + \big(\bm{v}_{k,\tau-1}^{(i)} - \widetilde{\nabla} f_i(\bm{w}_{k,\tau-1}^{(i)};\mathcal{B}_{k,\tau}^{(i)})\big)$ and
		\\
		$\widehat{\bm{v}}_{k-1,\tau}^{(i)} = \widetilde{\nabla} f_i(\widehat{\bm{w}}_{k-1,\tau}^{(i)};\mathcal{B}_{k,\tau}^{(i)}) + \big(\widehat{\bm{v}}_{k-1,\tau-1}^{(i)} - \widetilde{\nabla} f_i(\widehat{\bm{w}}_{k-1,\tau-1}^{(i)};\mathcal{B}_{k,\tau}^{(i)})\big)$. 
		{\color{blue} // \texttt{(Local Momentum)}} \label{l2}
		\ENDIF
		\vspace{0.1 cm}
		\STATE Update $\bm{w}_{k,\tau+1}^{(i)} = \bm{w}_{k,\tau}^{(i)} - \eta_{k}\bm{v}_{k,\tau}^{(i)}$ and  $\widehat{\bm{w}}_{k-1,\tau+1}^{(i)} = \widehat{\bm{w}}_{k-1,\tau}^{(i)} - \eta_{k}\widehat{\bm{v}}_{k-1,\tau}^{(i)}$.
		\label{l3}
		\vspace{0.1 cm}
		\ENDFOR
		\vspace{0.1 cm}
		\STATE Send $Q_{D}(\bm{w}_{k} - {\bm{w}_{k,E}^{(i)}})$ and $Q_{D}((\bm{w}_{k} - {\bm{w}_{k,E}^{(i)}}) - ({\bm{w}_{k-1} - \widehat{\bm{w}}_{k-1,E}^{(i)}}))$ 
		to the server.
		\label{comp}
	\end{algorithmic}
\end{algorithm}

\begin{algorithm}[!htb]
	\caption{\texttt{FedAvg} 
	}
	\label{alg:fed-avg}
	\begin{algorithmic}[1]
		\STATE {\bfseries Input:} 
		Initial point $\bm{w}_0$, \# of communication rounds $K$, period $E$, learning rates  $\{\eta_{k}\}_{k=0}^{K-1}$ and global batch size $r$.
		\FOR{$k =0,\dots, K-1$}
		\STATE Server sends $\bm{w}_k$ to a set $\mathcal{S}_k$ of $r$ clients chosen uniformly at random w/o replacement.
		\FOR{client $i \in \mathcal{S}_k$}
		\STATE Set $\bm{w}_{k,0}^{(i)} = \bm{w}_k$.
		\FOR{$\tau = 0,\ldots,E-1$}
		\STATE Pick a random batch of samples 
		in client $i$, $\mathcal{B}_{k,\tau}^{(i)}$.
		Compute the stochastic gradient of $f_i$ at $\bm{w}_{k,\tau}^{(i)}$ over 
		$\mathcal{B}_{k,\tau}^{(i)}$, viz. $\widetilde{\nabla} f_i(\bm{w}_{k,\tau}^{(i)};\mathcal{B}_{k,\tau}^{(i)})$.
		\STATE Update $\bm{w}_{k,\tau+1}^{(i)} = \bm{w}_{k,\tau}^{(i)} - \eta_{k} \widetilde{\nabla} f_i(\bm{w}_{k,\tau}^{(i)};\mathcal{B}_{k,\tau}^{(i)})$.
		\ENDFOR
		\STATE Send $(\bm{w}_k - \bm{w}_{k,E}^{(i)})$ to the server.
		\label{line:fedavg-1}
		\ENDFOR
		\STATE Update $\bm{w}_{k+1} = \bm{w}_k -  \frac{1}{r}\sum_{i \in \mathcal{S}_k}(\bm{w}_k - \bm{w}_{k,E}^{(i)})$.
		\label{line:fedavg-2}
		\ENDFOR
	\end{algorithmic}
\end{algorithm}

{There are two issues that need to be alleviated for improving the convergence rate in FL: (i) the high variance of simple averaging used in the \textit{global} server aggregation step (of \texttt{FedAvg}), when there are multiple local updates, which is exacerbated by heterogeneity of the clients, and (ii) the high variance associated with the noise of \textit{local} client-level stochastic gradients.
The key idea of \texttt{FedGLOMO} (\Cref{alg:2} and \ref{alg:2-local}) is to apply \textit{variance-reducing} \textbf{global} and \textbf{local} momentum to combat (i) and (ii), respectively. We now describe {global} and {local} momentum in detail.

\textbf{Global} momentum is applied to the sever aggregation step which is line \ref{glob-mom} in \Cref{alg:2}. To understand it better, let us revisit \texttt{FedAvg} (summarized in \Cref{alg:fed-avg}, although in a slightly different way than usual) and its server aggregation step (line \ref{line:fedavg-2}) which is just simple averaging. First, note that the high variance of this naive averaging step slows down the convergence rate of \texttt{FedAvg} (and other related methods). We now re-envision the server aggregation as a generalized gradient-based update by thinking of $(\bm{w}_k - \bm{w}_{k,E}^{(i)})$ as the generalized gradient. Then, we wish to incorporate the style of variance-reducing momentum applied in \texttt{STORM} \cite{cutkosky2019momentum,liu2020optimal} (note that their method is for stochastic gradients in the case of centralized optimization) to our generalized gradient-based update; for that, let us briefly recap \texttt{STORM}'s update rule. For a function $h(\bm{z})$, \texttt{STORM}'s update rule is as follows for the $j^{\text{th}}$ iteration:
\begin{multline}
    \label{apr20-1}
    \bm{z}_{j+1} = \bm{z}_j - \eta_j \bm{v}_j, \text{ where }
    \bm{v}_j = 
    \begin{cases}
      \widetilde{\nabla} h(\bm{z}_j;\xi_j) & \text{for}\ j=0 \\
      \widetilde{\nabla} h(\bm{z}_j;\xi_j) + (1-\beta_j) (\bm{v}_{j-1} - \widetilde{\nabla} h(\bm{z}_{j-1};\xi_j)) & \text{for}\ j>0. 
    \end{cases}
\end{multline}
In \cref{apr20-1}, $\xi_j$ denotes the source of randomness in the $j^{\text{th}}$ iteration and $\beta_j \in [0,1)$ is the momentum parameter. Note the use of the stochastic gradient at $\bm{z}_{j-1}$ computed on $\xi_j$.
Coming back to \Cref{alg:2}, the quantity $\bm{u}_k$ plays the role of $\bm{v}_j$ in  \cref{apr20-1}. To see this clearly, let us analyze $E_{Q_D}[\bm{u}_k]$ (see lines \ref{glob-mom-0} and \ref{glob-mom} in \Cref{alg:2}).
Under \Cref{as5}, $Q_D$ produces an unbiased estimate of the input. Then defining ${g}(\bm{w}_k;\mathcal{S}_k) \triangleq \frac{1}{r}\sum_{i \in \mathcal{S}_k}({\bm{w}_{k} - \bm{w}_{k,E}^{(i)}})$ and $\widehat{g}(\bm{w}_{k-1};\mathcal{S}_k) \triangleq \frac{1}{r}\sum_{i \in \mathcal{S}_k}({\bm{w}_{k-1} - \widehat{\bm{w}}_{k-1,E}^{(i)}})$, we have:
\begin{equation}
    \label{apr20-3}
    E_{Q_D}[\bm{u}_k] = 
    \begin{cases}
      {g}(\bm{w}_k;\mathcal{S}_k) & \text{for}\ k=0 \\
      {g}(\bm{w}_k;\mathcal{S}_k) + (1-\beta_k)\big(\bm{u}_{k-1} - \widehat{g}(\bm{w}_{k-1};\mathcal{S}_k)\big) & \text{for}\ k>0. 
    \end{cases}
\end{equation}
In \cref{apr20-3}, ${g}(\bm{w}_k;\mathcal{S}_k)$ and $\widehat{g}(\bm{w}_{k-1};\mathcal{S}_k)$ play the roles of $\widetilde{\nabla} h(\bm{z}_j;\xi_j)$ and $\widetilde{\nabla} h(\bm{z}_{j-1};\xi_j)$, respectively. 
With this, one can clearly see that \cref{apr20-3} is the analogue of \cref{apr20-1} for the global server aggregation in FL. However, this equivalence is not so apparent without looking at the expected value of $\bm{u}_k$ with respect to ${Q}_D$; in fact, the choice of quantities that are compressed in {line \ref{comp} of Alg. \ref{alg:2-local}} and used in line \ref{glob-mom} of Alg. \ref{alg:2} is crucial for making our theory work (also see \Cref{rem-sep21-3}). 

Now that we understand global momentum, let us move on to \textbf{local} momentum. For this see lines \ref{l1}, \ref{l2} and \ref{l3} in \Cref{alg:2-local}; these give us $(\bm{w}_{k} - {\bm{w}_{k,E}^{(i)}})$ and $({\bm{w}_{k-1} - \widehat{\bm{w}}_{k-1,E}^{(i)}})$ after running for $E$ steps. But notice that these lines are the same as \cref{apr20-1} with $\beta_j = 0$ and the stochastic gradient at the first iteration replaced by the full gradient. It is worth mentioning here that these local updates are also similar to \texttt{SPIDER} which is an SVRG-style update proposed in \cite{fang2018spider}. However, recognizing that this is also a special case of the \texttt{STORM} update with $\beta_j = 0$, we prefer calling it momentum in order to have a unifying terminology for both the global and local updates. 

One might wonder what is the role of global momentum as \texttt{SPIDER} can be extended to improve the complexity in distributed optimization \textit{without multiple local updates}. For this, {in \Cref{sec:lomo}}, we consider \texttt{FedLOMO} (Alg. \ref{alg:1} and \ref{alg:1-local}) which is a simpler version of \texttt{FedGLOMO} with only {lo}cal {mo}mentum and \textit{no} global momentum (i.e, plain averaging at the server which is equivalent to setting $\beta_k = 1$ in Alg. \ref{alg:2}),
and show that it does not achieve $\mathcal{O}(\epsilon^{-1.5})$ complexity ({see \Cref{fl-thm3}}). The root cause of this is client-heterogeneity which amplifies its effect under \textit{multiple local updates}; without incorporating some form of variance reduction in the server aggregation step, the complexity cannot be improved.

Let us try to provide some intuition as to how incorporating global momentum helps. Suppose we keep $\eta_k = \eta$ and $\beta_k = \beta < 1$ for all $k$. Theoretically, we get a lower bound for $\beta$ which is approximately $\mathcal{O}(\eta^2)$. Then with this momentum-based aggregation strategy, the variance reduces by a factor of $\mathcal{O}(\beta/\eta) = \mathcal{O}(\eta)$ as compared to aggregation by plain averaging. (There are some other terms too but these are sufficiently small.) This reduction in the variance by a factor of $\mathcal{O}(\eta)$ is what enables \texttt{FedGLOMO} to enjoy a faster convergence rate.

It is true that \texttt{FedGLOMO} has to communicate twice the amount of information per round as compared to a \texttt{FedAvg} (or \texttt{FedPAQ} \cite{reisizadeh2020fedpaq} which is \texttt{FedAvg} with compressed communication) per round. One can set the precision of the quantizer sufficiently low to account for the extra per-round communication cost of \texttt{FedGLOMO} -- we do this in our experiments.}

Also, we only assume access to the full client gradient in line \ref{l1} of \Cref{alg:2-local} for simplicity of analysis, but our main result (i.e., \Cref{nov4-thm1}) can be extended to the case of large enough batch sizes.

\section{Main Result for \texttt{FedGLOMO}}\label{sec:result:glomo}
First, we state our assumptions.
\begin{assumption}[\textbf{Smoothness}]
\label{as1} 
$\ell(\bm{x},\bm{w})$ is $L$-smooth with respect to $\bm{w}$, for all $\bm{x}$. Thus, each $f_i(\bm{w})$ ($i \in [n]$) is $L$-smooth, and so is $f(\bm{w})$.
\end{assumption}

\begin{assumption}[\textbf{Non-negativity}]
\label{as-may15}
Each $f_i(\bm{w})$ is non-negative and therefore, $f_i^{*} \triangleq \min f_i(\bm{w}) \geq 0$.
\end{assumption}
Most of the loss functions used in practice satisfy this anyways and if not, we can just add a constant offset to achieve non-negativity. 

\begin{assumption}[\textbf{Quantization operator}]\label{as5}
The randomized quantization operator $Q_D$ in \Cref{alg:2} and \ref{alg:2-local} is unbiased, i.e., $\mathbb{E}[Q_D(\bm{x}) | \bm{x}] = \bm{x}$, and its variance satisfies $\mathbb{E}[\|Q_D(\bm{x})-\bm{x}\|^2 | \bm{x}] \leq q\|\bm{x}\|^2$ for some $q > 0$. The \enquote{qsgd} operator proposed in Section 3.1 of \cite{alistarh2017qsgd} satisfies these properties.
\end{assumption}

{
\begin{assumption}[\textbf{Heterogeneity}]\label{as-het}
Suppose all clients participate, i.e. $r=n$, in the $(k+1)^{\text{st}}$ round of \texttt{FedGLOMO} (Alg. \ref{alg:2} and \ref{alg:2-local}). 
Let $\bm{w}_{k,\tau}^{(i)}$ be the $i^{\text{th}}$ client's local parameter at the $(\tau+1)^{\text{st}}$ local step of the $(k+1)^{\text{st}}$ round of \texttt{FedGLOMO}, for $i \in [n]$. Define $\widetilde{\bm{e}}_{k,\tau}^{(i)} \triangleq \nabla f_i(\bm{w}_{k,\tau}^{(i)}) - \nabla f_i(\overline{\bm{w}}_{k,\tau})$, where $\overline{\bm{w}}_{k,\tau} \triangleq \frac{1}{n}\sum_{i \in [n]} \bm{w}_{k,\tau}^{(i)}$.
Then for some $\alpha \ll n$: 
\begin{equation*}
    \mathbb{E}\Big[\Big\|\sum_{i \in [n]}\widetilde{\bm{e}}_{k, \tau}^{(i)}\Big\|^2\Big] \leq \alpha \sum_{i \in [n]} \mathbb{E}\Big[\Big\|\widetilde{\bm{e}}_{k,\tau}^{(i)}\Big\|^2\Big], \text{ } \forall \text{ } \tau \in [E].
\end{equation*}
\end{assumption}
Obviously, the above assumption always holds with $\alpha = n$; this follows from the Cauchy-Schwarz inequality. However, we empirically observe $\alpha \ll n$ in practice; {see \Cref{sec:het-asm-expt}}. 
The value of $\alpha$ depends on the degree of heterogeneity -- as the heterogeneity increases (decreases), we observe $\alpha$ to also increase (decrease). Thus, \Cref{as-het} can characterize the degree of heterogeneity in the system.}

{From the plots in \Cref{sec:het-asm-expt}, we see that for the highly heterogeneous setting that we consider for experiments in \Cref{sec:exp}, $\alpha < 0.06 n$ for most of the trajectory of \texttt{FedGLOMO} on both CIFAR-10 and Fashion-MNIST (abbreviated as FMNIST). In the homogeneous case, $\alpha < 0.03 n$ and $\alpha < 0.02 n$ for most of the trajectory on CIFAR-10 and FMNIST, respectively.}

Further, this assumption is not specific to \texttt{FedGLOMO} and is potentially applicable for other FL algorithms. In \Cref{sec:fed_avg_conv}, we empirically observe that this assumption also holds for \texttt{FedAvg} and we use it to derive a novel convergence result for \texttt{FedAvg} without the bounded client dissimilarity assumption (i.e., \cref{eq:bcd}).

We now present our main result, followed by some important remarks and a proof sketch. The detailed proof can be found in \Cref{sec-pf-2}.

\begin{theorem} [\textbf{Smooth non-convex case}]
\label{nov4-thm1}
Suppose Assumptions \ref{as1}, \ref{as-may15}, \ref{as5} and \ref{as-het} hold. 
In \texttt{FedGLOMO}, set: 
\[\text{$\eta_{k} = \eta = \frac{1}{6 L E K^{1/3} (\frac{1}{n}(\alpha + \frac{4}{E}) + 800 e^2 (1+q) (E+1)^2 (\frac{q}{n} + \frac{(1+q)(n-r)}{r (n-1)}))^{1/3}}$ and}\] 
\[\beta_k = \beta = 160 e^2 (1+q) \eta^2 L^2 E^2 (E+1)^2.\]
Suppose we use full-device participation (i.e., the global batch size is $n$) \textbf{only at} $k = 0$. Then if $\frac{K^{-1/3}}{1200e^2(1+q) \big(\frac{q}{n} + \frac{(1+q)(n-r)}{r (n-1)}\big)} \leq E+1 \leq \frac{\sqrt{1+q} (n-r)}{3 r(n-1)} K$, we have:
\begin{flalign*}
   \frac{1}{K} \sum_{k=0}^{K-1}\mathbb{E}[\|\nabla f(\bm{w}_{k})\|^2] \leq \frac{39 L f(\bm{w}_{0})}{K^{2/3}} \Big({\frac{1}{n}\Big(\alpha + \frac{4}{E}\Big)} + 800 e^2 (1+q) (E+1)^2 \Big(\frac{q}{n} + \frac{(1+q)(n-r)}{r(n-1)}\Big)\Big)^{1/3}.
\end{flalign*}
Thus, \texttt{FedGLOMO} can achieve $\mathbb{E}[\|\nabla f(\bm{w}_{k^{*}})\|^2] \leq \epsilon$, where $k^{*} \sim \text{Unif}[0,K-1]$, in $E = \mathcal{O}(1)$ local steps and $K = \mathcal{O}\Big( \max\Big\{\sqrt{\frac{\alpha}{n}}, (1+q)\sqrt{\frac{(n-r)}{r(n-1)}}\Big\}{\epsilon^{-1.5}}\Big)$ rounds of communication.
\end{theorem}
Note that the above result is independent of the variance of local stochastic gradients (of the clients). In short, this happens because we use local full gradients at $\tau=0$ and because the local stochastic gradients are Lipschitz.

We now make some remarks to discuss the implications of \Cref{nov4-thm1}.
\begin{remark}[\textbf{Better iteration complexity}]
\label{rem-sep21-1}
{According to \Cref{nov4-thm1}, for converging to an $\epsilon$-stationary point, \texttt{FedGLOMO} needs $T = KE$ to be $\mathcal{O}\big(\max\big(\sqrt{\frac{\alpha}{n}}, \sqrt{\frac{(n-r)}{r(n-1)}}\big){\epsilon}^{-1.5}\big)$. {This iteration complexity is the same as that of \texttt{MimeMVR} \cite{karimireddy2020mime} \textit{but without using the bounded client dissimilarity assumption}, i.e. \cref{eq:bcd}, (also see the next remark for more details on this) and better than other related works in the federated setting; see \Cref{tb:comp}.}
We underscore the significance of global momentum here by comparing this complexity of \texttt{FedGLOMO} to that of \texttt{FedLOMO} (recall this is a simpler version of \texttt{FedGLOMO} with only local momentum and \textit{no} global momentum, described in \Cref{sec:lomo}) which is $\mathcal{O}(\frac{1}{r} \epsilon^{-2})$ ({see \Cref{fl-thm3}}).}
\end{remark}

\begin{remark}[\textbf{No requirement of bounded client dissimilarity (BCD) assumption}]
\label{rem-sep21-2}
Divergent from related works, \Cref{nov4-thm1} \textit{does not use} the commonly used BCD assumption, i.e., \cref{eq:bcd}.
This is achieved by utilizing the smoothness and non-negativity of the $f_i$'s, specifically $\frac{1}{n}\sum_{i \in [n]}\|\nabla f_i(\bm{w})\|^2 \leq \frac{1}{n}\sum_{i \in [n]} 2L(f_i(\bm{w}) - f_i^{*}) \leq 2L f(\bm{w})$; see the proof outline of \Cref{nov4-thm1} after \Cref{rem-sep29-5}. In lieu of the hard to verify BCD assumption, we use the empirically verified \Cref{as-het} to tighten our convergence result. Note that \Cref{as-het} will always hold for some $\alpha \leq n$, regardless of the degree of client heterogeneity. Thus, our theory also allows for \textit{arbitrary client heterogeneity}.
\end{remark}

\begin{remark}[\textbf{Compressed communication}]
\label{rem-sep21-3}
To our knowledge, \texttt{FedGLOMO} is the \textit{first algorithm} that attains the aforementioned improved iteration complexity for FL on smooth non-convex functions \textit{with compressed communication}. We emphasize that  the choice of quantities that are compressed in line \ref{comp} of \Cref{alg:2-local}
is important. This particular choice enables deriving the improved rate by first deriving a result analogous to smoothness, i.e.,
$\|({\bm{w}_{k} - \bm{w}_{k,E}^{(i)}}) - ({\bm{w}_{k-1} - \widehat{\bm{w}}_{k-1,E}^{(i)}})\| \leq \widehat{L}\|\bm{w}_{k} - \bm{w}_{k-1}\|$ ({this derivation is done in \Cref{nov-1-lem3} in \Cref{sec-pf-2}}). The straightforward choice of sending $Q_{D}(\bm{w}_{k} - {\bm{w}_{k,E}^{(i)}})$ and $Q_{D}({\bm{w}_{k-1} - \widehat{\bm{w}}_{k-1,E}^{(i)}})$ prohibits us from deriving the improved rate, unless 
we also assume $Q_D(.)$ to be a Lipschitz operator.

{In \Cref{sec:red_bits}}, for $r \ll n$, we show that using the quantization scheme of \cite{alistarh2017qsgd} with $s = \sqrt{d}$, \texttt{FedGLOMO} achieves more than a five-fold saving in the \textit{total} communication cost as compared to when there is full-precision communication in \texttt{FedGLOMO}.
\end{remark}

\begin{remark}[\textbf{A limitation}]
\label{rem-sep29-5}
{Even though our iteration complexity of $T = \mathcal{O}(\epsilon^{-1.5})$ is better than that of \texttt{FedCOMGATE}  \cite{haddadpour2020federated} (which is $\mathcal{O}(\epsilon^{-2})$), our communication complexity of $K = \mathcal{O}(\epsilon^{-1.5})$ is higher than that of \texttt{FedCOMGATE} which is $K = \mathcal{O}(\epsilon^{-1})$ (albeit under an extra assumption on the quantizer, namely Assumption 5 in their paper). Ideally, we would like to have $E = \mathcal{O}(\epsilon^{-p})$ and $K = \mathcal{O}(\epsilon^{-(1.5-p)})$ for some $p > 0$, in order to reduce \texttt{FedGLOMO}'s communication complexity. 
Exploring whether such a result is obtainable with our proposed style of momentum is an interesting future direction.}
\end{remark}

\subsection*{Proof Sketch of Theorem~\ref{nov4-thm1}:}
Before getting to the proof outline, we would like to mention that the key technical challenge in deriving the improved convergence result with global momentum-based variance reduction is obtaining an analogue of the Lipschitzness of stochastic gradients to the change in local parameters over $E$ local steps. More specifically, for pure stochastic optimization, a key step in proving convergence of momentum-based variance reduction methods is using the Lipschitzness of the stochastic gradients (or the update quantities) \cite{cutkosky2019momentum,liu2020optimal}, i.e., \[\|\nabla\widetilde{f}(\bm{x}_t,\xi_t) - \nabla\widetilde{f}(\bm{x}_{t-1},\xi_t)\| \leq L \|\bm{x}_t - \bm{x}_{t-1}\|.\] In the FL setting where aggregation is performed at the server, we need an analogue of this at the server, i.e., something like
\[\|({\bm{w}_{k} - \bm{w}_{k,E}^{(i)}}) - ({\bm{w}_{k-1} - \widehat{\bm{w}}_{k-1,E}^{(i)}})\| \leq \widetilde{L}\|\bm{w}_{k} - \bm{w}_{k-1}\|.\] Deriving this result is a part of our contribution and is done in  \Cref{nov-1-lem3} (in \Cref{sec-pf-2}).

\begin{proof}
We set $\eta_k = \eta$ and $\beta_k = \beta$ $\forall$ $k \in \{0,\ldots,K-1\}$. Then, using \Cref{nov-1-lem0} with full global as well as local batch sizes at $k=0$ (by which $\bm{u}_{0} = \overline{\bm{\delta}}_{0}$ in the statement of \Cref{nov-1-lem0}), we have at any $k' > 0$:
\begin{multline}
    \label{eq-feb9-11}
    \mathbb{E}[f(\bm{w}_{k'})] \leq 
    f(\bm{w}_{0}) -\frac{\eta E}{4}\sum_{k=0}^{k'-1}\mathbb{E}[\|\nabla f(\bm{w}_{k})\|^2] 
    + \frac{16 \eta^3 L^2 E^2 (\alpha E + 4) }{n^2} \sum_{k=0}^{k'-1}\sum_{i \in [n]}\mathbb{E}[\|\nabla f_i(\bm{w}_k)\|^2]
    \\
    + 
    160\eta E \beta \Big(\frac{q}{n^2} + \frac{(1+q)}{r(n-1)}\Big(1 - \frac{r}{n}\Big)\Big)
    \sum_{k=0}^{k'-1}\sum_{i \in [n]} \mathbb{E}[\|\nabla f_i(\bm{w}_k)\|^2],
\end{multline}
for $4\eta L E^2 \leq 1$ and $\beta \geq \frac{80 e^2 (1+q) \eta^2 L^2 E^2 (E+1)^2}{(1 - 4\eta L E)}$. 

Also, since the $f_i$'s are $L$-smooth and non-negative, using \Cref{lem1-oct20}, we have that: 
\[\sum_{i \in [n]} \mathbb{E}[\|\nabla f_i(\bm{w}_k)\|^2 \leq \sum_{i \in [n]} 2L(\mathbb{E}[f_i(\bm{w}_k)] - f_i^{*}) \leq 2n L \mathbb{E}[f(\bm{w}_k)] - 2L \sum_{i \in [n]} f_i^{*} \leq 2n L \mathbb{E}[f(\bm{w}_k)].\]
This step allows us to circumvent the need for the bounded client dissimilarity assumption. Using this in (\ref{eq-feb9-11}), we get:
\begin{multline}
    \label{eq:may23-1}
    \mathbb{E}[f(\bm{w}_{k'})] \leq 
    f(\bm{w}_{0}) -\frac{\eta E}{4}\sum_{k=0}^{k'-1}\mathbb{E}[\|\nabla f(\bm{w}_{k})\|^2] 
    \\
    + \underbrace{64 \eta L E \Big(\frac{\eta^2 L^2 E (\alpha E + 4)}{n} + 5 \beta \Big(\frac{q}{n} + \frac{(1+q)(n-r)}{r(n-1)}\Big)\Big)}_{=\gamma}\sum_{k=0}^{k'-1} \mathbb{E}[f(\bm{w}_{k})].
\end{multline}
Unfolding the above recursion and simplifying a bit, we get:
\begin{flalign}
    \label{eq:may23-2}
    \sum_{k=0}^{k'-1}\mathbb{E}[f(\bm{w}_{k})] \leq k' f(\bm{w}_{0}) - \frac{\eta E}{4}\sum_{k=0}^{k'-1} \mathbb{E}[\|\nabla f(\bm{w}_{k})\|^2] + \gamma k' \sum_{k=0}^{k'-1} \mathbb{E}[f(\bm{w}_{k})].
\end{flalign}
Let us now ensure that $\gamma k' \leq \frac{1}{2}$ for all $k' \in \{1,\ldots,K\}$, so that we can simplify (\ref{eq:may23-2}) to:
\begin{equation}
    \label{eq:may23-3}
    \sum_{k=0}^{k'-1}\mathbb{E}[f(\bm{w}_{k})] \leq 2k' f(\bm{w}_{0}) - \frac{\eta E}{2}\sum_{k=0}^{k'-1} \mathbb{E}[\|\nabla f(\bm{w}_{k})\|^2] \leq 2k' f(\bm{w}_{0}).
\end{equation}
Now for $8 \eta L E^2 \leq 1$, it can be verified that $\beta = 160e^2 (1+q) \eta^2 L^2 E^2 (E+1)^2$ is a valid choice. Using this, we get that:
\begin{equation}
    \label{eq:may23-4}
    \gamma k' \leq \gamma K = \underbrace{64 \eta^3 L^3 E^3 K \Big(\frac{1}{n}\Big(\alpha + \frac{4}{E}\Big) + 800 e^2 (1+q) (E+1)^2 \Big(\frac{q}{n} + \frac{(1+q)(n-r)}{r(n-1)}\Big)\Big)}_{\text{(A)}}.
\end{equation}
Setting
\begin{equation*}
    \eta = \frac{1}{6 L E K^{1/3} (\frac{1}{n}(\alpha + \frac{4}{E}) + 800 e^2 (1+q) (E+1)^2 (\frac{q}{n} + \frac{(1+q)(n-r)}{r (n-1)}))^{1/3}},
\end{equation*}
we have (A) $ < \frac{1}{2}$. We also need to ensure that $8 \eta L E^2 \leq 1$ and $\beta = 160e^2 (1+q) \eta^2 L^2 E^2 (E+1)^2 < 1$. The range of $E$ in the theorem statement is obtained by combining the constraints (on $E$) that we get from these two requirements.

Finally, using (\ref{eq:may23-3}) in (\ref{eq:may23-1}) with $k' = K$, substituting our choice of $\eta$ and $\beta$, and then simplifying a bit more, we get the final convergence result.
\end{proof}

\section{Experiments}
\label{sec:exp}
To show the efficacy of \textit{global} momentum in \texttt{FedGLOMO}, we compare it against \texttt{FedLOMO} (recall {this has only local momentum and no global momentum}; see \Cref{sec:lomo}) and the default algorithm of choice for FL, i.e., \texttt{FedAvg} \cite{mcmahan2017communication} with the standard momentum available in PyTorch applied to its local updates -- both with and without compression.
Note that \texttt{FedAvg} with compression is referred to as \texttt{FedPAQ} \cite{reisizadeh2020fedpaq}.
We call the standard momentum versions of \texttt{FedAvg} and \texttt{FedPAQ} as \texttt{FedAvg}-m and \texttt{FedPAQ}-m henceforth. For quantization, we use the \enquote{qsgd} operator proposed in Section 3.1 of \cite{alistarh2017qsgd}. In the no-compression heterogeneous case, we also compare against \texttt{Mime} (specifically, \enquote{\texttt{MimeSGDm}}) \cite{karimireddy2020mime} which is also shown to attain the improved complexity of $\mathcal{O}(\epsilon^{-1.5})$ but without compressed communication, and is tailored to handle client heterogeneity.

We consider the task of 
classification on CIFAR-10 and Fashion-MNIST \cite{xiao2017fashion} abbreviated as FMNIST henceforth. The model used is a two-layer neural network with ReLU activation in the hidden layers. The size of both the hidden layers is 300/600 for FMNIST/CIFAR-10. We train the models using the categorical cross-entropy loss with $\ell_2$-regularization. 
The weight decay value in PyTorch (to apply $\ell_2$-regularization) is set to 1e-4. We consider both homogeneous 
and heterogeneous data distribution among the clients. Similar to \cite{mcmahan2017communication}, for the heterogeneous case, we distribute the data among the clients 
such that each client can have data from either one or (at most) two classes -- note that this is a high degree of heterogeneity.
{The exact procedure is described in \Cref{sec:extra-exp}.} The number of clients ($n$) in all the experiments is set to 50, with each client having the same number of samples. In this set of experiments, the global batch-size $r$ is 25, and the number of local updates per round (i.e., $E$) is 10. For \texttt{FedGLOMO}, we use a constant value of $\beta_k = 0.2$. 
For \texttt{FedAvg}-m and \texttt{FedPAQ}-m, the momentum parameter in Pytorch is set to its standard value, i.e., 0.9. As suggested in \cite{karimireddy2020mime}, we search $\beta$ (momentum hyper-parameter in \texttt{MimeSGDm}) over $\{0,0.9,0.99\}$. All full gradients in \texttt{FedGLOMO}, \texttt{FedLOMO} and \texttt{Mime} are replaced by stochastic gradients computed on a (per-client) batch size of 256.
{The learning rates and some other experimental details are in \Cref{sec:extra-exp}.}

\begin{figure*}[!htb]
\centering 
\subfloat[Heterogeneous FMNIST train loss]{
    \label{fig:1_a}
	\includegraphics[width=0.45\textwidth]{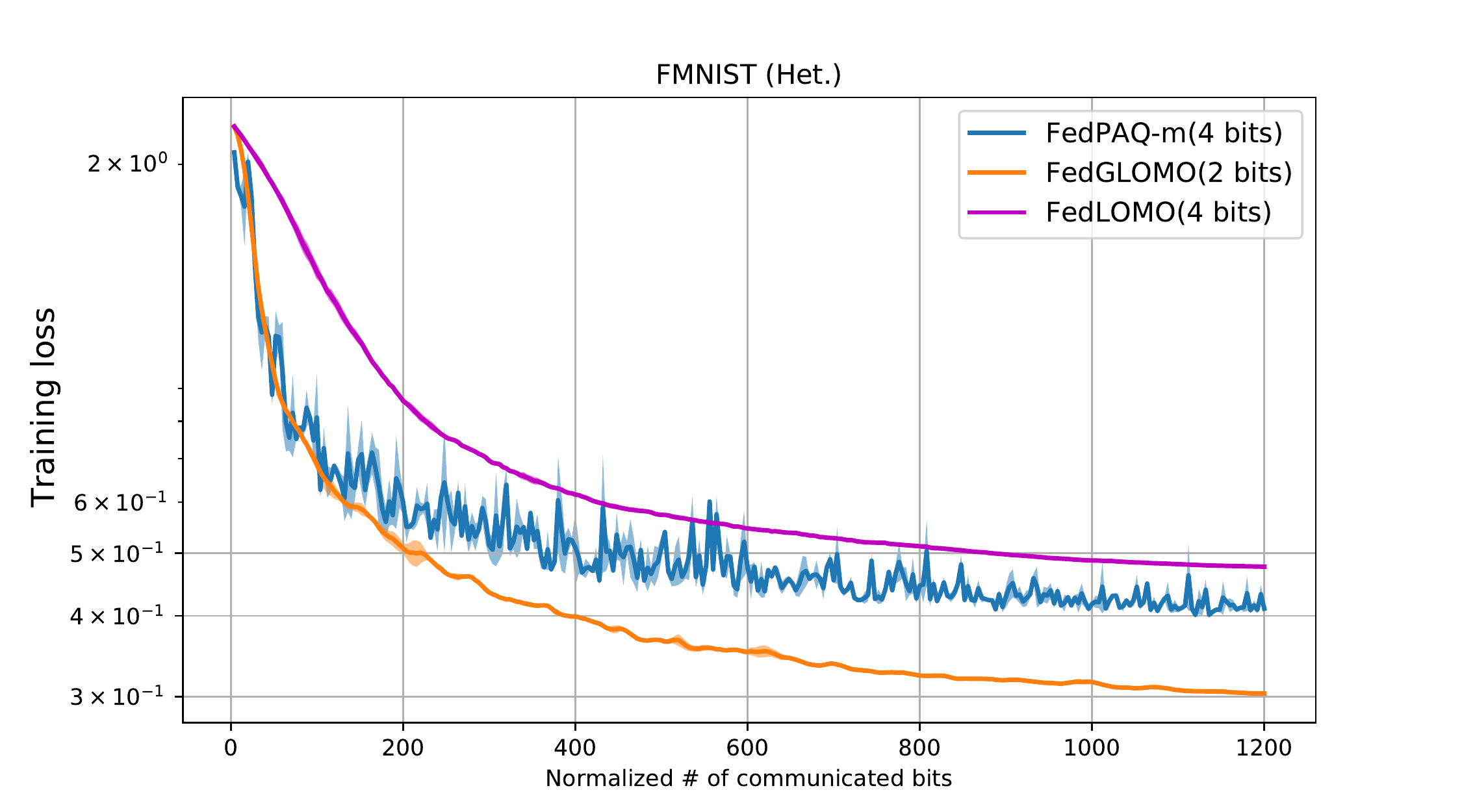}
	} 
\subfloat[Heterogeneous FMNIST test error]{
    \label{fig:1_b}
	\includegraphics[width=0.45\textwidth]{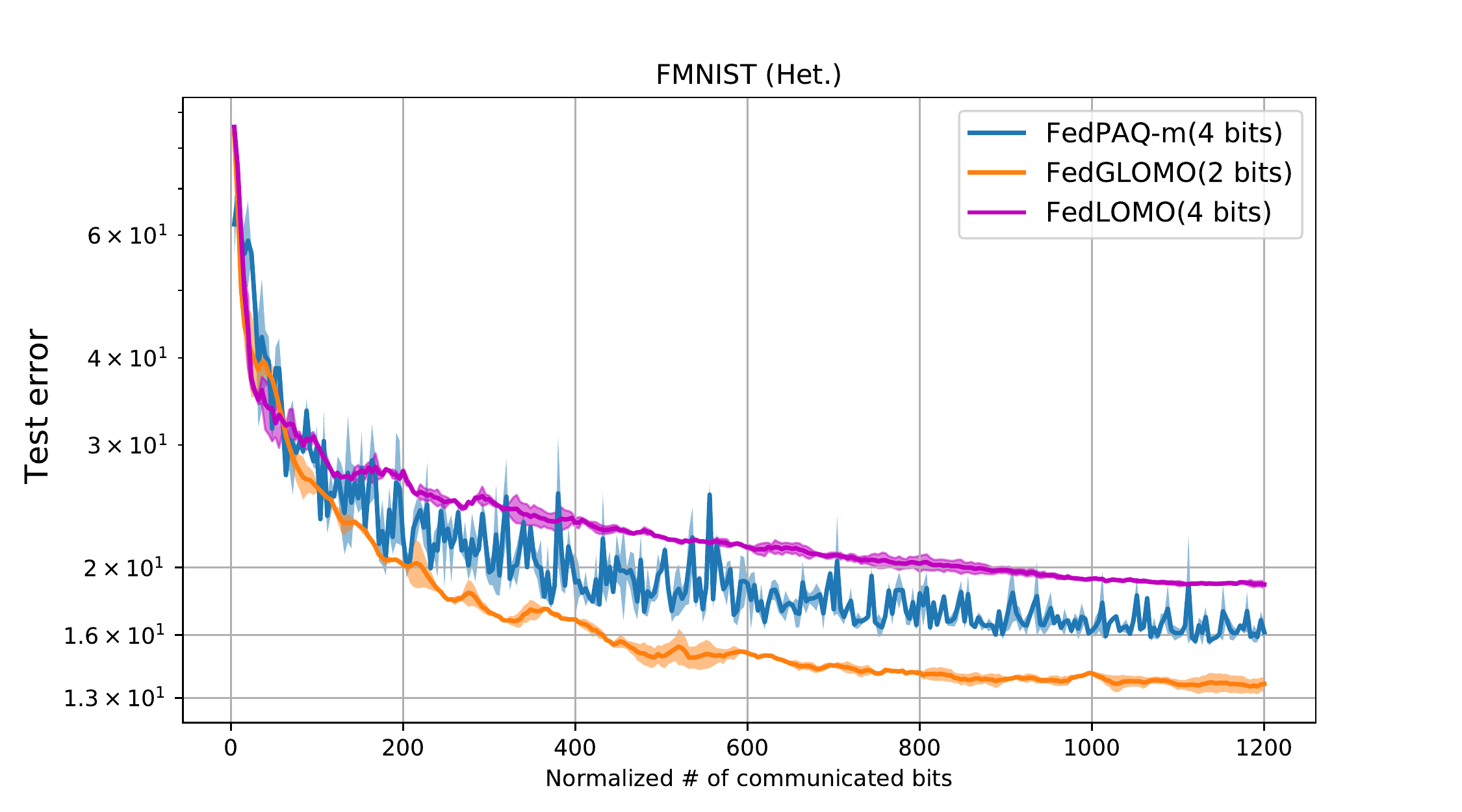}
	} 
\\
\subfloat[Heterogeneous CIFAR-10 train loss]{
    \label{fig:1_c}
	\includegraphics[width=0.45\textwidth]{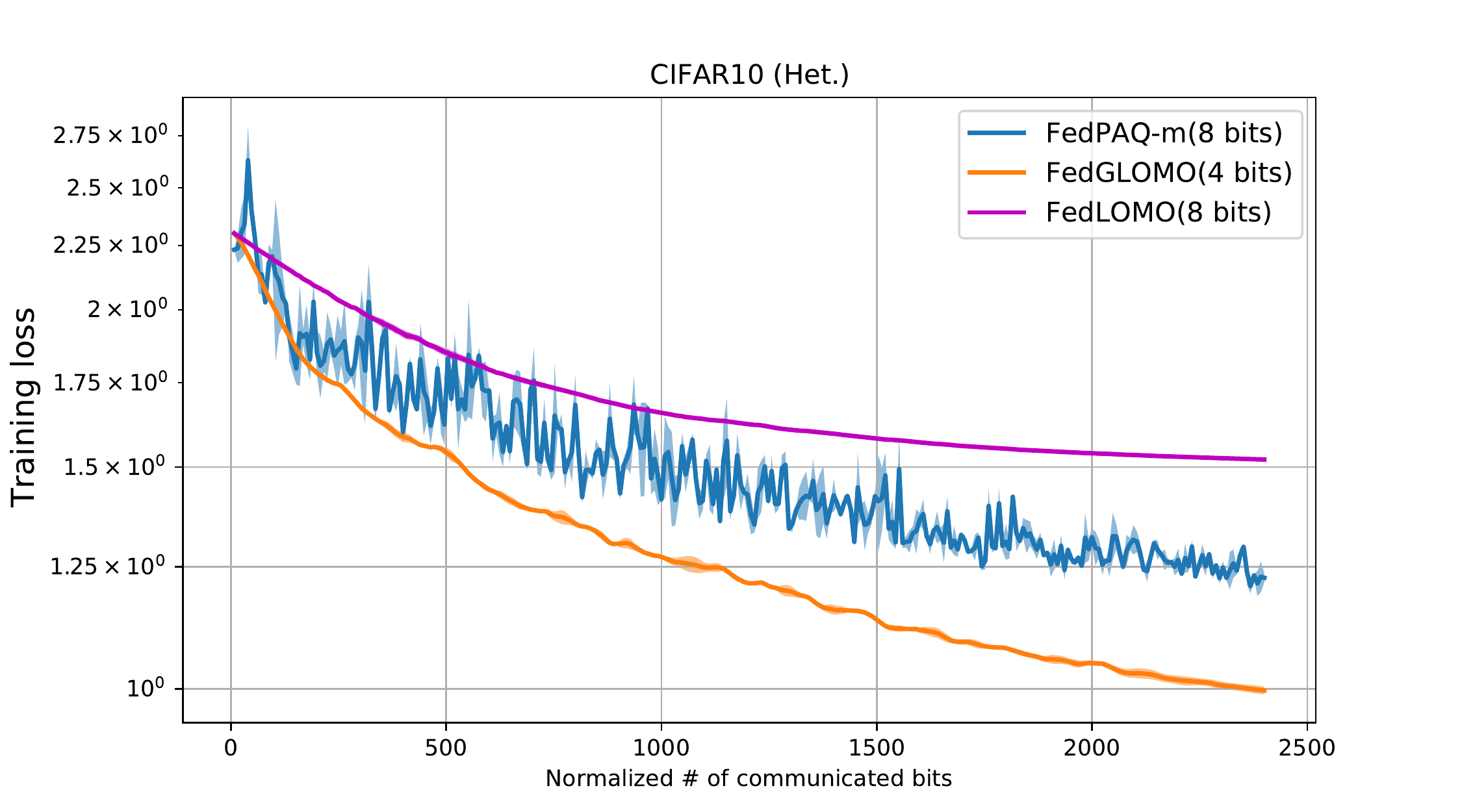}
	} 
\subfloat[Heterogeneous CIFAR-10 test error]{
    \label{fig:1_d}
	\includegraphics[width=0.45\textwidth]{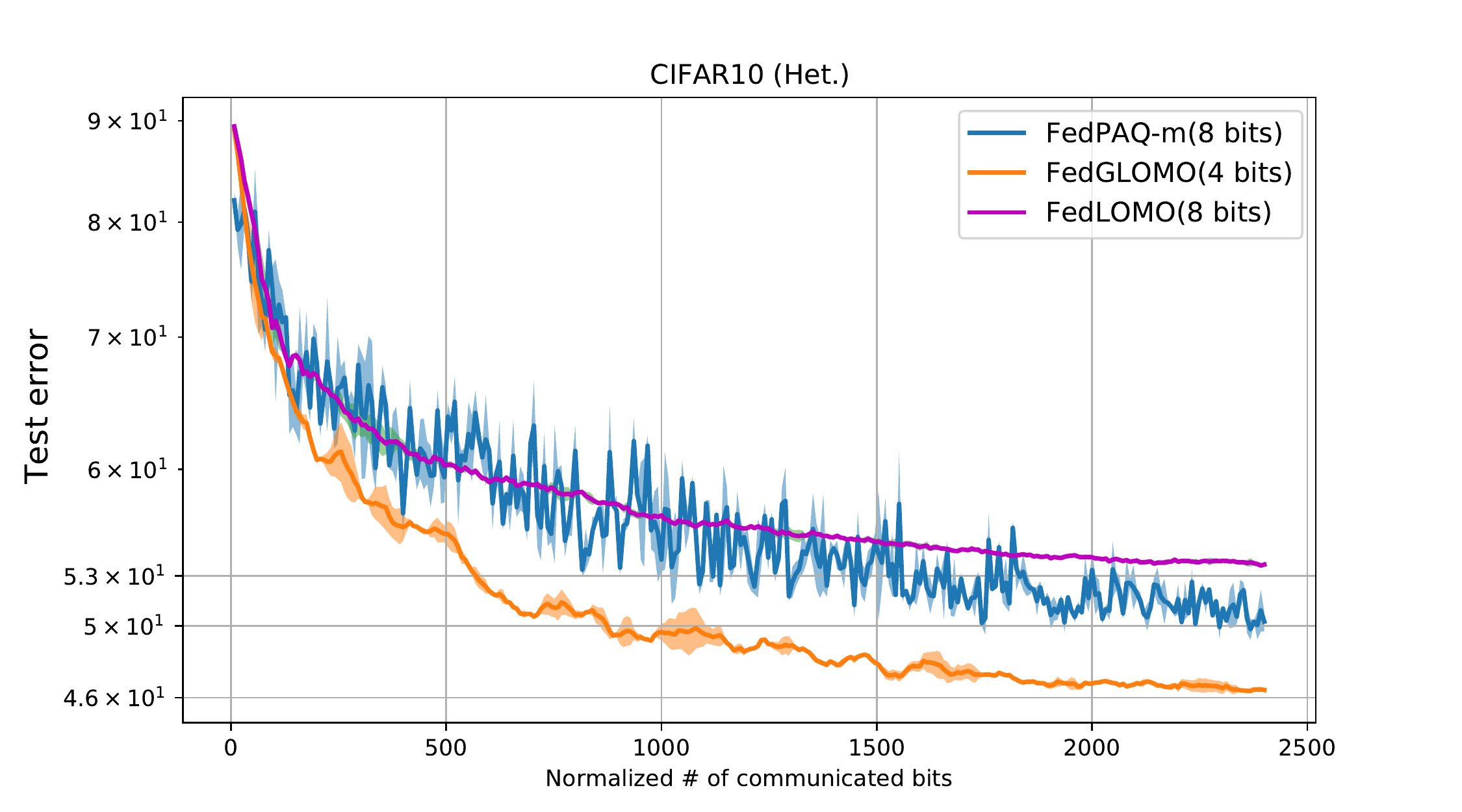}
	} 
\\
\subfloat[Homogeneous FMNIST train loss]{
    \label{fig:11_a}
	\includegraphics[width=0.45\textwidth]{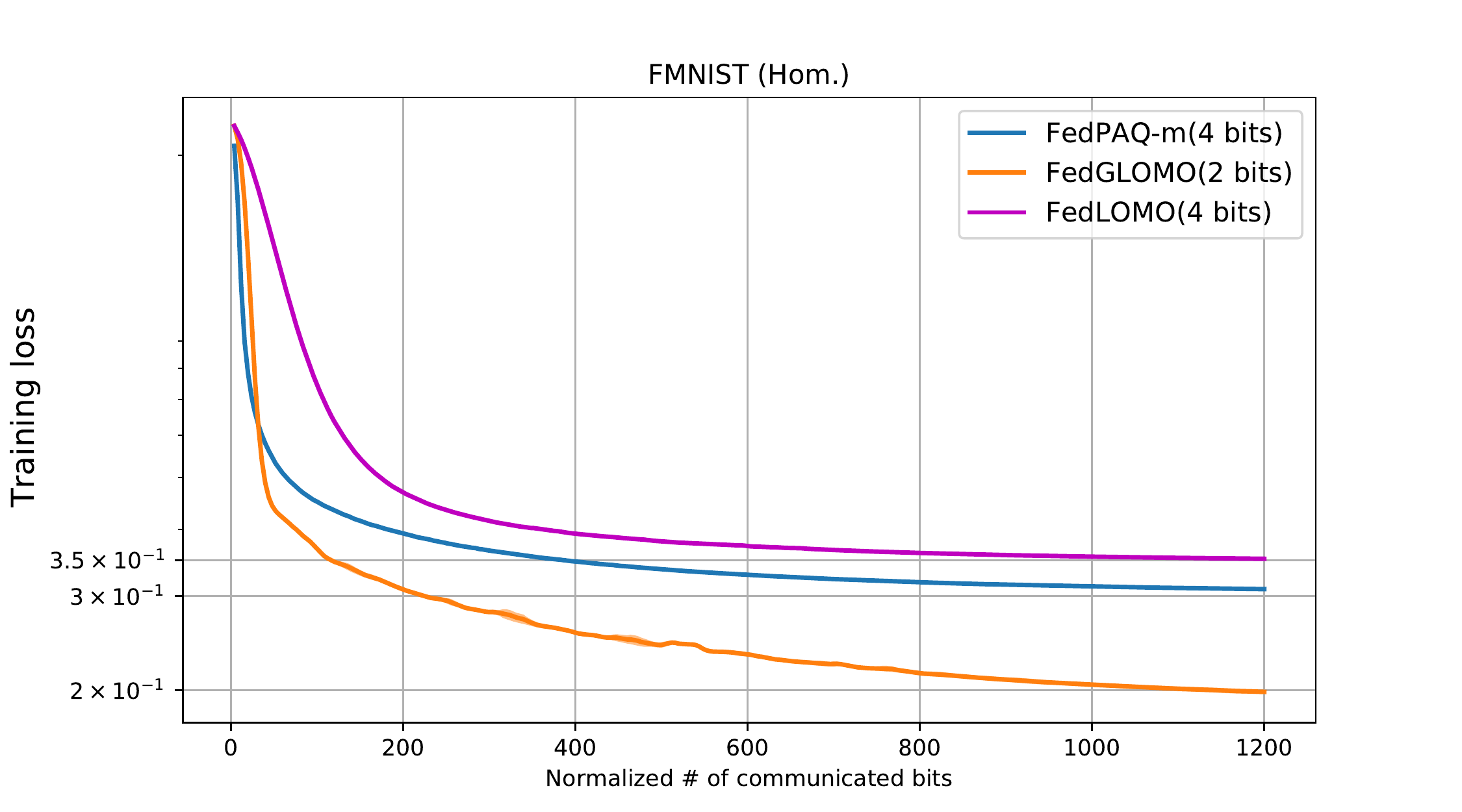}
	} 
\subfloat[Homogeneous FMNIST test error]{
    \label{fig:11_b}
	\includegraphics[width=0.45\textwidth]{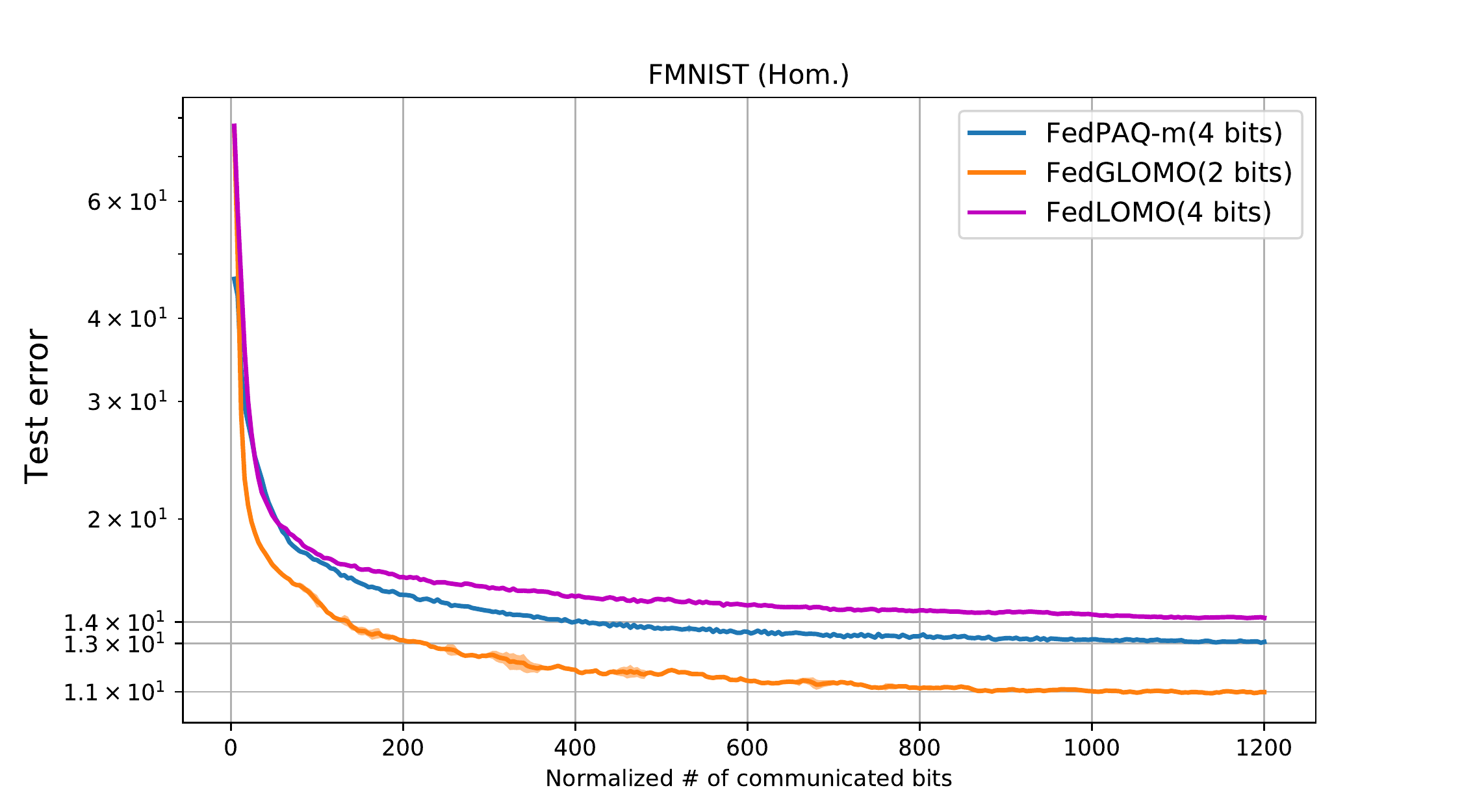}
	} 
\\
\subfloat[Homogeneous CIFAR-10 train loss]{
    \label{fig:11_c}
	\includegraphics[width=0.45\textwidth]{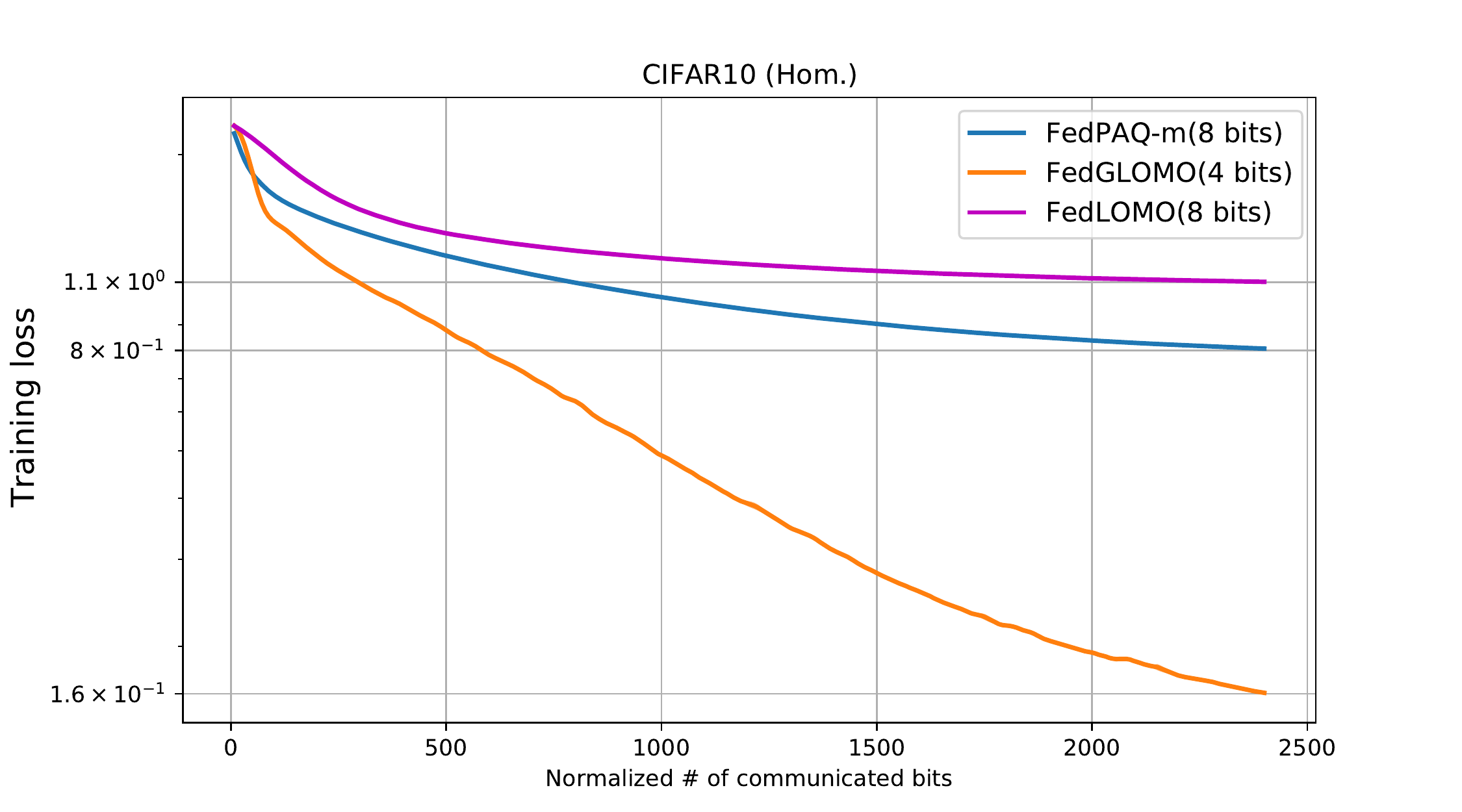}
	} 
\subfloat[Homogeneous CIFAR-10 test error]{
    \label{fig:11_d}
	\includegraphics[width=0.45\textwidth]{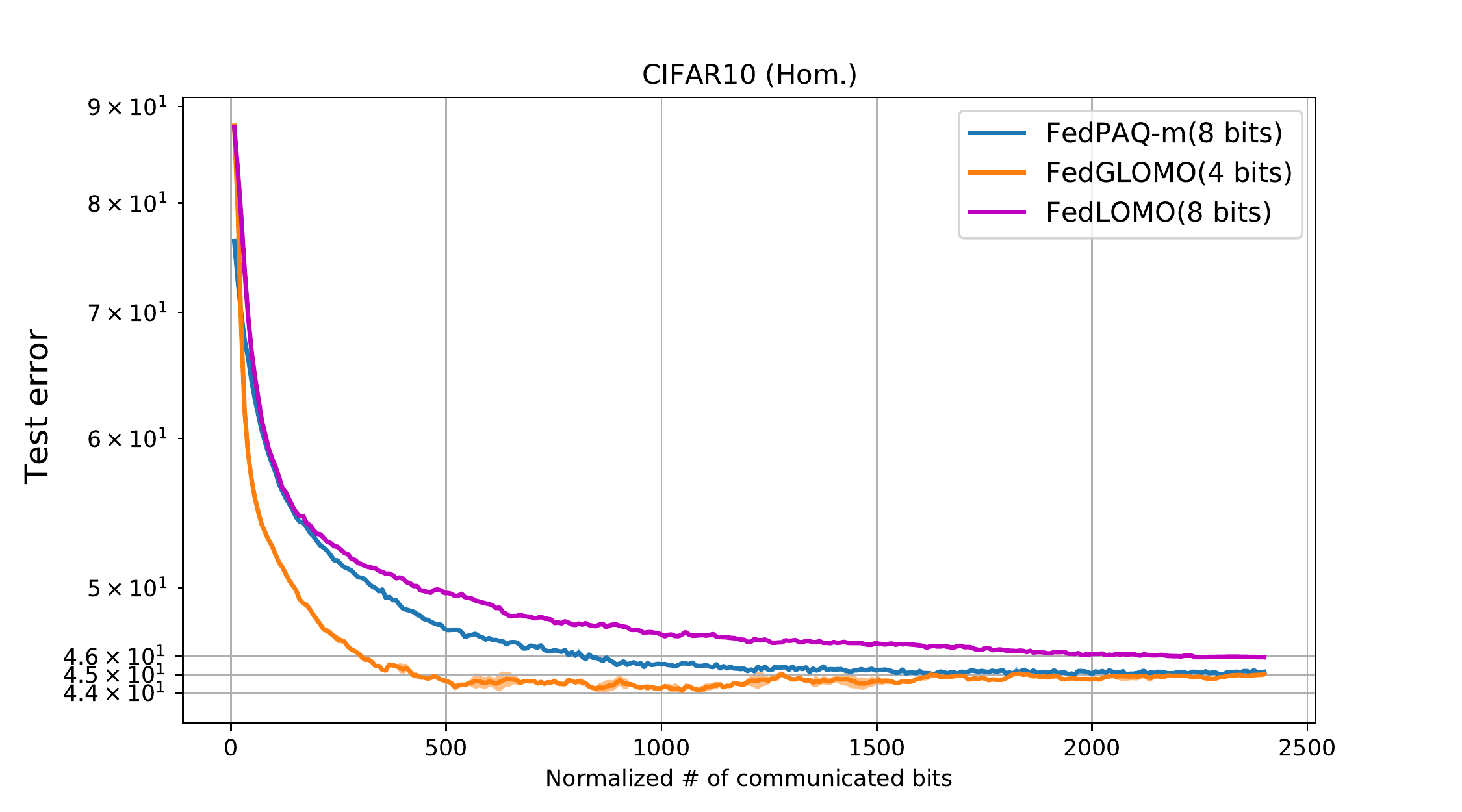}
	} 
\caption{Comparison of \texttt{FedPAQ}-m, \texttt{FedGLOMO} and \texttt{FedLOMO} (recall this has only local momentum and no global momentum) with the same per-round communication budget on FMNIST and CIFAR-10 in the heterogeneous (top four figures) and homogeneous (bottom four figures) settings, respectively. The x-axis is the total number of communicated bits divided by the dimension $d$ and the global batch-size $r$. \texttt{FedLOMO} is the slowest while \texttt{FedGLOMO} is the fastest, showing the ineffectiveness of {only} local momentum and \textit{the power of combining both local and global momentum}.
For both datasets, in the heterogeneous (resp., homogeneous) case, \texttt{FedGLOMO} attains the final test error of \texttt{FedPAQ}-m with only about a \textbf{third} (resp., less than a \textbf{fifth}) of the number of bits used by \texttt{FedPAQ}-m. In the heterogeneous case, \texttt{FedGLOMO} as well as \texttt{FedLOMO} have a smoother trajectory than \texttt{FedPAQ}-m due to the application of variance-reducing momentum. 
}
\label{fig:1}
\end{figure*}

In Fig. \ref{fig:1}, we compare \texttt{FedPAQ}-m and \texttt{FedLOMO} with 4 (resp., 8) bits per-round against \texttt{FedGLOMO} with 2 (resp., 4) bits per-round on FMNIST (resp., CIFAR-10) in the heterogeneous and homogeneous cases.
We set the number of per-round bits used by \texttt{FedPAQ}-m and \texttt{FedLOMO} to be twice that of \texttt{FedGLOMO} so that all algorithms have the same \textit{per-round} communication budget. All plots depict results over 3 independent runs; the shaded regions represent $\pm 1$ standard deviation whereas the solid lines are the respective means. Please see the discussion in the figure caption. Having shown the suboptimality of \texttt{FedLOMO} in Fig. \ref{fig:1}, we compare \texttt{FedAvg}-m, \texttt{FedGLOMO} without compression and \texttt{MimeSGDm} in the heterogeneous case in Fig. \ref{fig:2}. These experiments illustrate the \textit{power of global momentum}. Also see \Cref{extra_res} for additional experiments.

\begin{figure*}[!htb]
\centering 
\subfloat[FMNIST train loss]{
    \label{fig:2_a}
	\includegraphics[width=0.45\textwidth]{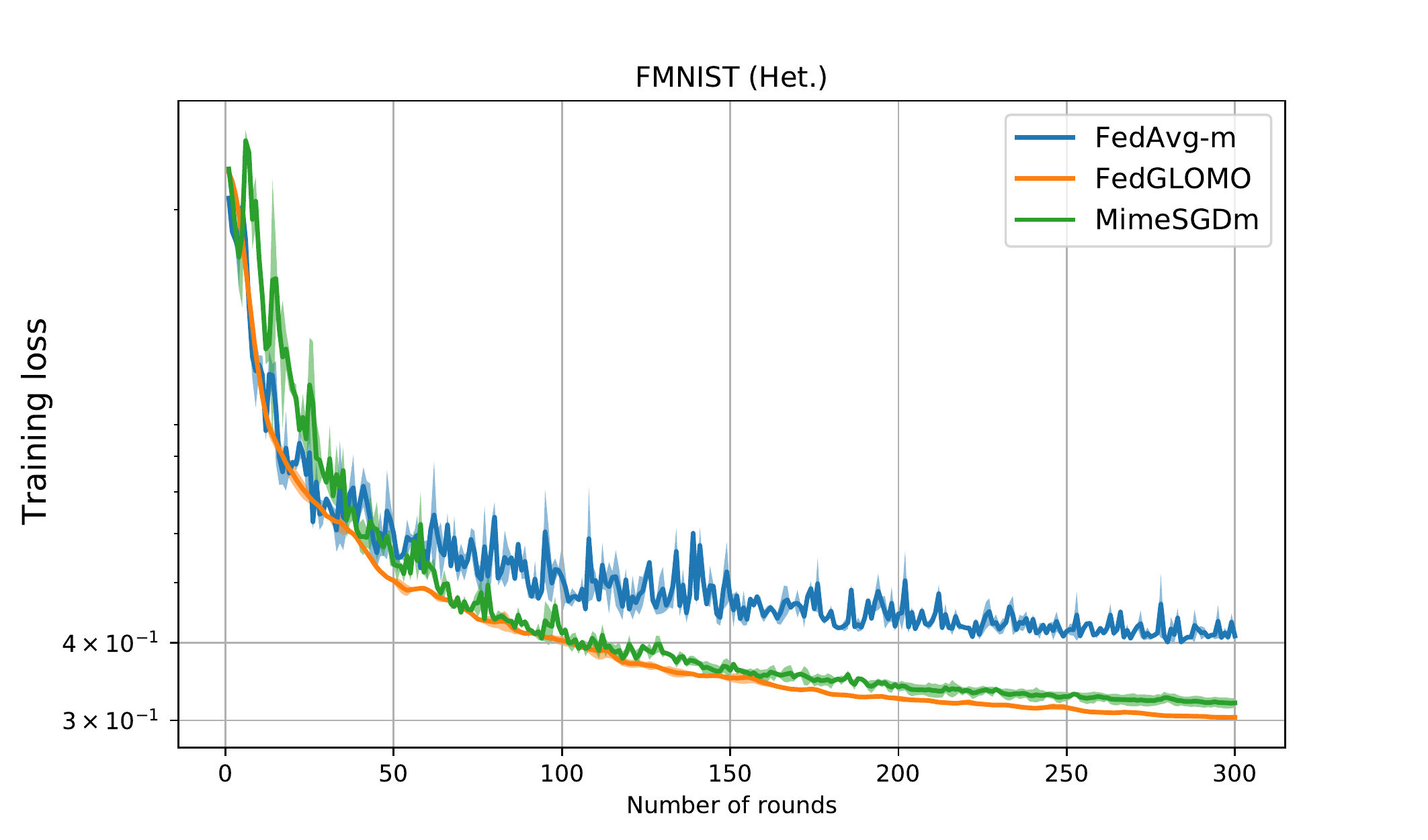}
	} 
\subfloat[FMNIST test err]{
    \label{fig:2_b}
	\includegraphics[width=0.45\textwidth]{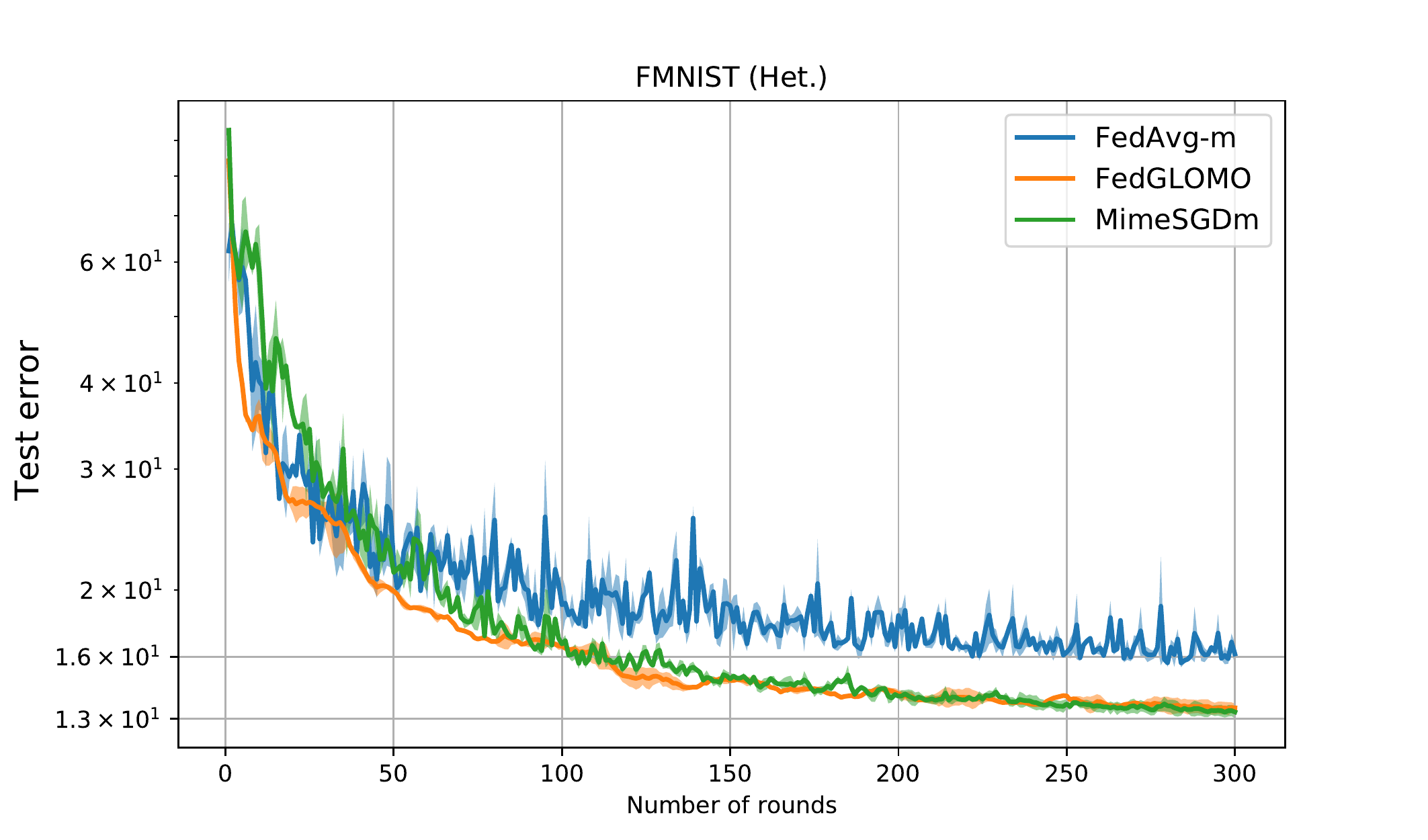}
	} 
\\
\subfloat[CIFAR-10 train loss]{
    \label{fig:2_c}
	\includegraphics[width=0.45\textwidth]{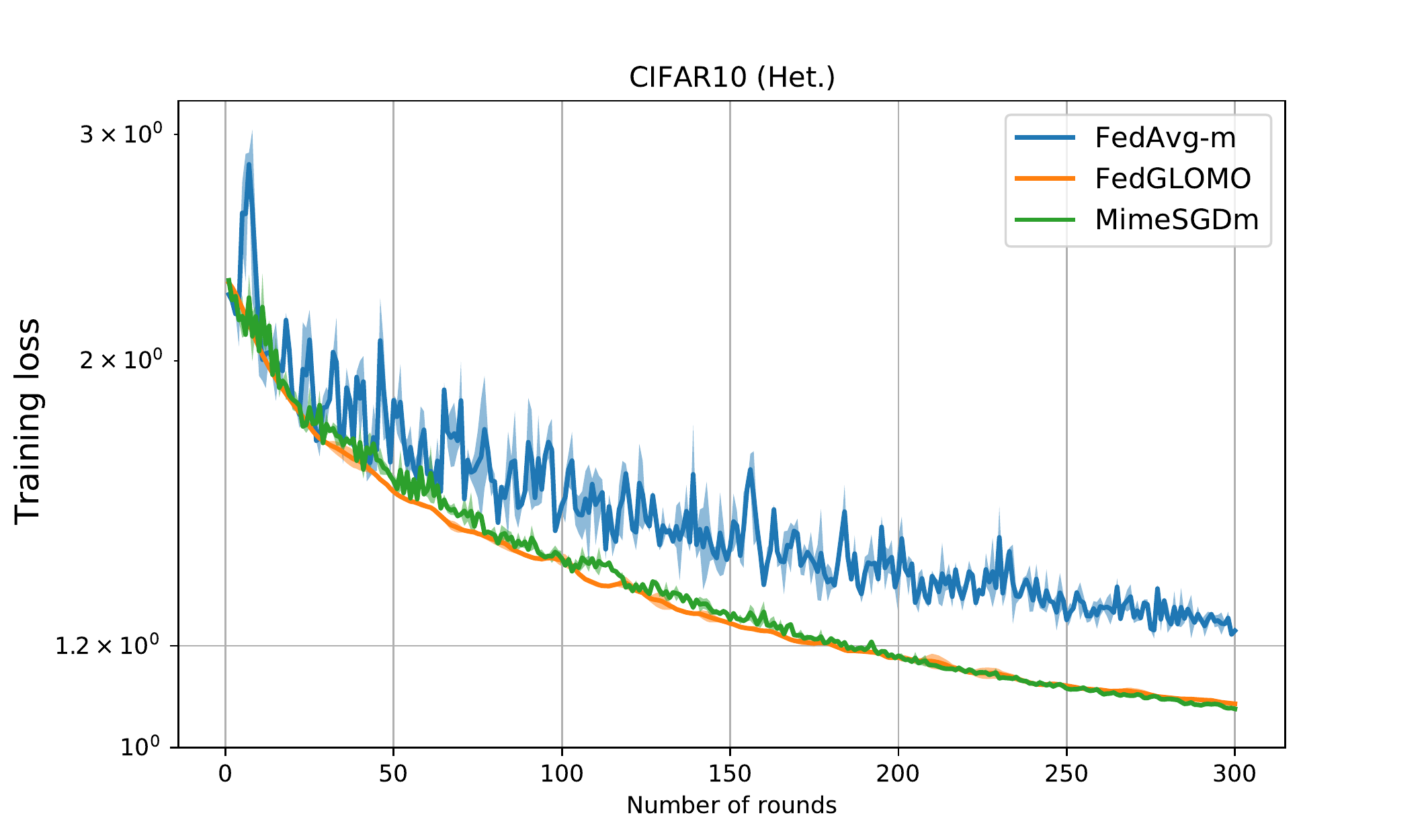}
	} 
\subfloat[CIFAR-10 test err]{
    \label{fig:2_d}
	\includegraphics[width=0.45\textwidth]{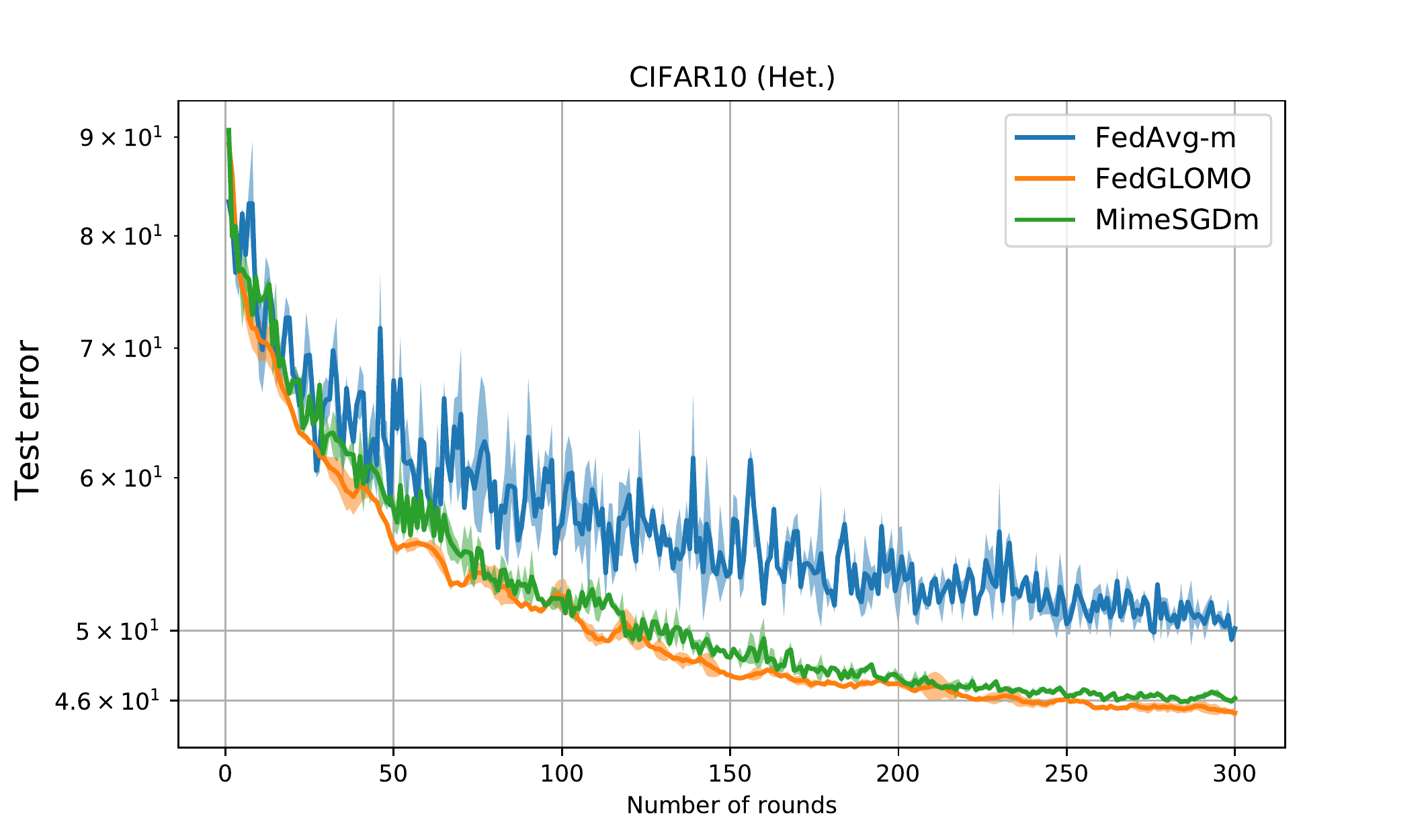}
	} 
\caption{\textbf{Heterogeneous case:} Comparison of \texttt{FedAvg}-m, \texttt{FedGLOMO} (without compression) and \texttt{MimeSGDm} on FMNIST (top) and CIFAR-10 (bottom). 
On both datasets, \texttt{FedAvg}-m is the slowest while \texttt{FedGLOMO} is somewhat faster than \texttt{MimeSGDm}. While \texttt{Mime} has an explicit client-drift control mechanism, we do not have that in \texttt{FedGLOMO}, but still \textit{our proposed global momentum implicitly mitigates client-drift}.}
\label{fig:2}
\end{figure*}

\subsection{Verifying 
\texorpdfstring{\Cref{as-het}}{Lg} for \texttt{FedGLOMO}}
\label{sec:het-asm-expt}
We compute $\alpha = \max_{\tau \in [E]} 
\frac{\|\sum_{i \in [n]}\widetilde{\bm{e}}_{\tau}^{(i)}\|^2}{\sum_{i \in [n]} \|\widetilde{\bm{e}}_{\tau}^{(i)}\|^2}$, where $\widetilde{\bm{e}}_{\tau}^{(i)}$ is as defined in \Cref{as-het}, for 4 and 2 bit \texttt{FedGLOMO} on CIFAR-10 and FMNIST, respectively. 
Note that we remove the expectation (w.r.t. the stochastic gradients) while computing $\alpha$ for empirical verification. In Fig. \ref{fig:het0}, we plot $(\alpha/n)$ over different rounds for the heterogeneous as well as homogeneous case on both datasets; see the discussion in the figure caption.
\begin{figure}[!htb]
\centering 
\subfloat[CIFAR-10]{
    \label{fig:het_a}
	\includegraphics[width=0.45\textwidth]{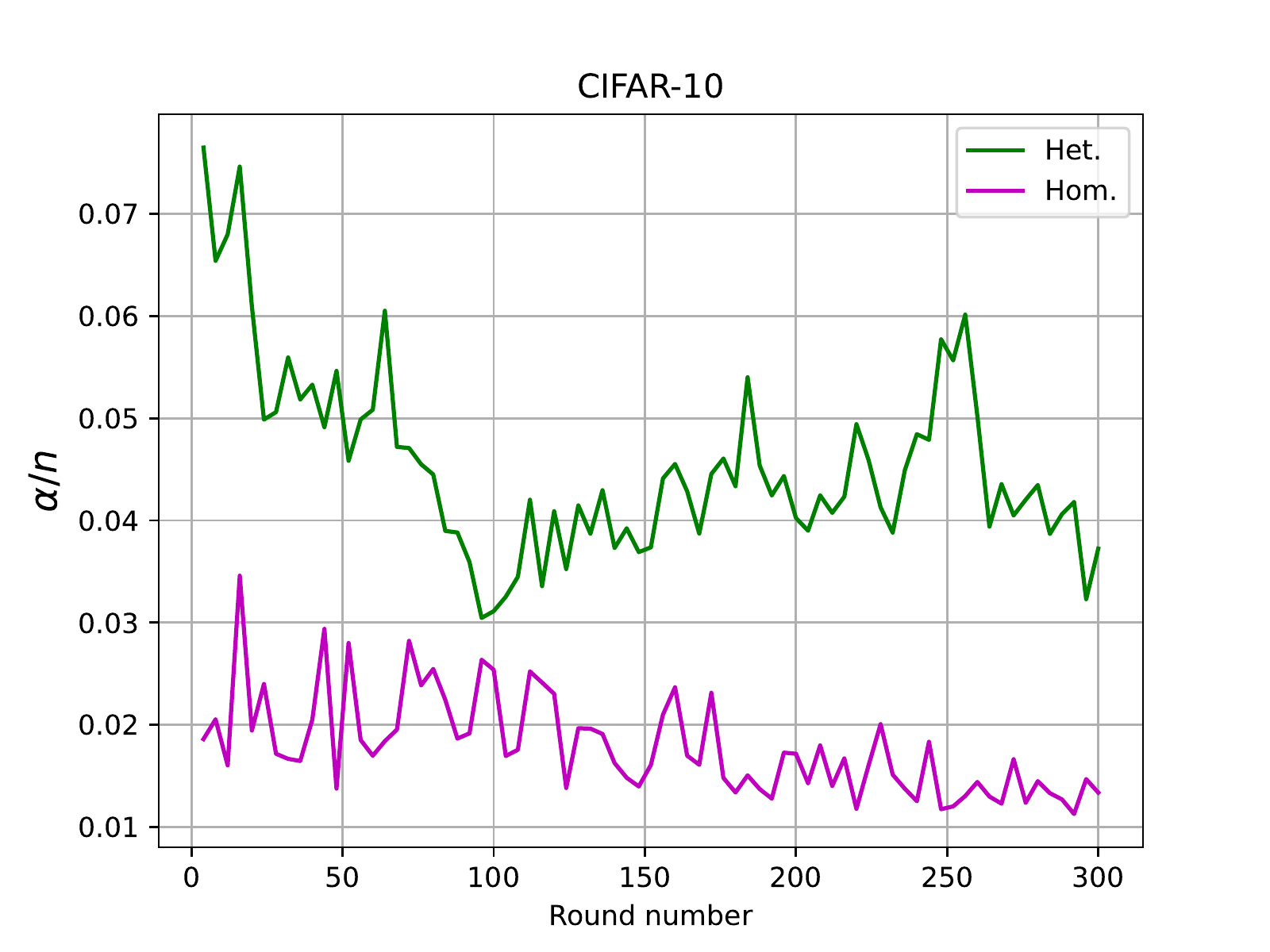}
	} 
\subfloat[FMNIST]{
    \label{fig:het_b}
	\includegraphics[width=0.45\textwidth]{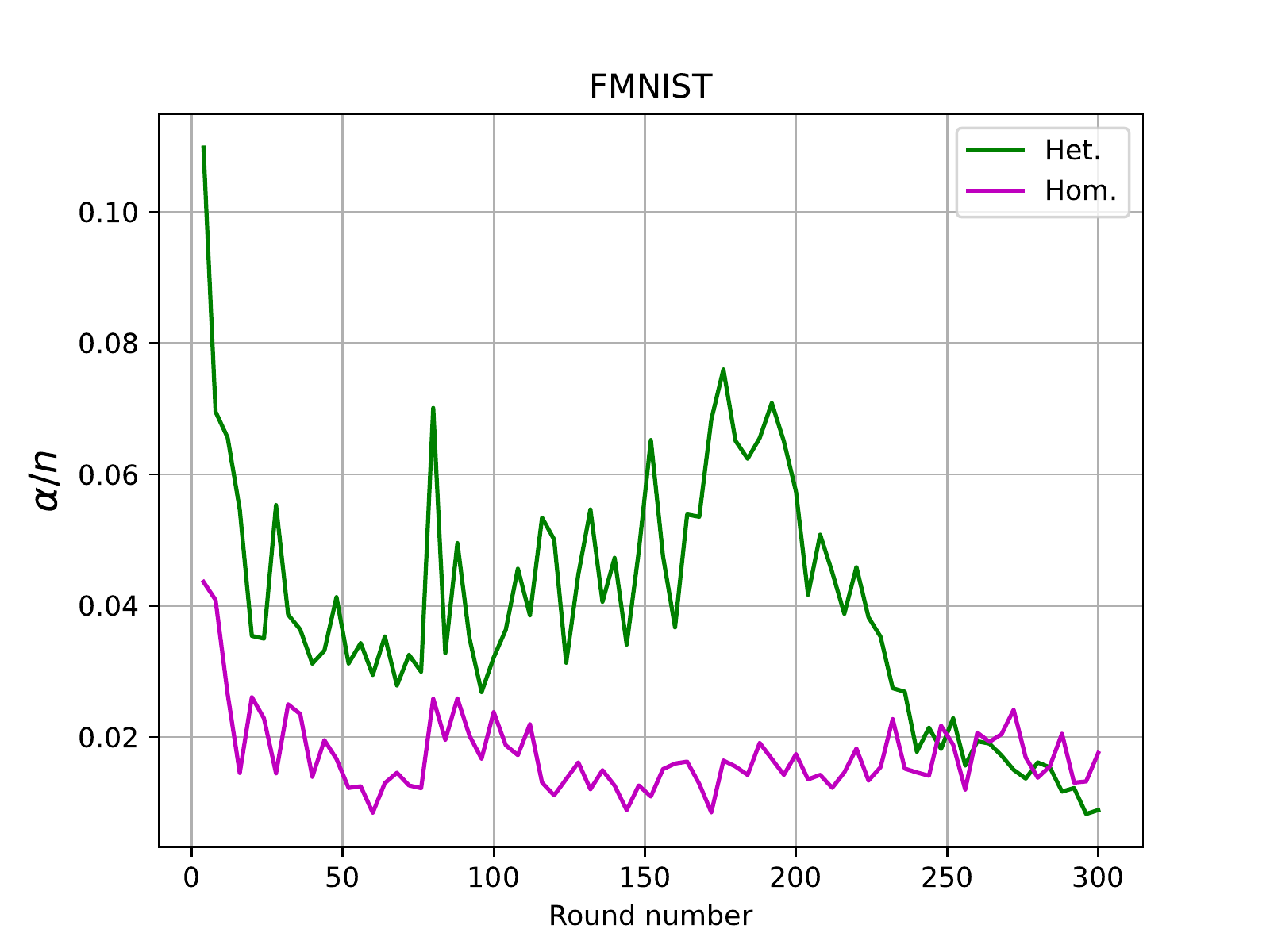}
	} 
\caption{Variation of $(\frac{\alpha}{n})$ over different rounds of $4$ and $2$ bit \texttt{FedGLOMO} for CIFAR-10 (Fig. \ref{fig:het_a}) and FMNIST (Fig. \ref{fig:het_b}) in the heterogeneous and homogeneous cases. In both cases, notice that $\alpha \ll n$ throughout training. {Also, as discussed after the statement of 
\Cref{as-het}, observe that $(\frac{\alpha}{n})$ is higher for the heterogeneous case (except towards the end of training for FMNIST).}
}
\label{fig:het0}
\end{figure}
\clearpage
\section{Conclusion}
{We presented \texttt{FedGLOMO}, a communication-efficient algorithm for faster federated learning (FL) via the application of variance-reducing momentum, both in the aggregation step at the server as well as local client updates. We showed that \texttt{FedGLOMO} has better iteration complexity than prior work on smooth non-convex functions with compressed communication. Further, unlike prior work, our result is under \Cref{as-het}, which is a novel and verifiable client-heterogeneity assumption, even allowing for arbitrary client heterogeneity.
We also demonstrate the efficacy of \texttt{FedGLOMO} via extensive experiments.

Apart from addressing the limitation discussed in \Cref{rem-sep29-5}, there are several avenues of future work possible such as verifying \Cref{as-het} for other FL algorithms and deriving convergence results based on it, obtaining lower bounds on the iteration  complexity of non-convex FL, coming up with an error-compensated version of \texttt{FedGLOMO} for biased compressors, etc.}

\section{Acknowledgement}
This work is supported in part by NSF grants CCF-1564000, IIS-1546452 and HDR-1934932.

\bibliography{bib}

\begin{thebibliography}{10}

\bibitem{alistarh2017qsgd}
{\sc Alistarh, D., Grubic, D., Li, J., Tomioka, R., and Vojnovic, M.}
\newblock Qsgd: Communication-efficient sgd via gradient quantization and
  encoding.
\newblock In {\em Advances in Neural Information Processing Systems\/} (2017),
  pp.~1709--1720.

\bibitem{alistarh2018convergence}
{\sc Alistarh, D., Hoefler, T., Johansson, M., Konstantinov, N., Khirirat, S.,
  and Renggli, C.}
\newblock The convergence of sparsified gradient methods.
\newblock In {\em Advances in Neural Information Processing Systems\/} (2018),
  pp.~5973--5983.

\bibitem{arjevani2019lower}
{\sc Arjevani, Y., Carmon, Y., Duchi, J.~C., Foster, D.~J., Srebro, N., and
  Woodworth, B.}
\newblock Lower bounds for non-convex stochastic optimization.
\newblock {\em arXiv preprint arXiv:1912.02365\/} (2019).

\bibitem{basu2019qsparse}
{\sc Basu, D., Data, D., Karakus, C., and Diggavi, S.}
\newblock Qsparse-local-sgd: Distributed sgd with quantization, sparsification
  and local computations.
\newblock In {\em Advances in Neural Information Processing Systems\/} (2019),
  pp.~14695--14706.

\bibitem{bayoumi2020tighter}
{\sc Bayoumi, A. K.~R., Mishchenko, K., and Richt{\'a}rik, P.}
\newblock Tighter theory for local sgd on identical and heterogeneous data.
\newblock In {\em International Conference on Artificial Intelligence and
  Statistics\/} (2020), pp.~4519--4529.

\bibitem{bernstein2018signsgd}
{\sc Bernstein, J., Wang, Y.-X., Azizzadenesheli, K., and Anandkumar, A.}
\newblock signsgd: Compressed optimisation for non-convex problems.
\newblock {\em arXiv preprint arXiv:1802.04434\/} (2018).

\bibitem{bottou2012stochastic}
{\sc Bottou, L.}
\newblock Stochastic gradient descent tricks.
\newblock In {\em Neural networks: Tricks of the trade}. Springer, 2012,
  pp.~421--436.

\bibitem{chen2020communication}
{\sc Chen, Y., Hashemi, A., and Vikalo, H.}
\newblock Communication-efficient algorithms for decentralized optimization
  over directed graphs.
\newblock {\em arXiv preprint arXiv:2005.13189\/} (2020).

\bibitem{cutkosky2019momentum}
{\sc Cutkosky, A., and Orabona, F.}
\newblock Momentum-based variance reduction in non-convex sgd.
\newblock In {\em Advances in Neural Information Processing Systems\/} (2019),
  pp.~15236--15245.

\bibitem{fang2018spider}
{\sc Fang, C., Li, C.~J., Lin, Z., and Zhang, T.}
\newblock Spider: Near-optimal non-convex optimization via stochastic
  path-integrated differential estimator.
\newblock In {\em Advances in Neural Information Processing Systems\/} (2018),
  pp.~689--699.

\bibitem{gorbunov2021marina}
{\sc Gorbunov, E., Burlachenko, K., Li, Z., and Richt{\'a}rik, P.}
\newblock Marina: Faster non-convex distributed learning with compression.
\newblock {\em arXiv preprint arXiv:2102.07845\/} (2021).

\bibitem{haddadpour2020federated}
{\sc Haddadpour, F., Kamani, M.~M., Mokhtari, A., and Mahdavi, M.}
\newblock Federated learning with compression: Unified analysis and sharp
  guarantees.
\newblock {\em arXiv preprint arXiv:2007.01154\/} (2020).

\bibitem{hashemi2020delicoco}
{\sc Hashemi, A., Acharya, A., Das, R., Vikalo, H., Sanghavi, S., and Dhillon,
  I.}
\newblock On the benefits of multiple gossip steps in communication-constrained
  decentralized optimization.
\newblock {\em arXiv\/} (2020).

\bibitem{horvath2019stochastic}
{\sc Horv{\'a}th, S., Kovalev, D., Mishchenko, K., Stich, S., and
  Richt{\'a}rik, P.}
\newblock Stochastic distributed learning with gradient quantization and
  variance reduction.
\newblock {\em arXiv preprint arXiv:1904.05115\/} (2019).

\bibitem{huo2020faster}
{\sc Huo, Z., Yang, Q., Gu, B., Huang, L.~C., et~al.}
\newblock Faster on-device training using new federated momentum algorithm.
\newblock {\em arXiv preprint arXiv:2002.02090\/} (2020).

\bibitem{karimireddy2020mime}
{\sc Karimireddy, S.~P., Jaggi, M., Kale, S., Mohri, M., Reddi, S.~J., Stich,
  S.~U., and Suresh, A.~T.}
\newblock Mime: Mimicking centralized stochastic algorithms in federated
  learning.
\newblock {\em arXiv preprint arXiv:2008.03606\/} (2020).

\bibitem{karimireddy2019scaffold}
{\sc Karimireddy, S.~P., Kale, S., Mohri, M., Reddi, S.~J., Stich, S.~U., and
  Suresh, A.~T.}
\newblock Scaffold: Stochastic controlled averaging for federated learning.
\newblock {\em arXiv preprint arXiv:1910.06378\/} (2019).

\bibitem{koloskova2020unified}
{\sc Koloskova, A., Loizou, N., Boreiri, S., Jaggi, M., and Stich, S.}
\newblock A unified theory of decentralized sgd with changing topology and
  local updates.
\newblock In {\em International Conference on Machine Learning\/} (2020), PMLR,
  pp.~5381--5393.

\bibitem{li2018federated}
{\sc Li, T., Sahu, A.~K., Zaheer, M., Sanjabi, M., Talwalkar, A., and Smith,
  V.}
\newblock Federated optimization in heterogeneous networks.
\newblock {\em arXiv preprint arXiv:1812.06127\/} (2018).

\bibitem{li2019convergence}
{\sc Li, X., Huang, K., Yang, W., Wang, S., and Zhang, Z.}
\newblock On the convergence of fedavg on non-iid data.
\newblock {\em arXiv preprint arXiv:1907.02189\/} (2019).

\bibitem{liang2019variance}
{\sc Liang, X., Shen, S., Liu, J., Pan, Z., Chen, E., and Cheng, Y.}
\newblock Variance reduced local sgd with lower communication complexity.
\newblock {\em arXiv preprint arXiv:1912.12844\/} (2019).

\bibitem{lin2017deep}
{\sc Lin, Y., Han, S., Mao, H., Wang, Y., and Dally, W.~J.}
\newblock Deep gradient compression: Reducing the communication bandwidth for
  distributed training.
\newblock {\em arXiv preprint arXiv:1712.01887\/} (2017).

\bibitem{liu2020optimal}
{\sc Liu, D., Nguyen, L.~M., and Tran-Dinh, Q.}
\newblock An optimal hybrid variance-reduced algorithm for stochastic composite
  nonconvex optimization.
\newblock {\em arXiv preprint arXiv:2008.09055\/} (2020).

\bibitem{mcmahan2017communication}
{\sc McMahan, B., Moore, E., Ramage, D., Hampson, S., and y~Arcas, B.~A.}
\newblock Communication-efficient learning of deep networks from decentralized
  data.
\newblock In {\em Artificial Intelligence and Statistics\/} (2017), PMLR,
  pp.~1273--1282.

\bibitem{patel2019communication}
{\sc Patel, K.~K., and Dieuleveut, A.}
\newblock Communication trade-offs for synchronized distributed sgd with large
  step size.
\newblock {\em arXiv preprint arXiv:1904.11325\/} (2019).

\bibitem{pu2020distributed}
{\sc Pu, S., and Nedi{\'c}, A.}
\newblock Distributed stochastic gradient tracking methods.
\newblock {\em Mathematical Programming\/} (2020), 1--49.

\bibitem{qu2020federated}
{\sc Qu, Z., Lin, K., Kalagnanam, J., Li, Z., Zhou, J., and Zhou, Z.}
\newblock Federated learning's blessing: Fedavg has linear speedup.
\newblock {\em arXiv preprint arXiv:2007.05690\/} (2020).

\bibitem{reddi2020adaptive}
{\sc Reddi, S., Charles, Z., Zaheer, M., Garrett, Z., Rush, K.,
  Kone{\v{c}}n{\`y}, J., Kumar, S., and McMahan, H.~B.}
\newblock Adaptive federated optimization.
\newblock {\em arXiv preprint arXiv:2003.00295\/} (2020).

\bibitem{reisizadeh2020fedpaq}
{\sc Reisizadeh, A., Mokhtari, A., Hassani, H., Jadbabaie, A., and Pedarsani,
  R.}
\newblock Fedpaq: A communication-efficient federated learning method with
  periodic averaging and quantization.
\newblock In {\em International Conference on Artificial Intelligence and
  Statistics\/} (2020), pp.~2021--2031.

\bibitem{stich2018local}
{\sc Stich, S.~U.}
\newblock Local sgd converges fast and communicates little.
\newblock {\em arXiv preprint arXiv:1805.09767\/} (2018).

\bibitem{stich2018sparsified}
{\sc Stich, S.~U., Cordonnier, J.-B., and Jaggi, M.}
\newblock Sparsified sgd with memory.
\newblock In {\em Advances in Neural Information Processing Systems\/} (2018),
  pp.~4447--4458.

\bibitem{stich2019error}
{\sc Stich, S.~U., and Karimireddy, S.~P.}
\newblock The error-feedback framework: Better rates for sgd with delayed
  gradients and compressed communication.
\newblock {\em arXiv preprint arXiv:1909.05350\/} (2019).

\bibitem{suresh2017distributed}
{\sc Suresh, A.~T., Felix, X.~Y., Kumar, S., and McMahan, H.~B.}
\newblock Distributed mean estimation with limited communication.
\newblock In {\em International Conference on Machine Learning\/} (2017),
  pp.~3329--3337.

\bibitem{tang2018communication}
{\sc Tang, H., Gan, S., Zhang, C., Zhang, T., and Liu, J.}
\newblock Communication compression for decentralized training.
\newblock In {\em Advances in Neural Information Processing Systems\/} (2018),
  pp.~7652--7662.

\bibitem{wang2018cooperative}
{\sc Wang, J., and Joshi, G.}
\newblock Cooperative sgd: A unified framework for the design and analysis of
  communication-efficient sgd algorithms.
\newblock {\em arXiv preprint arXiv:1808.07576\/} (2018).

\bibitem{wang2019slowmo}
{\sc Wang, J., Tantia, V., Ballas, N., and Rabbat, M.}
\newblock Slowmo: Improving communication-efficient distributed sgd with slow
  momentum.
\newblock {\em arXiv preprint arXiv:1910.00643\/} (2019).

\bibitem{woodworth2020local}
{\sc Woodworth, B., Patel, K.~K., Stich, S.~U., Dai, Z., Bullins, B., McMahan,
  H.~B., Shamir, O., and Srebro, N.}
\newblock Is local sgd better than minibatch sgd?
\newblock {\em arXiv preprint arXiv:2002.07839\/} (2020).

\bibitem{wu2018error}
{\sc Wu, J., Huang, W., Huang, J., and Zhang, T.}
\newblock Error compensated quantized sgd and its applications to large-scale
  distributed optimization.
\newblock {\em arXiv preprint arXiv:1806.08054\/} (2018).

\bibitem{xiao2017fashion}
{\sc Xiao, H., Rasul, K., and Vollgraf, R.}
\newblock Fashion-mnist: a novel image dataset for benchmarking machine
  learning algorithms.
\newblock {\em arXiv preprint arXiv:1708.07747\/} (2017).

\bibitem{yu2018parallel}
{\sc Yu, H., Yang, S., and Zhu, S.}
\newblock Parallel restarted sgd for non-convex optimization with faster
  convergence and less communication.
\newblock {\em arXiv preprint arXiv:1807.06629 2}, 4 (2018), 7.

\bibitem{zhou2018stochastic}
{\sc Zhou, D., Xu, P., and Gu, Q.}
\newblock Stochastic nested variance reduced gradient descent for nonconvex
  optimization.
\newblock {\em Advances in neural information processing systems\/} (2018).

\bibitem{zinkevich2010parallelized}
{\sc Zinkevich, M., Weimer, M., Li, L., and Smola, A.}
\newblock Parallelized stochastic gradient descent.
\newblock {\em Advances in neural information processing systems 23\/} (2010),
  2595--2603.

\end{thebibliography}
\bibliographystyle{acm}

\clearpage

\appendix

\clearpage
\onecolumn
\appendix

\begin{center}
    \textbf{\Large Appendix}\vspace{5mm}
\end{center}

\section{Reduction in total communication
cost when \texorpdfstring{$r \ll n$}{Lg}} 
\label{sec:red_bits}
Here, we derive the claim made in the second paragraph of \Cref{rem-sep21-3}. We consider the practical regime of $r \ll n$ (as well as $\alpha \ll n$).

First, consider the case where the clients communicate at full precision using 32 bits, i.e., $q=0$. The number of rounds of communication $K_1$ needed to reach an $\epsilon$ stationary point is:
\begin{flalign*}
    K_1 \approx \Big(\frac{39L f(\bm{w}_0)}{\epsilon}\Big)^{1.5} \Big(800 e^2 (E+1)^2 \frac{(n-r)}{r(n-1)}\Big)^{0.5}
\end{flalign*}
Since the communication cost per-round is proportional to $r \times (32d)$ bits (recall $d$ is the model dimension), the total communication cost $C_1$ in this case is:
\begin{flalign*}
    C_1 & = 32dr \times K_1 
    \\
    & = 32dr \Big(\frac{39L f(\bm{w}_0)}{\epsilon}\Big)^{1.5} \Big(800 e^2 (E+1)^2 \frac{(n-r)}{r(n-1)}\Big)^{0.5}.
\end{flalign*}
Now, let us consider the QSGD quantizer of \cite{alistarh2017qsgd} with $s = \sqrt{d}$ \footnote{\cite{alistarh2017qsgd} uses $n$ to denote the dimension}. With this choice, $q=1$. Here, the number of rounds of communication $K_2$ needed to reach an $\epsilon$ stationary point is:
\begin{flalign*}
    K_2 \approx \Big(\frac{39L f(\bm{w}_0)}{\epsilon}\Big)^{1.5} \Big(3200 e^2 (E+1)^2 \frac{(n-r)}{r(n-1)} \Big)^{0.5}.
\end{flalign*}
Now using Theorem 3.4 of \cite{alistarh2017qsgd}, under the special case of $s = \sqrt{d}$, the communication cost per-round can be reduced to $r \times (2.8d + 32)$ bits. Hence, the total communication cost $C_2$ in this case is:
\begin{flalign*}
    C_2 & \approx (2.8d + 32) r \times K_2 
    \\
    & \approx (2.8d + 32) r \Big(\frac{39L f(\bm{w}_0)}{\epsilon}\Big)^{1.5} \Big(3200 e^2 (E+1)^2 \frac{(n-r)}{r(n-1)} \Big)^{0.5}.
\end{flalign*}
Therefore,
\[\frac{C_1}{C_2} \approx \frac{32}{2.8 \times 4^{0.5}} \approx 5.7.\]

\section{Algorithmic and Theoretical Comparison with MIME \texorpdfstring{\cite{karimireddy2020mime}}{Lg}}
\label{sec:disc}
We now discuss the major algorithmic and theoretical differences of our work from  \cite{karimireddy2020mime}. 

\begin{itemize}
    \item Algorithmically, \cite{karimireddy2020mime} do not explicitly \textit{apply} any momentum at the server. Instead, they apply globally computed momentum in the local updates of the clients.
    On the other hand, \texttt{FedGLOMO} has an explicit momentum-based update at the server to enable global variance reduction, apart from local momentum applied in the client updates.
    \item Unlike \texttt{FedGLOMO}, the algorithms of \cite{karimireddy2020mime} do not have any quantized/compressed communication. As we discussed in \Cref{rem-sep21-3}, 
    maintaining the improved complexity of $\mathcal{O}(\epsilon^{-1.5})$ with compressed communication is not trivial.
    \item Even in the absence of any compressed communication, \texttt{FedGLOMO} is more communication-efficient than Mime requiring three-fourth / half the number of bits that Mime requires per-round for server to clients as well as clients to server communication / only clients to server communication (which is typically the bottleneck in FL). This is because in Mime, the server needs to send $\bm{x}$ (sending some other statistics $\bm{s}$ would require even more bits) and $\bm{c}$ to the clients, and the clients need to send back $(\bm{y}_i,\nabla f_i(\bm{x})$) to the server (please see their notation). In \texttt{FedGLOMO}, the server needs to send $\bm{w}_k$ and $\bm{w}_{k-1}$ to the clients, but the clients can just send back $\{{(\bm{w}_{k} - {\bm{w}_{k,E}^{(i)}})}  - (1-\beta_k){({\bm{w}_{k-1} - \widehat{\bm{w}}_{k-1,E}^{(i)}})}\}$ to the server in the absence of any quantization -- this can be verified by just removing the quantization operator $Q_D$ and expanding the update rule of $\bm{u}_k$ (line 10 of \Cref{alg:2}) for $k>0$.
    \item Our theory for \texttt{FedGLOMO} does not use the bounded client dissimilarity (BCD) assumption, i.e., \cref{eq:bcd}.
    Instead, we propose and use \Cref{as-het}, which allows for arbitrary client heterogeneity; in the worst case, \Cref{as-het} will hold with $\alpha = n$. In contrast, the results of MimeMVR use the BCD assumption.
    \item See the full version of Theorem V (on page 39 of the latest arXiv draft) of \cite{karimireddy2020mime} for MimeMVR. Their result is in terms of the gradient of $f$ at the local client parameters and not the actual server parameters, which is not ideal. Our result for \texttt{FedGLOMO} is completely in terms of the gradient of $f$ at the server parameters.
\end{itemize}

\section{Comparison with \texorpdfstring{\cite{gorbunov2021marina}}{Lg}}
\label{sec:marina}
As mentioned in \Cref{rel-work}, \cite{gorbunov2021marina} also propose algorithms with improved complexity in the \textit{distributed setting without multiple local update steps}. 
Since our work is under partial-device participation, we compare against their algorithm for the same case, i.e. PP-MARINA. Note that PP-MARINA has a probability $p$ of using \textit{full gradients from all clients} (i.e., full device participation) in each iteration. For a fair comparison against \texttt{FedGLOMO}, which uses gradients from all the clients \textit{only in the first round} (see \Cref{nov4-thm1}), $p$ should be set to $\frac{1}{K}$ ($K$ being the number of rounds) -- in which case, their complexity is $O(\epsilon^{-2})$ which is worse than ours. See Theorem 4.1 in their paper for this.

\section{Experimental Details and More Experiments}
\label{sec:extra-exp}
We first describe the procedure we have used to generate heterogeneous data distribution (among the clients). First, the training data (of both CIFAR-10 and FMNIST) was sorted based on labels and then divided into 100 equal data-shards. Splitting the data into 100 equal shards (after sorting) ensures that each shard contains data from only one class for both CIFAR-10 and FMNIST. Since the number of clients in our experiments is fixed to 50, each client is assigned 2 shards chosen uniformly at random without replacement -- this ensures that each client can have data belonging to either just one class or two classes at the most. For the homogeneous case, we distribute the training data uniformly at random among the clients.

In all our experiments (including the ones in \Cref{extra_res}), we use the learning rate schedule suggested in \cite{bottou2012stochastic} where we decimate the client learning rate by 1\% after every round, i.e., $\eta_k = (0.99)^k \eta_0$, $\eta_0$ being the initial learning rate. Note that this learning rate schedule has been used earlier for FL experiments in \cite{haddadpour2020federated}.
We search the initial learning rates over $\{10^{-3}, 5 \times 10^{-3}, 10^{-2}, 5 \times 10^{-2}, 10^{-1}\}$; 
the best performance is obtained with an initial learning rate of $10^{-2}$ in almost all the cases. 

We make a small modification to \texttt{FedGLOMO} in our experiments for the heterogeneous case. Specifically, we modify line \ref{l2} (which is the local momentum application step) of \Cref{alg:2-local} as follows:
\[\text{Update: } \bm{v}_{k,\tau}^{(i)} = \widetilde{\nabla} f_i(\bm{w}_{k,\tau}^{(i)};\mathcal{B}_{k,\tau}^{(i)}) + {\color{blue} 0.8}\big(\bm{v}_{k,\tau-1}^{(i)} - \widetilde{\nabla} f_i(\bm{w}_{k,\tau-1}^{(i)};\mathcal{B}_{k,\tau}^{(i)})\big) \text{ and}\]
\[\widehat{\bm{v}}_{k-1,\tau}^{(i)} = \widetilde{\nabla} f_i(\widehat{\bm{w}}_{k-1,\tau}^{(i)};\mathcal{B}_{k,\tau}^{(i)}) + {\color{blue} 0.8}\big(\widehat{\bm{v}}_{k-1,\tau-1}^{(i)} - \widetilde{\nabla} f_i(\widehat{\bm{w}}_{k-1,\tau-1}^{(i)};\mathcal{B}_{k,\tau}^{(i)})\big).\]
Without applying the above damping factor of 0.8, \texttt{FedGLOMO} seems to diverge -- this is probably because we have chosen the number of local updates to be too large.

All experiments are run on a single NVIDIA TITAN Xp GPU.

\subsection{More experimental results:}
\label{extra_res}
Having shown the suboptimality of \texttt{FedLOMO} in the experiments of the main paper (i.e., in \Cref{sec:exp}), we do not compare against \texttt{FedLOMO} anymore. For all the experiments, \textit{the figure captions discuss the results in detail}.
\\
\\
\noindent \textbf{More comparisons against \texttt{FedPAQ}-m:}
In \Cref{fig:1s}, similar to \Cref{fig:1}, we compare \texttt{FedPAQ}-m and \texttt{FedGLOMO} in the heterogeneous case with a smaller per-round communication budget. 

In all the experiments so far, we have chosen $E = 10$ and $K = 300$, and so $T = K E = 3000$. In \Cref{fig:4}, we compare 2 bit \texttt{FedPAQ}-m against 1 bit \texttt{FedGLOMO} in the heterogeneous case on FMNIST with two other combinations of $E$ and $K$ such that $T = 3000$ -- these are $E=5,K=600$ and $E=20,K=150$. Compared to $E=10,K=300$, we halve and double the initial learning rate $\eta_0$ for $E=20,K=150$ and $E=5,K=600$, respectively; we do this because from our theory, we are required to have $\eta_0 E < c$ for some constant $c$. 
Everything else is the same as described in \Cref{sec:exp}.

Now, we move onto changing the number of participating clients or global batch size, i.e., $r$. Recall that so far, we have set $r = 0.5n$. In \Cref{fig:3}, we compare 8-bit \texttt{FedPAQ}-m against 4-bit \texttt{FedGLOMO} with smaller values of $r$ in the heterogeneous case on CIFAR-10 with everything else remaining the same (including $K = 300$ and $E = 10$).

\begin{figure*}[!htb]
\centering 
\subfloat[FMNIST Train Loss]{
    \label{fig:1s_a}
	\includegraphics[width=0.45\textwidth]{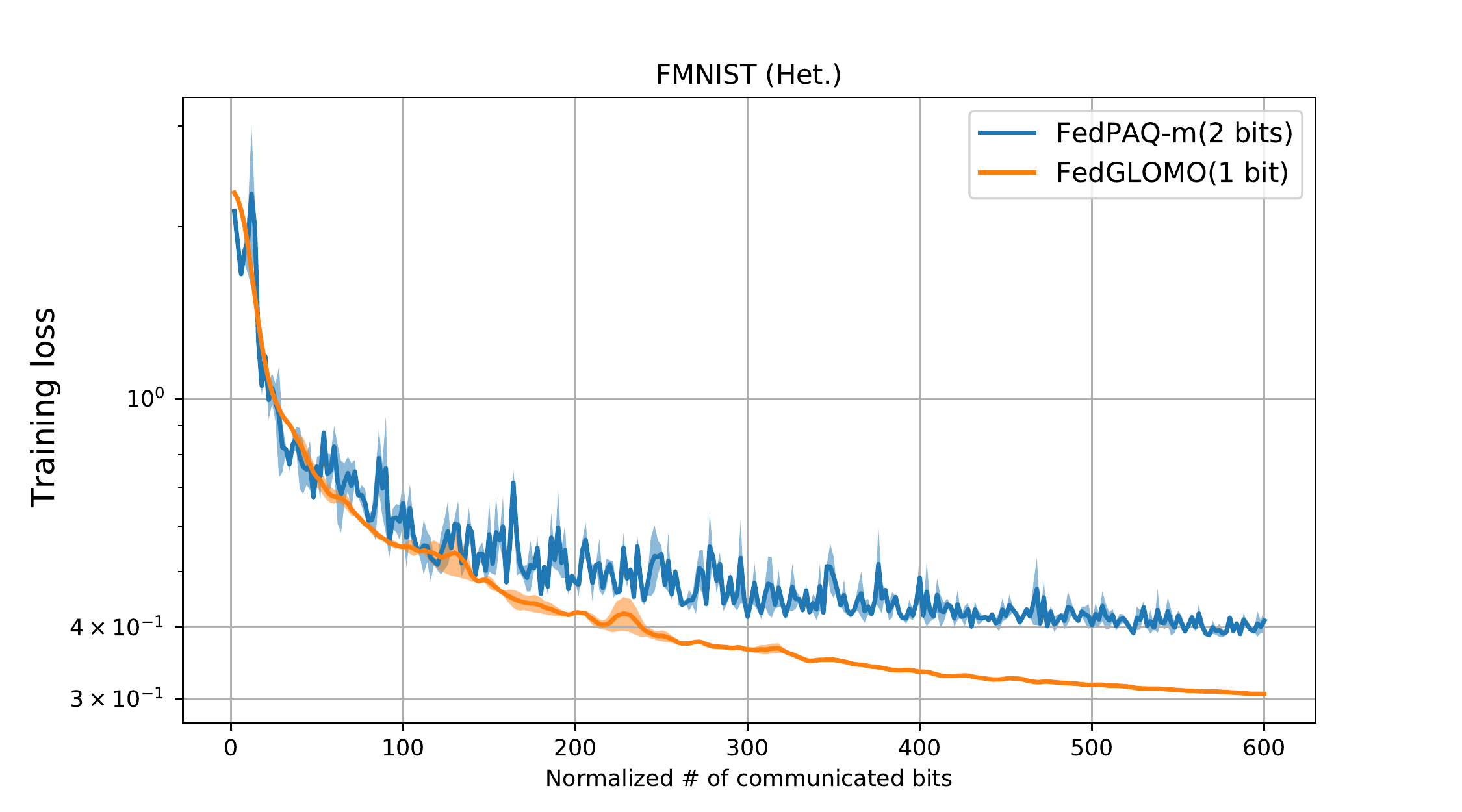}
	} 
\subfloat[FMNIST Test Error]{
    \label{fig:1s_b}
	\includegraphics[width=0.45\textwidth]{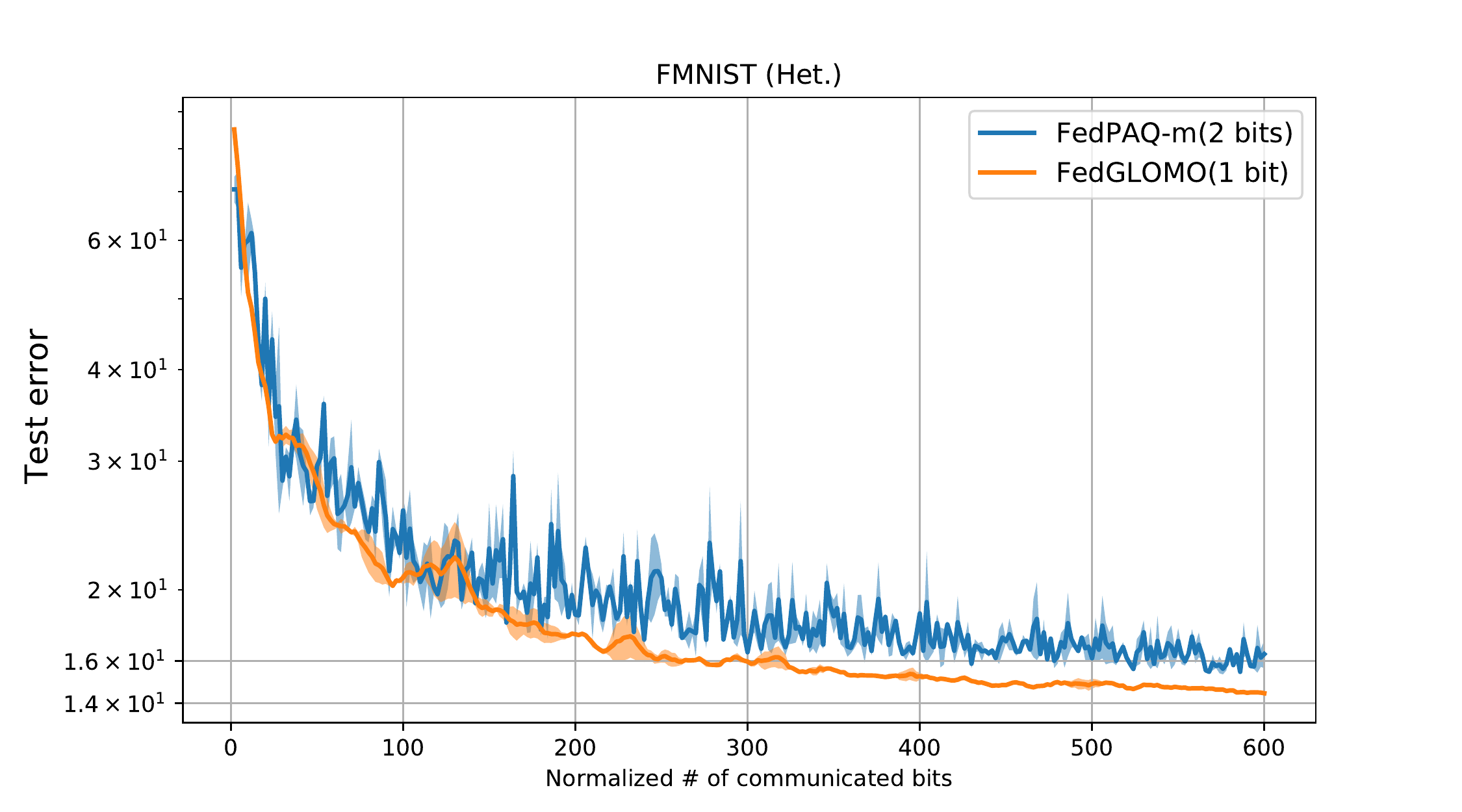}
	} 
\\
\subfloat[CIFAR-10 Train Loss]{
    \label{fig:1s_c}
	\includegraphics[width=0.45\textwidth]{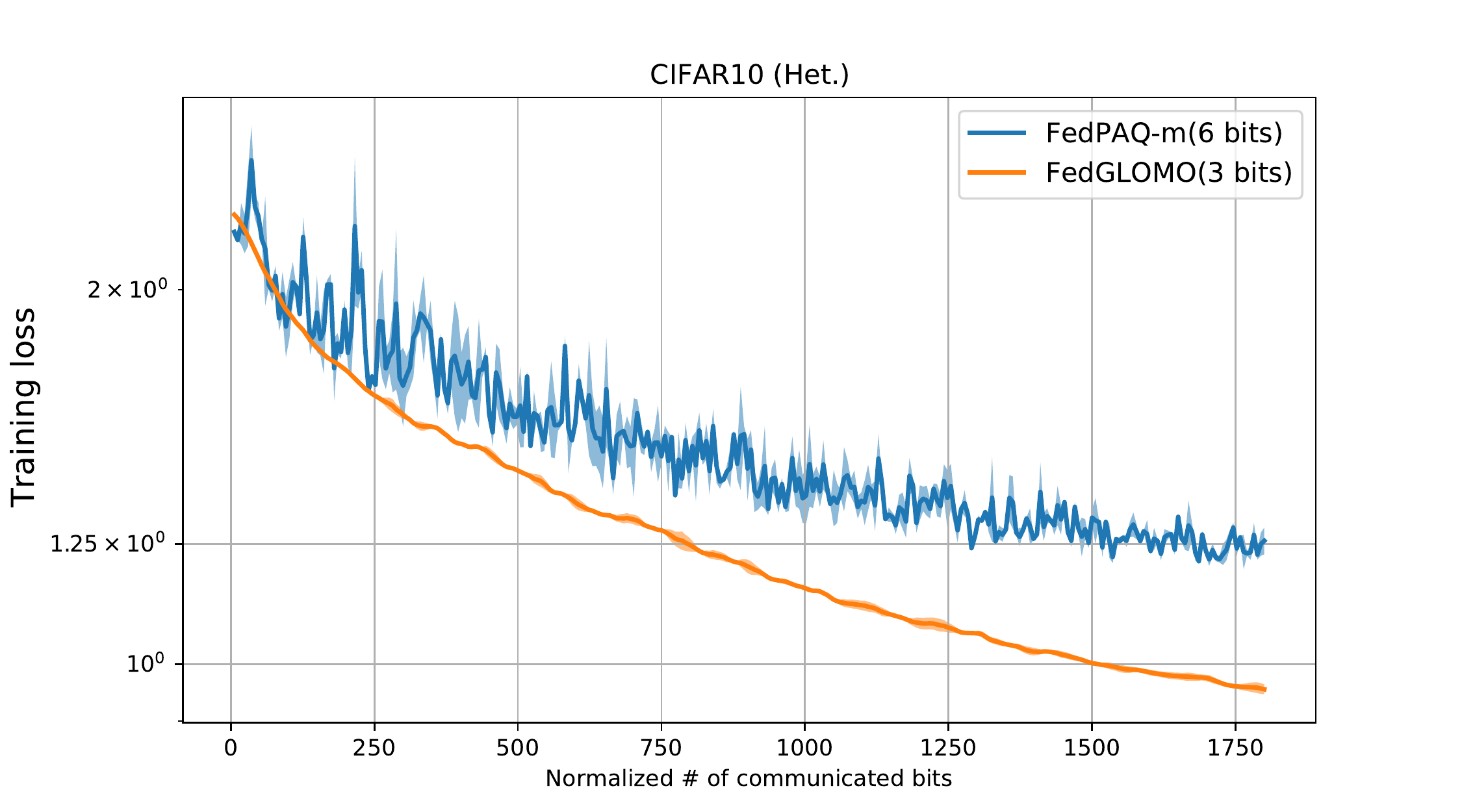}
	} 
\subfloat[CIFAR-10 Test Error]{
    \label{fig:1s_d}
	\includegraphics[width=0.45\textwidth]{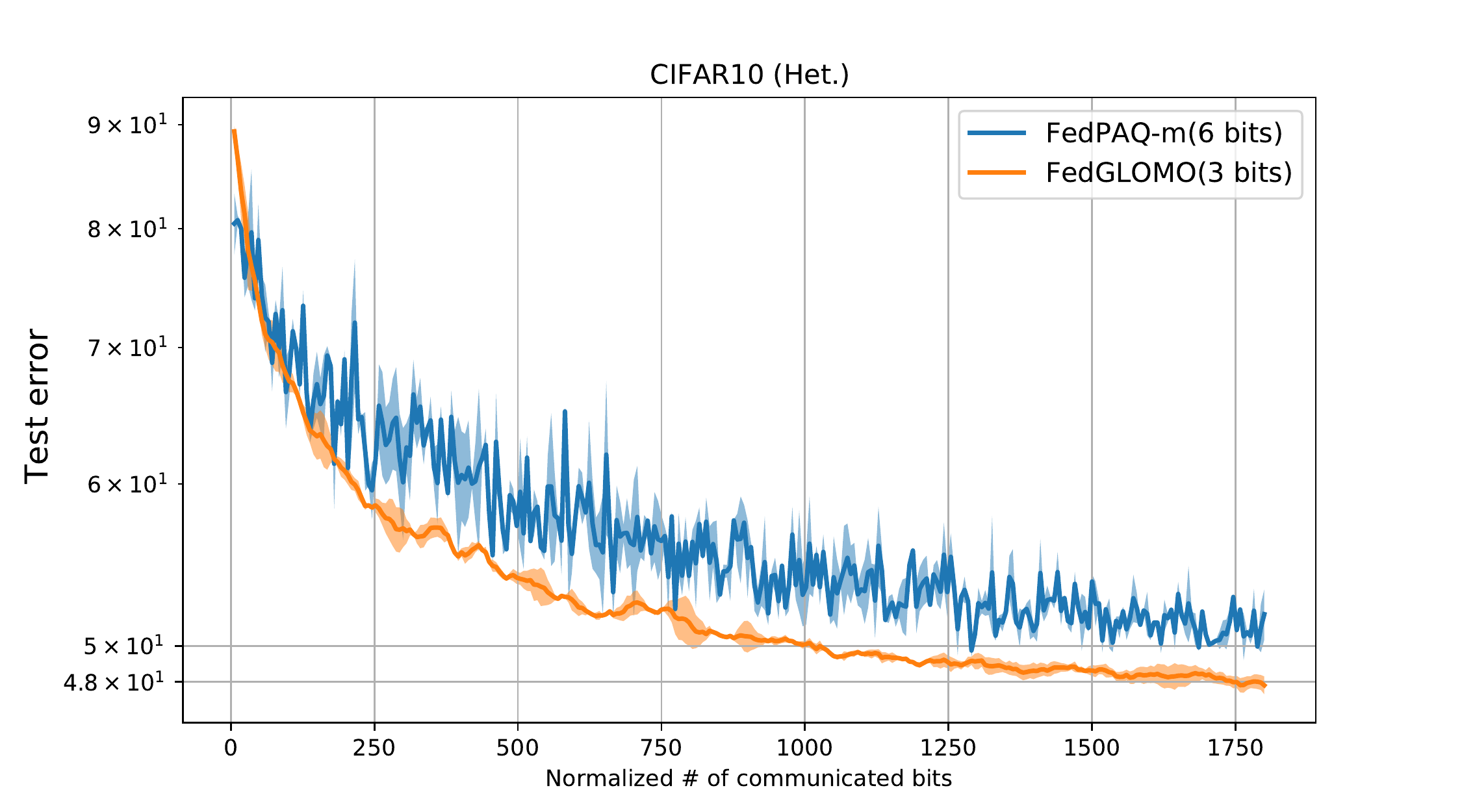}
	} 
\caption{\textbf{Heterogeneous case:} {
2 bit \texttt{FedPAQ}-m vs. 1 bit \texttt{FedGLOMO} on FMNIST at the top, and 6 bit \texttt{FedPAQ}-m vs. 3 bit \texttt{FedGLOMO} on CIFAR-10 at the bottom. The x-axis is the total number of communicated bits divided by the dimension $d$ and the global batch-size $r$. The trend is similar to \Cref{fig:1}.}}
\label{fig:1s}
\end{figure*}

\begin{figure*}[!htb]
\centering 
\subfloat[$E=5,K=600$]{
    \label{fig:4_a}
	\includegraphics[width=0.45\textwidth]{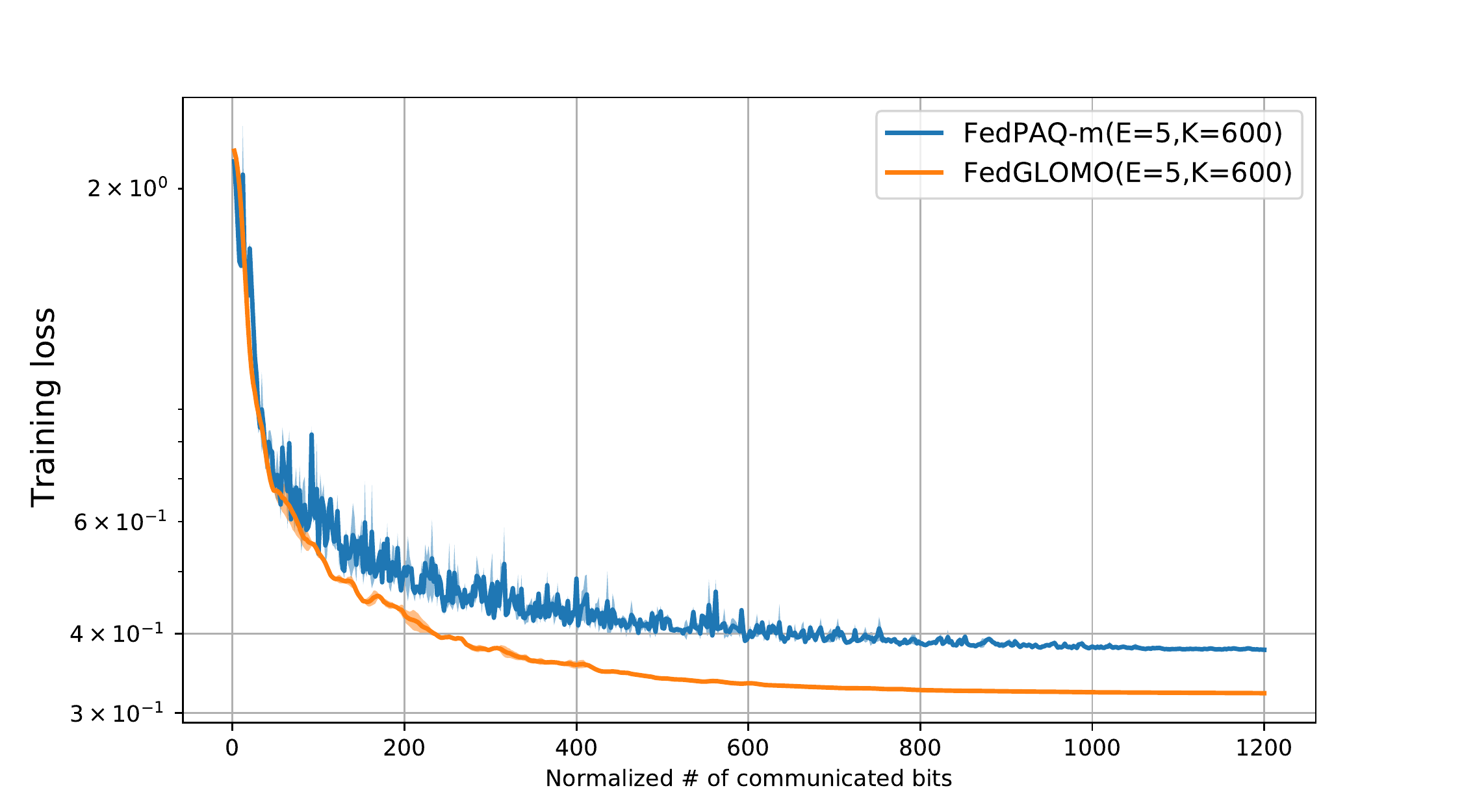}
	} 
\subfloat[$E=5,K=600$]{
    \label{fig:4_b}
	\includegraphics[width=0.45\textwidth]{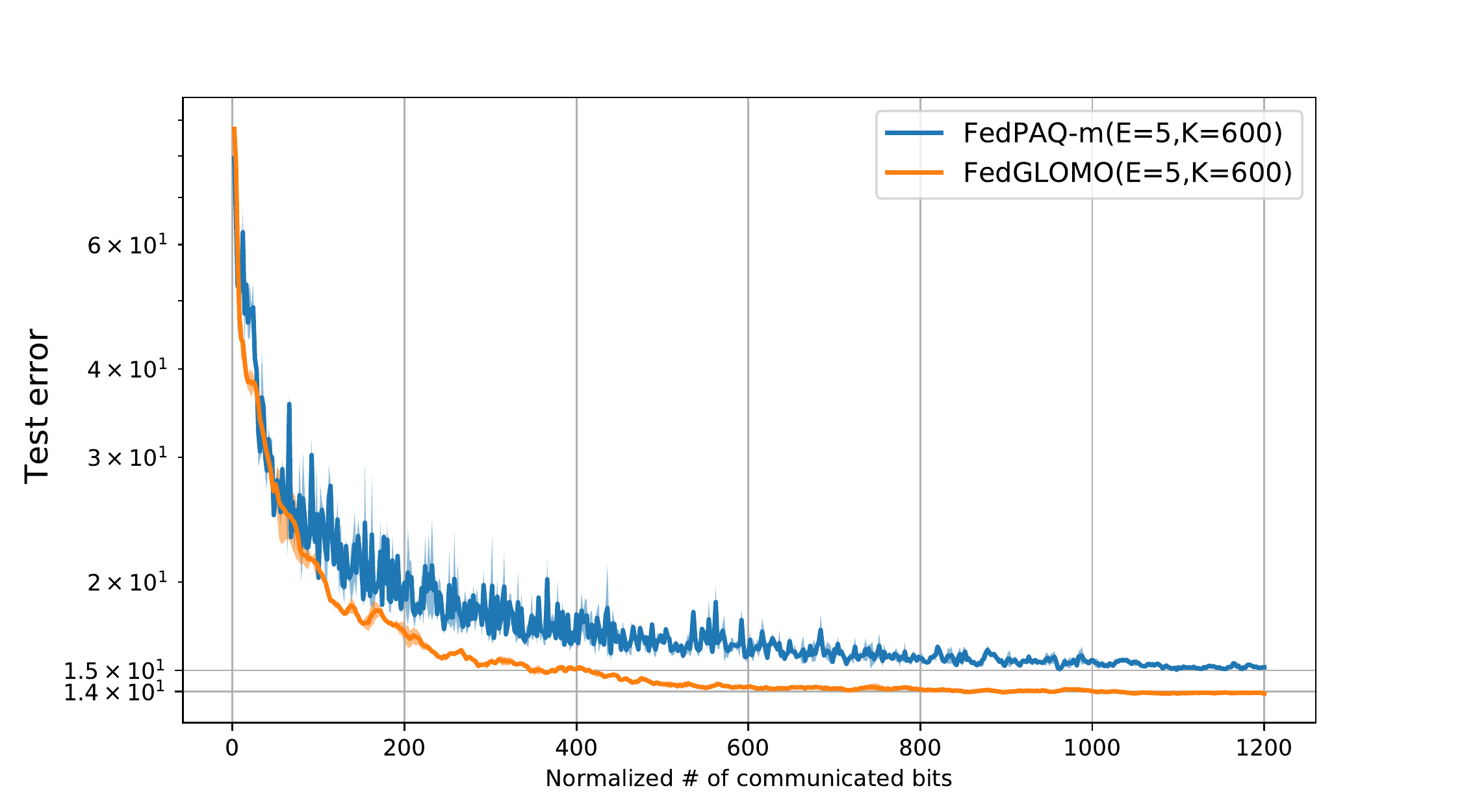}
	} 
\\
\subfloat[$E=20,K=150$]{
    \label{fig:4_c}
	\includegraphics[width=0.45\textwidth]{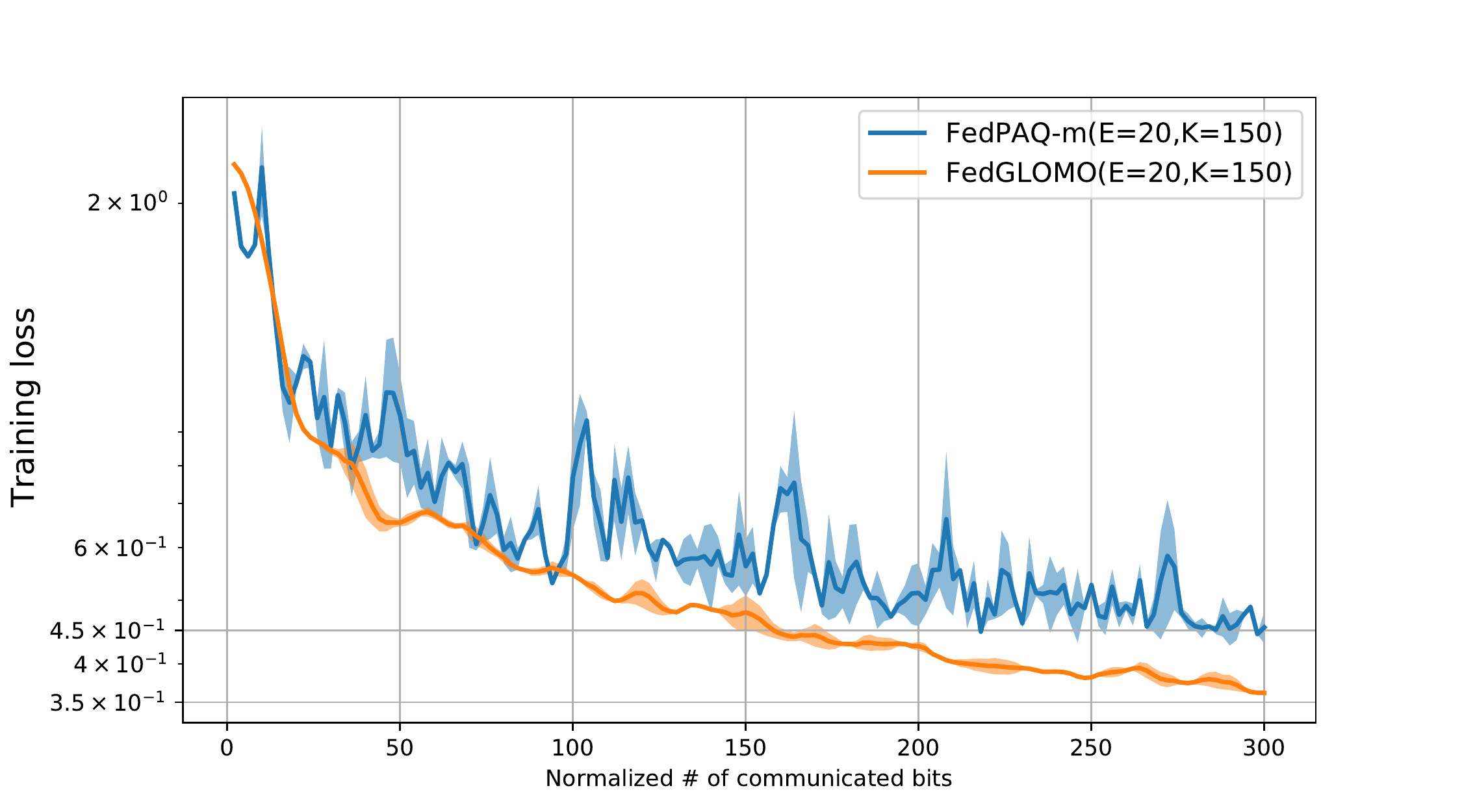}
	} 
\subfloat[$E=20,K=150$]{
    \label{fig:4_d}
	\includegraphics[width=0.45\textwidth]{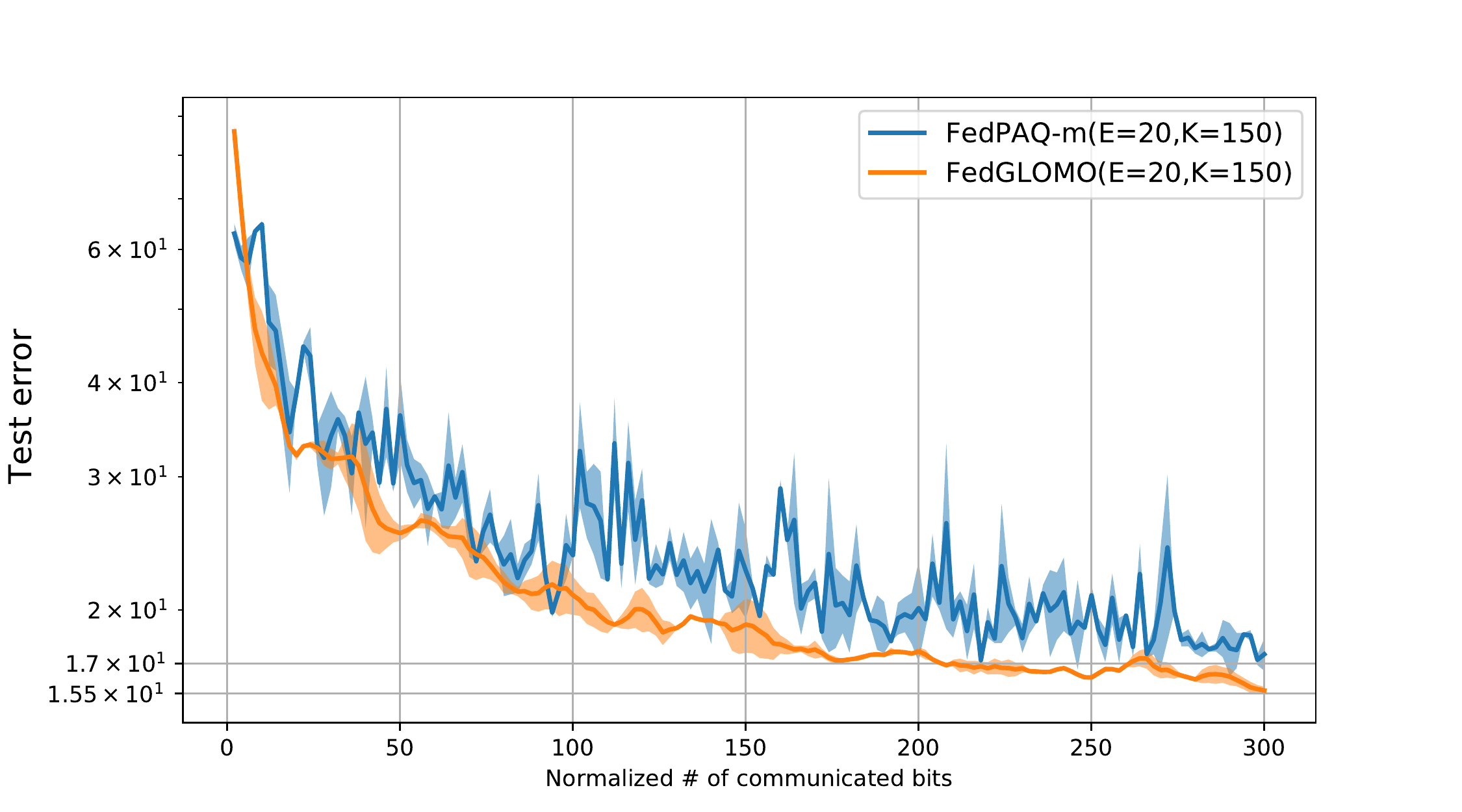}
	} 
\caption{\textbf{FMNIST Heterogeneous case:} 1-bit \texttt{FedGLOMO} vs. 2-bit \texttt{FedPAQ}-m for other values of $E$ and $K$, such that $T = K E = 3000$. Just as \Cref{fig:1s}, the x-axis is the total number of communicated bits divided by the dimension $d$ and the global batch-size $r$.
For $E = 5, K = 600$, \texttt{FedGLOMO} reaches the final test error of \texttt{FedPAQ}-m with about a \textbf{third} of the number of bits used by \texttt{FedPAQ}-m. For $E = 20, K = 150$, the corresponding number increases to just above two-thirds.}
\label{fig:4}
\end{figure*}

\begin{figure*}[!htb]
\centering 
\subfloat[$r=0.3n$]{
    \label{fig:3_a}
	\includegraphics[width=0.45\textwidth]{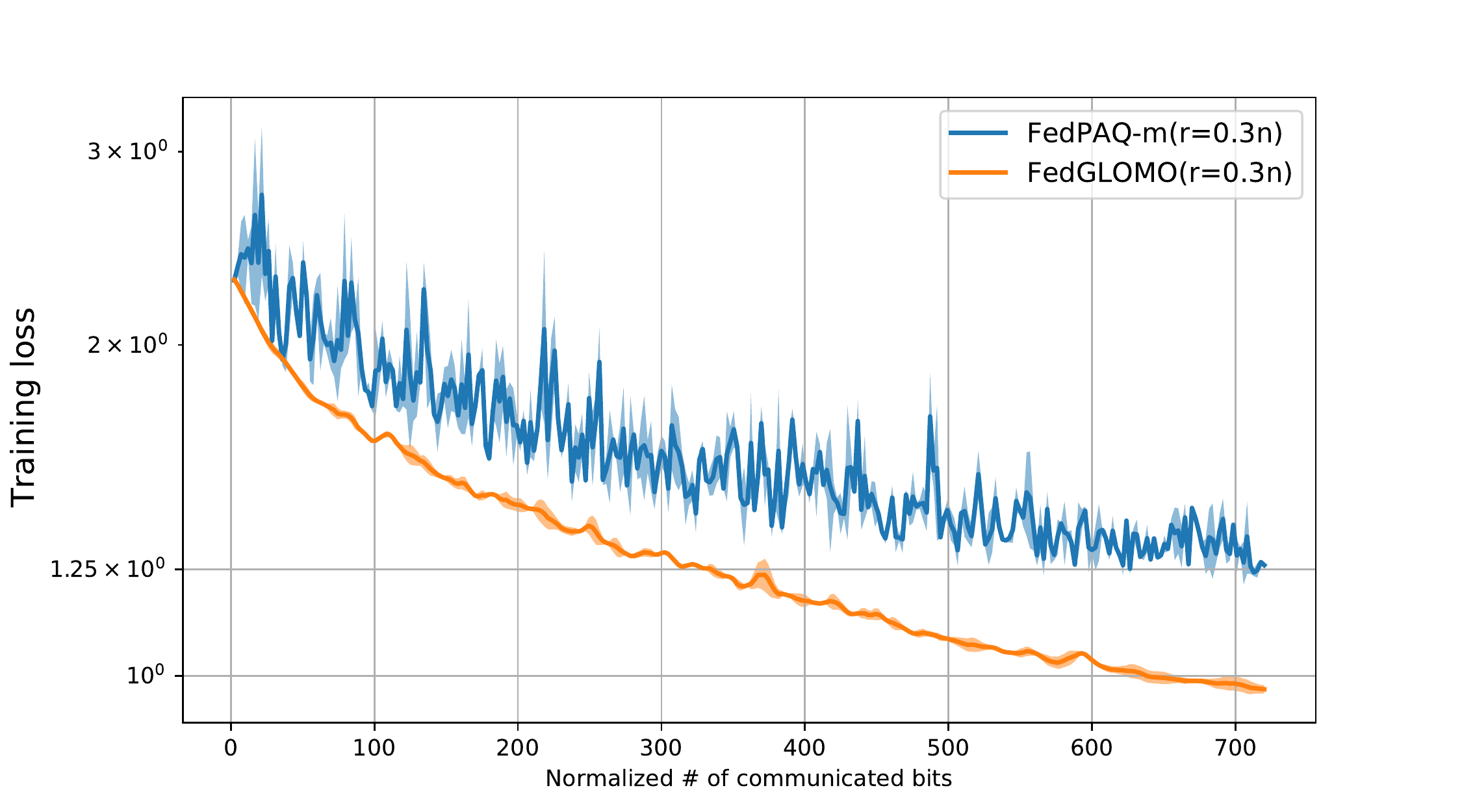}
	} 
\subfloat[$r=0.3n$]{
    \label{fig:3_b}
	\includegraphics[width=0.45\textwidth]{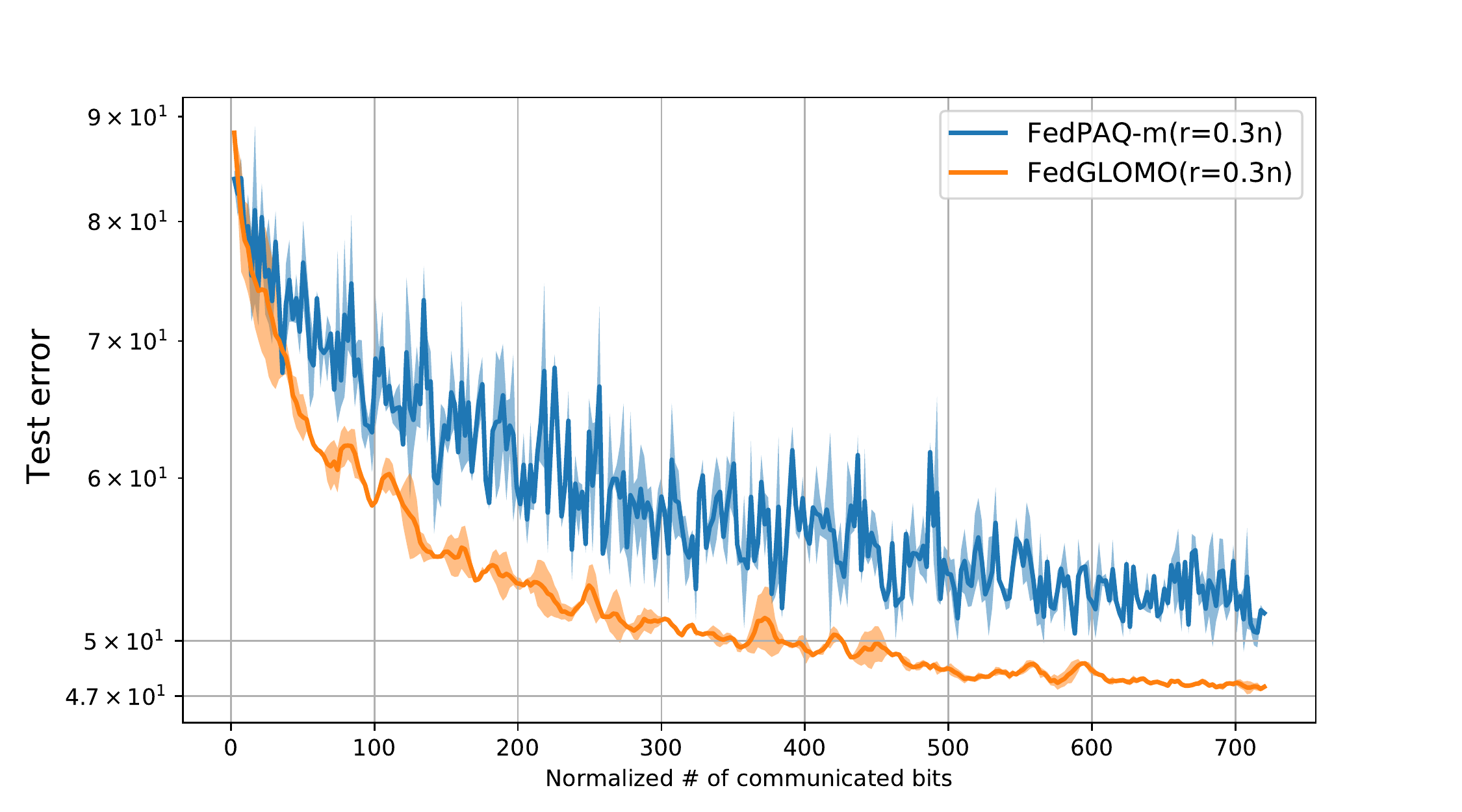}
	} 
\\
\subfloat[$r=0.1n$]{
    \label{fig:3_c}
	\includegraphics[width=0.45\textwidth]{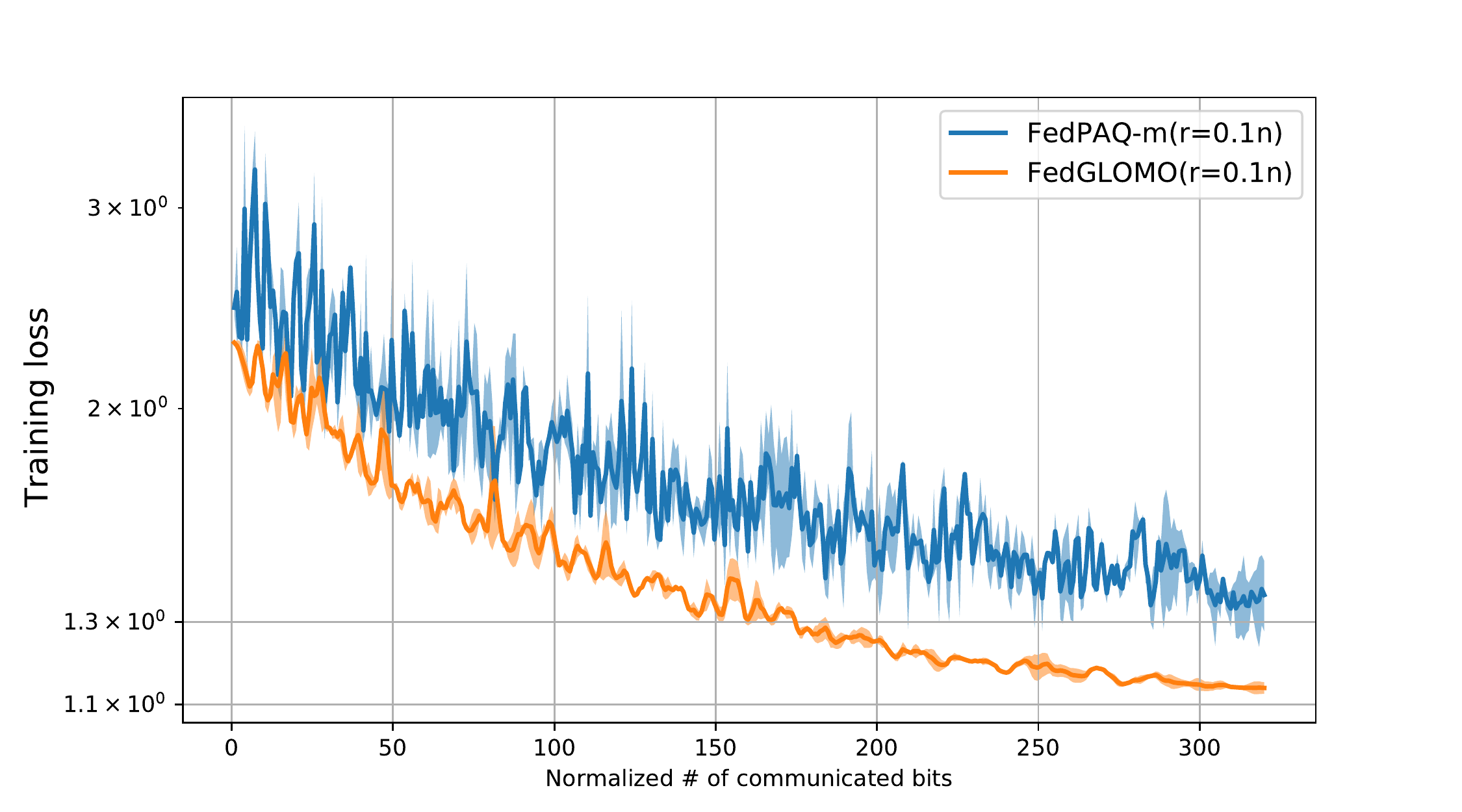}
	} 
\subfloat[$r=0.1n$]{
    \label{fig:3_d}
	\includegraphics[width=0.45\textwidth]{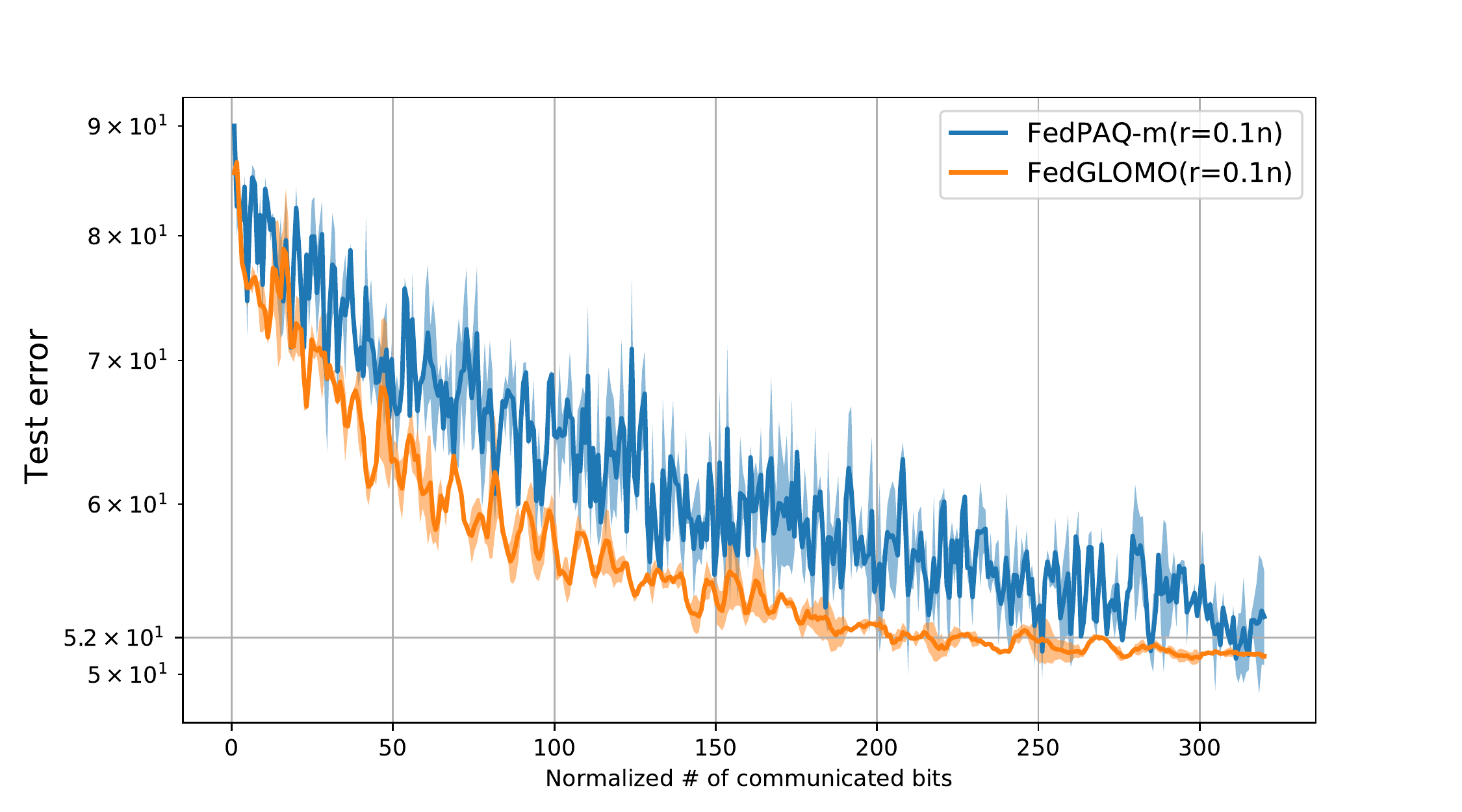}
	} 
\caption{\textbf{CIFAR-10 Heterogeneous case:} 4-bit \texttt{FedGLOMO} vs. 8-bit \texttt{FedPAQ}-m for smaller values of global batch size, i.e., $r$.
In this case, the x-axis is the total number of communicated bits divided by the dimension $d$ and the number of clients $n$.
For $r = 0.3n$, \texttt{FedGLOMO} attains the final test error of
\texttt{FedPAQ}-m with only about \textbf{half} the number of bits used by \texttt{FedPAQ}-m. For $r = 0.1n$, the corresponding ratio increases to roughly 80\% the number of bits used by \texttt{FedPAQ}-m. Overall, \texttt{FedGLOMO} consistently outperforms \texttt{FedPAQ}-m and has a smoother performance due to the application of variance-reducing momentum.}
\label{fig:3}
\end{figure*}

\noindent \textbf{Comparison against \texttt{FedCOMGATE} \cite{haddadpour2020federated}:}
As discussed in \Cref{rel-work}, \texttt{FedCOMGATE} \cite{haddadpour2020federated} is another communication-efficient FL algorithm incorporating gradient tracking to improve the convergence rate.
In \Cref{fig:0s}, we compare 4 bit \texttt{FedGLOMO} against 8 bit \texttt{FedCOMGATE} and 8 bit \texttt{FedPAQ}-m; again, the per-round communication budget of all algorithms is the same.
\\
\begin{figure*}[!htb]
\centering 
\subfloat[FMNIST Train Loss]{
    \label{fig:0s_a}
	\includegraphics[width=0.45\textwidth]{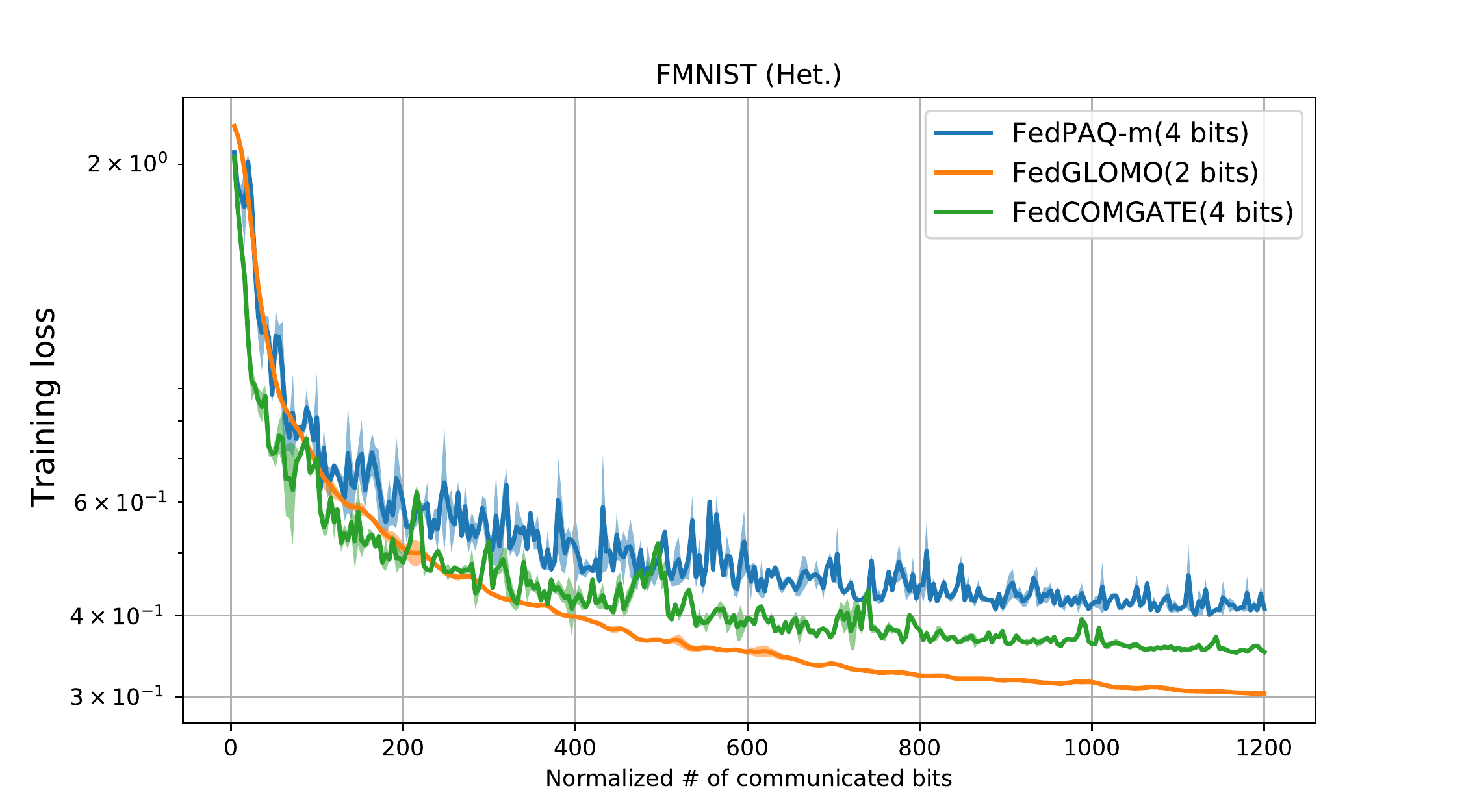}
	} 
\subfloat[FMNIST Test Error]{
    \label{fig:0s_b}
	\includegraphics[width=0.45\textwidth]{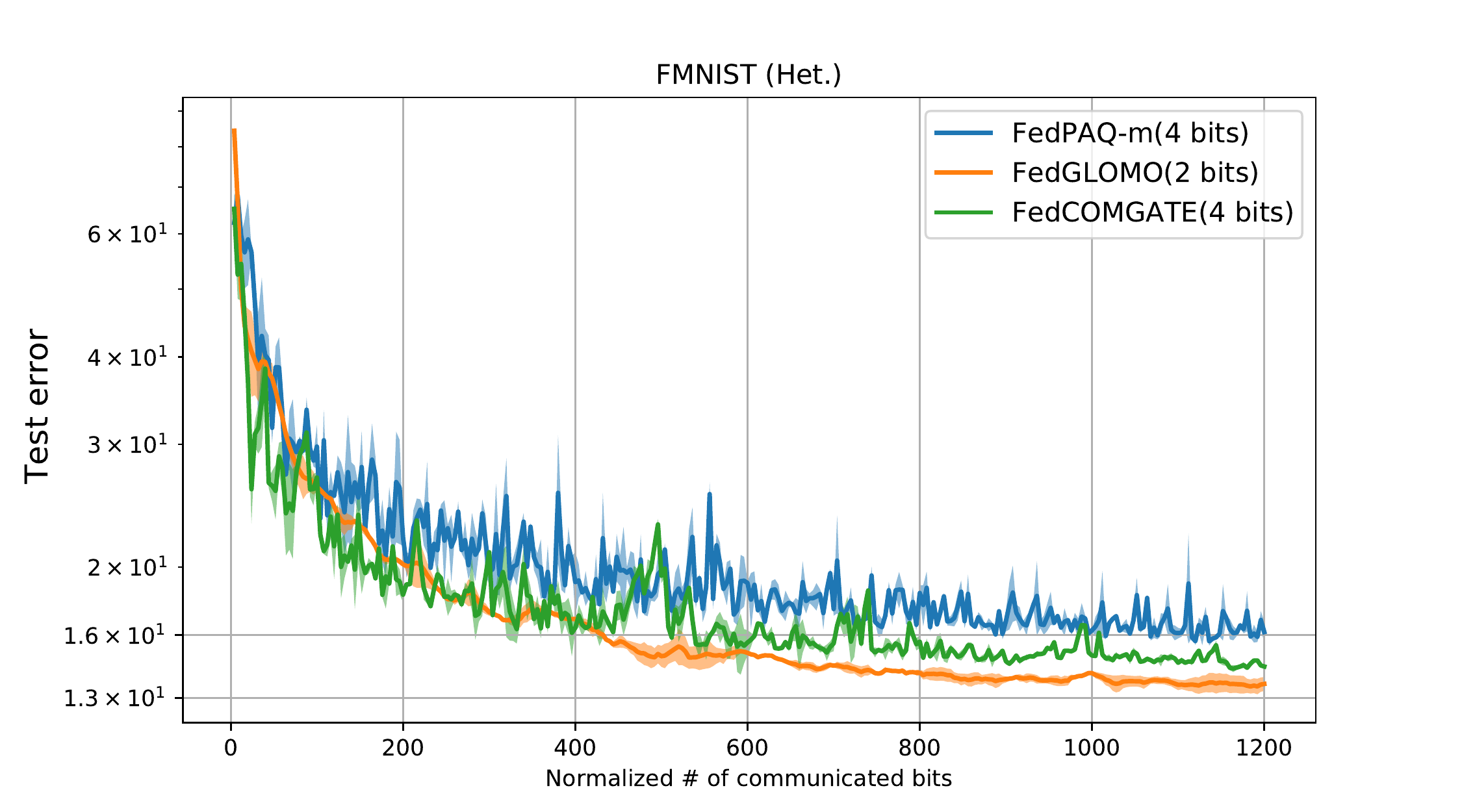}
	} 
\\
\subfloat[CIFAR-10 Train Loss]{
    \label{fig:0s_c}
	\includegraphics[width=0.45\textwidth]{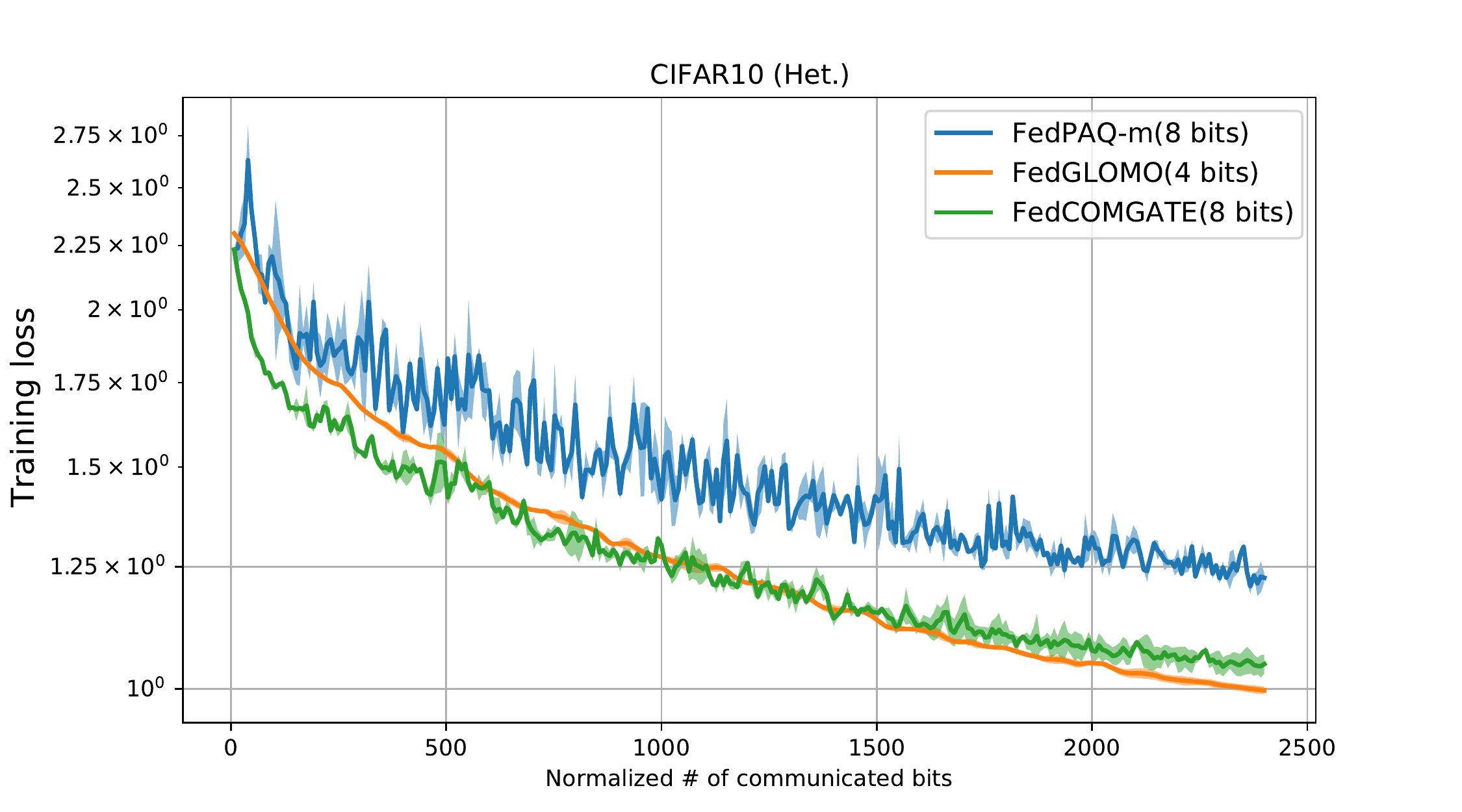}
	} 
\subfloat[CIFAR-10 Test Error]{
    \label{fig:0s_d}
	\includegraphics[width=0.45\textwidth]{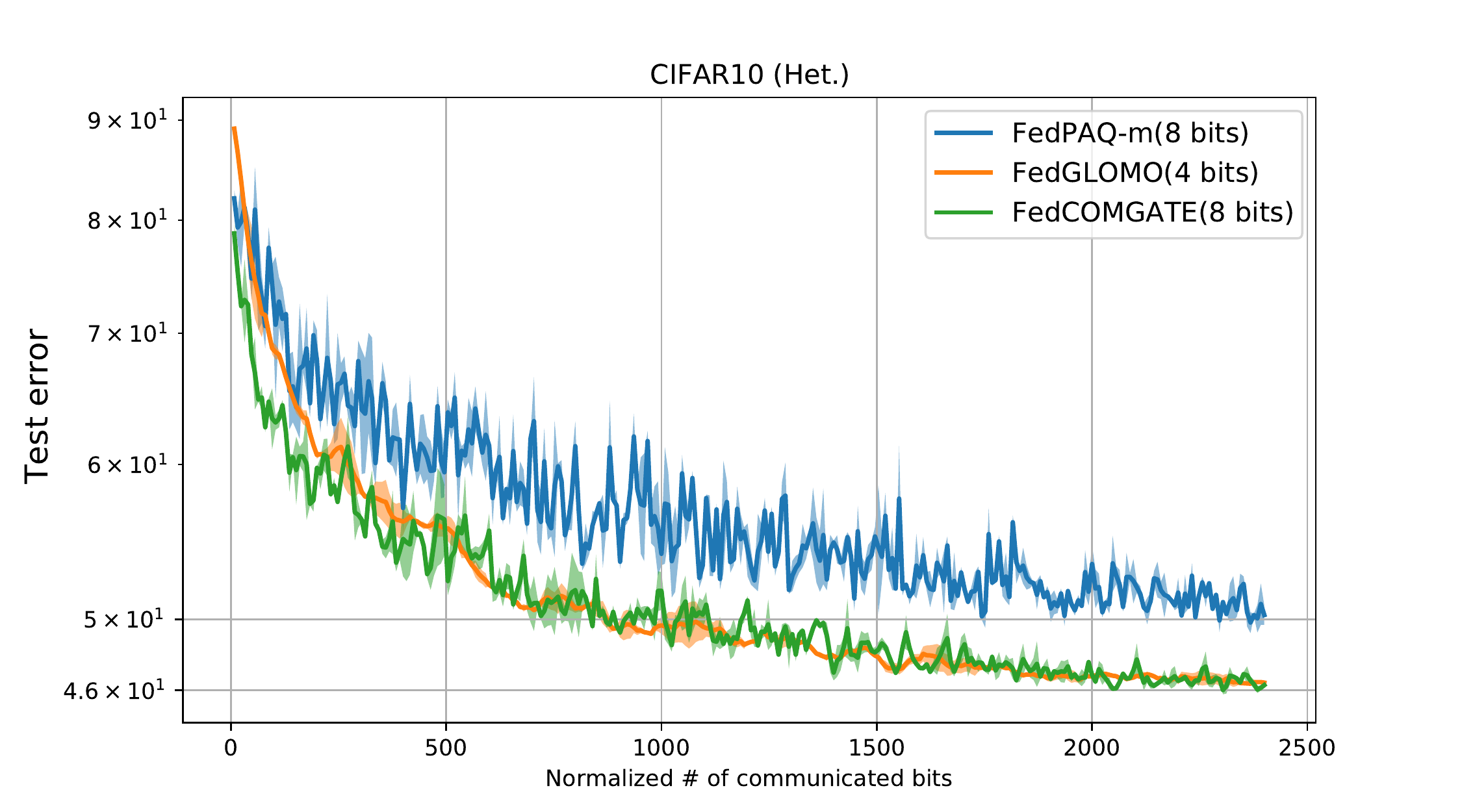}
	} 
\caption{\textbf{Heterogeneous case:} {Comparison of \texttt{FedPAQ}-m, \texttt{FedGLOMO} and \texttt{FedCOMGATE} \cite{haddadpour2020federated} with the same per-round communication budget
on FMNIST (top) and CIFAR-10 (bottom). The x-axis is the total number of communicated bits divided by the dimension $d$ and the global batch-size $r$. {On FMNIST, \texttt{FedGLOMO} significantly outperforms \texttt{FedCOMGATE} both in terms of training loss as well as test error}.
However, on CIFAR-10, both \texttt{FedGLOMO} and \texttt{FedCOMGATE} have similar test set performance; with respect to the training loss, \texttt{FedCOMGATE} is initially faster but after about 200 rounds, \texttt{FedGLOMO} takes over. From these experiments, we see that our proposed idea of variance-reducing global and local momentum does have some advantage over gradient tracking (which is the main ingredient of \texttt{FedCOMGATE}) 
when applied for a sufficiently large number of rounds.}}
\label{fig:0s}
\end{figure*}

\noindent\textbf{Comparison against server-level PyTorch-like momentum:} One can even apply server-level momentum by using the kind of momentum provided by PyTorch. This can be done by implementing the server update as a PyTorch optimizer update with momentum.
Here, we show the superiority of our proposed scheme as compared to this kind of server-level momentum. Specifically, we compare \texttt{FedGLOMO} against \texttt{FedPAQ}-m augmented with server-level PyTorch momentum; we call this \texttt{FedPAQ}-scm (\enquote{scm} stands for server and client momentum). We tried three different values of server-level momentum which are $\{0.9,0.7,0.5\}$ and show the results with the best value in \Cref{fig:sm-1} and \Cref{fig:sm-2} for the homogeneous and heterogeneous cases, respectively. 

\begin{figure*}[!htb]
\centering 
\subfloat[FMNIST Train Loss]{
    \label{fig:sm1_a}
	\includegraphics[width=0.45\textwidth]{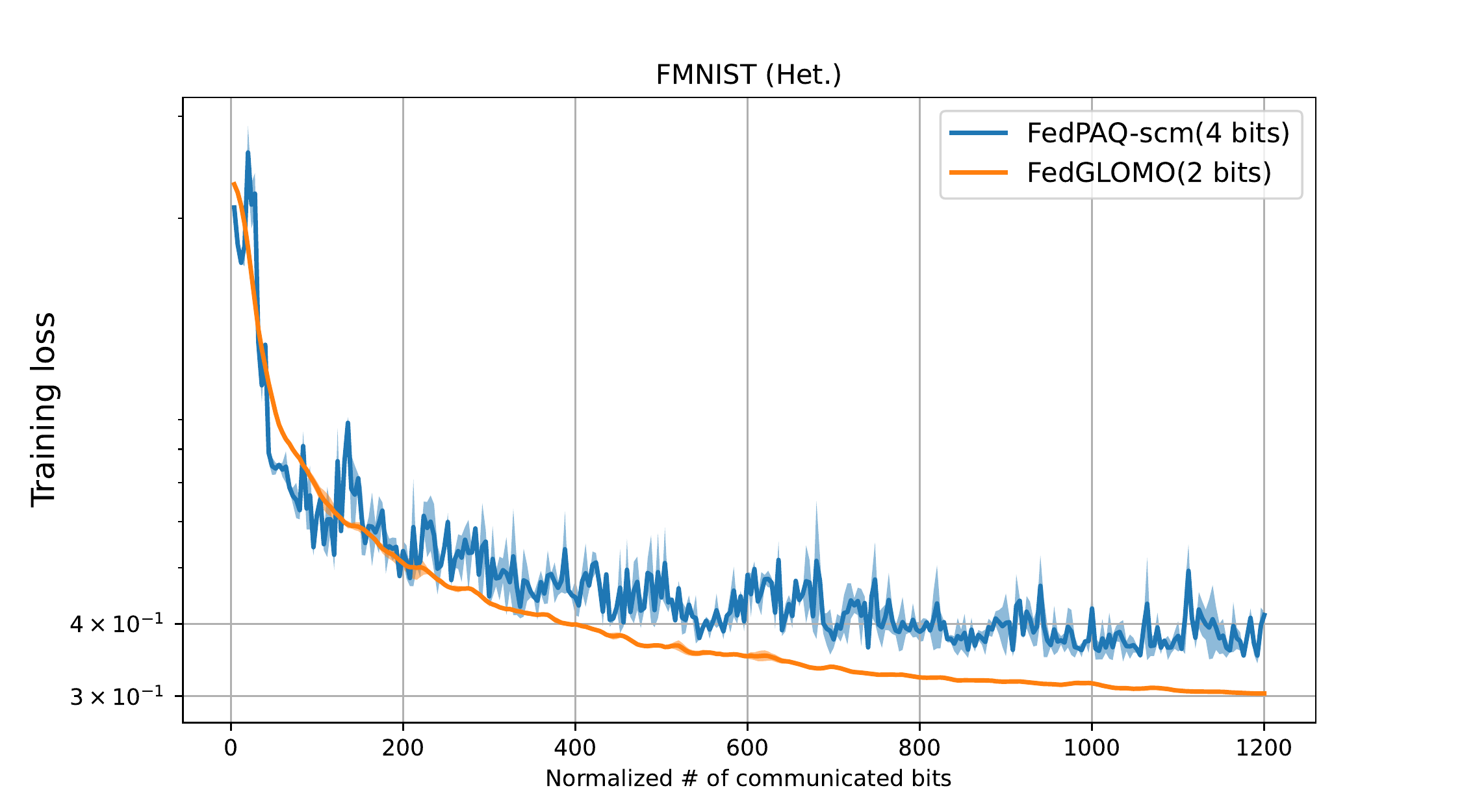}
	} 
\subfloat[FMNIST Test Error]{
    \label{fig:sm1_b}
	\includegraphics[width=0.45\textwidth]{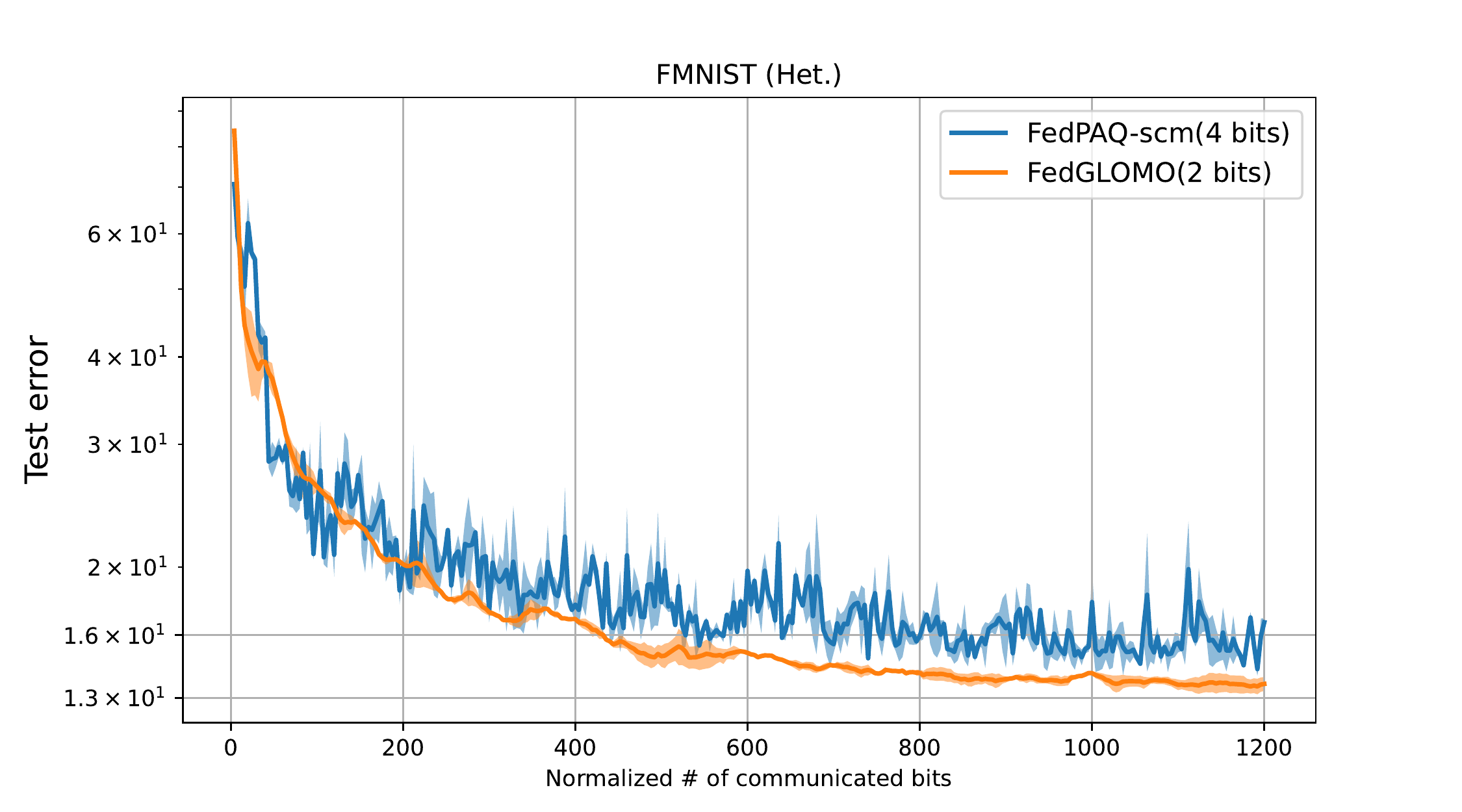}
	} 
\\
\subfloat[CIFAR-10 Train Loss]{
    \label{fig:sm1_c}
	\includegraphics[width=0.45\textwidth]{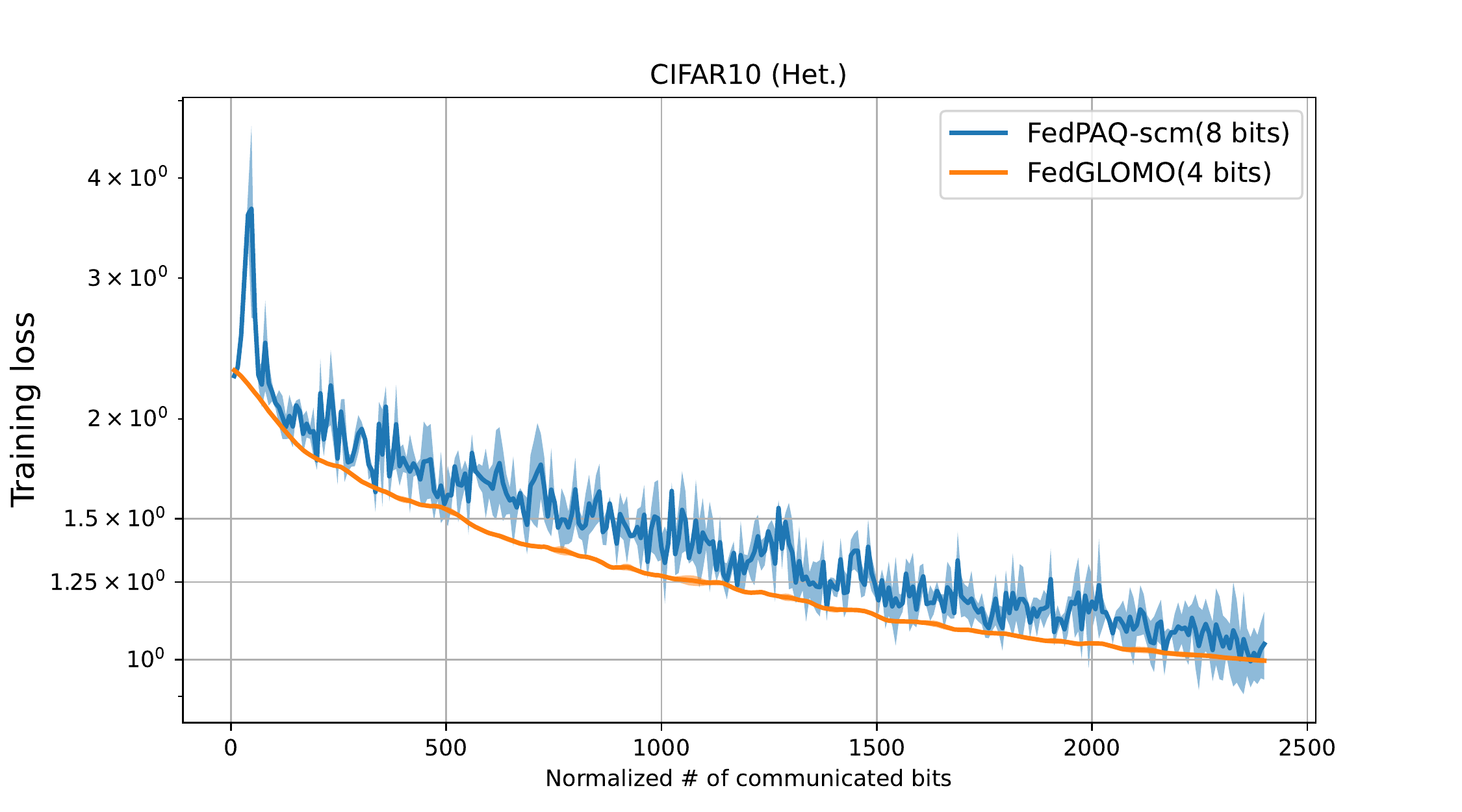}
	} 
\subfloat[CIFAR-10 Test Error]{
    \label{fig:sm1_d}
	\includegraphics[width=0.45\textwidth]{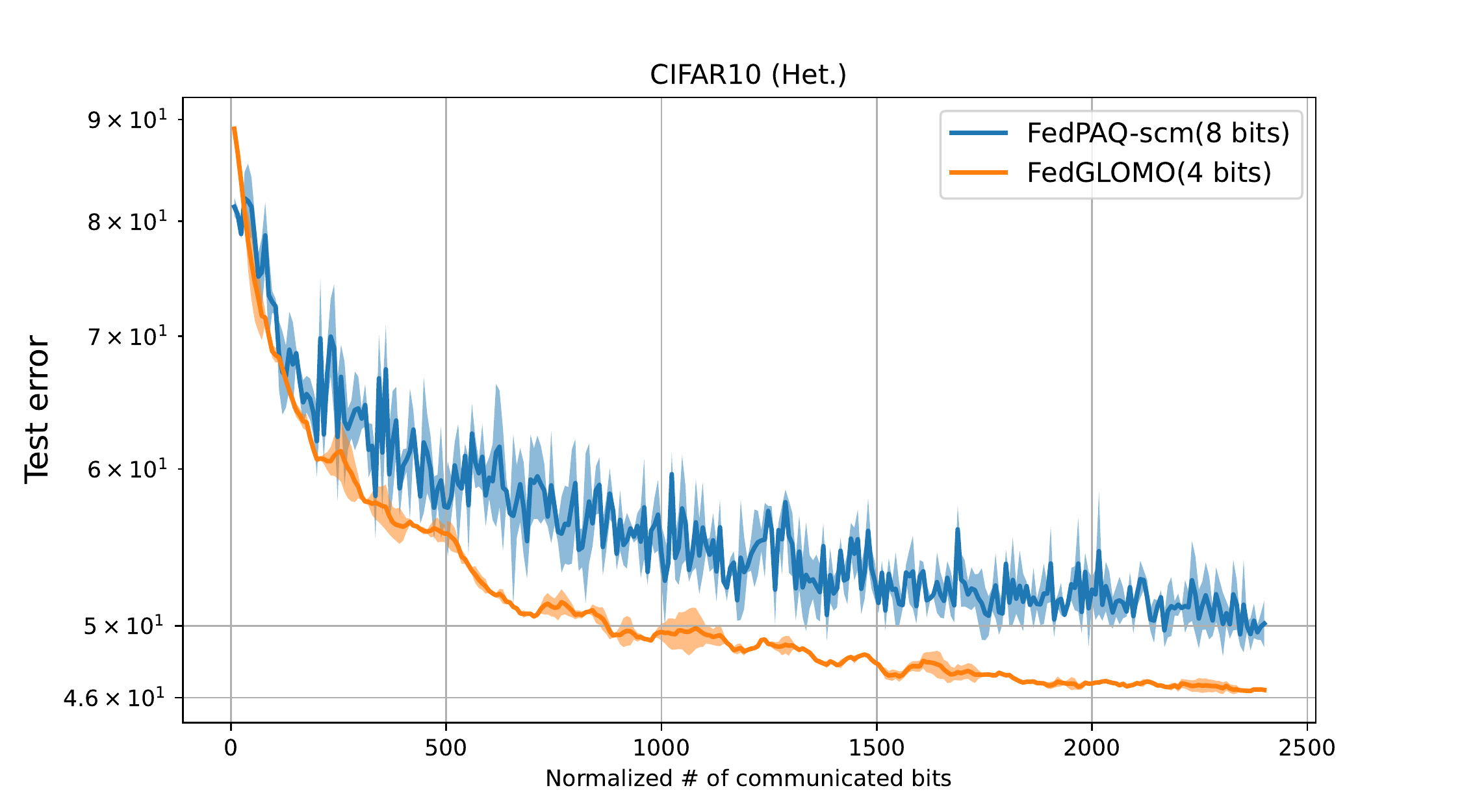}
	} 
\caption{{\textbf{Heterogeneous case:} 4 (resp., 8) bit \texttt{FedPAQ}-scm vs. 2 (resp., 4) bit \texttt{FedGLOMO} on FMNIST (resp., CIFAR-10) at the top (resp., bottom). The x-axis is the total number of communicated bits divided by the dimension $d$ and the global batch-size $r$. \texttt{FedGLOMO} outperforms \texttt{FedPAQ}-scm and has a smoother performance than it due to the application of variance-reducing momentum.}}
\label{fig:sm-1}
\end{figure*}

\begin{figure*}[!htb]
\centering 
\subfloat[FMNIST Train Loss]{
    \label{fig:sm2_a}
	\includegraphics[width=0.45\textwidth]{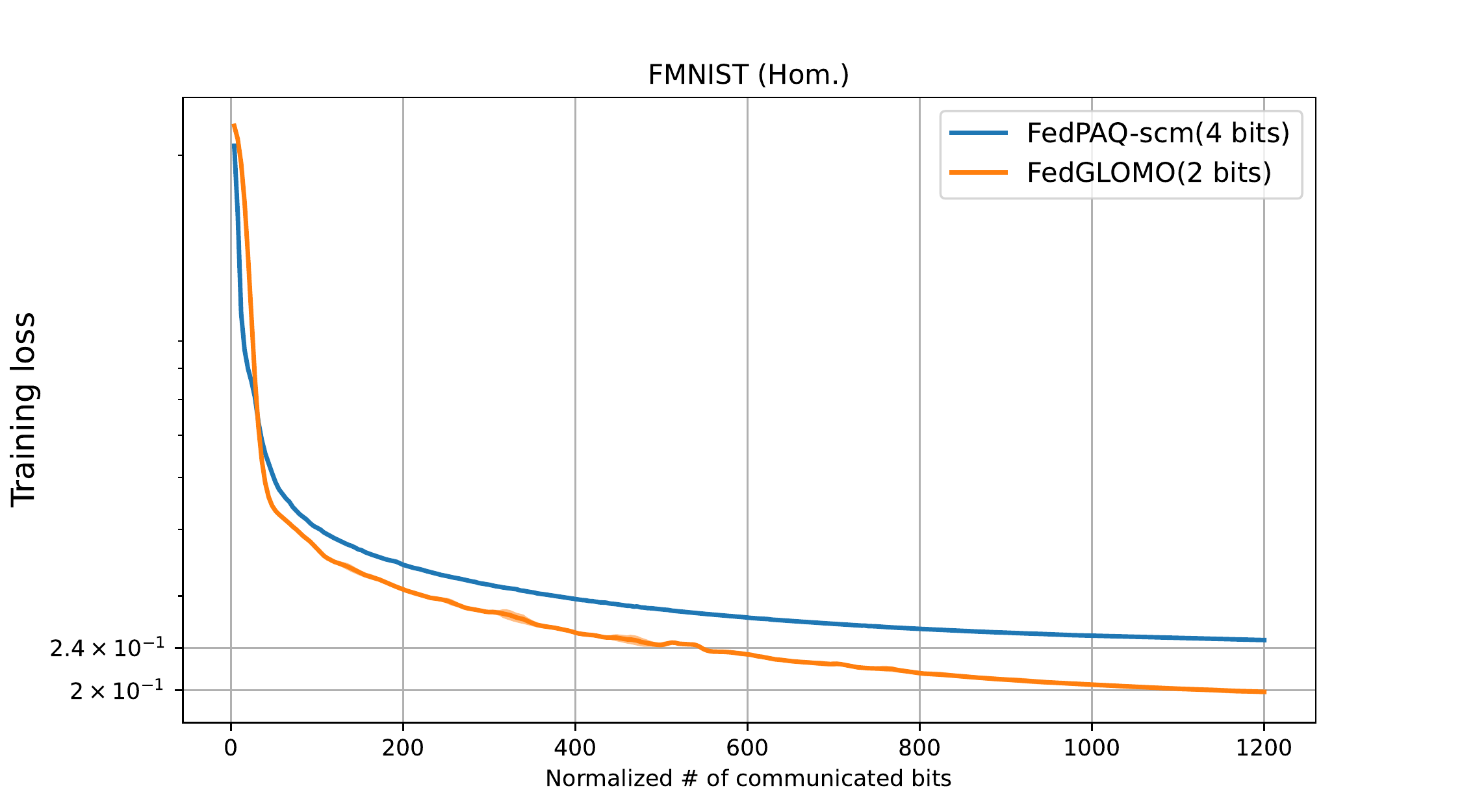}
	} 
\subfloat[FMNIST Test Error]{
    \label{fig:sm2_b}
	\includegraphics[width=0.45\textwidth]{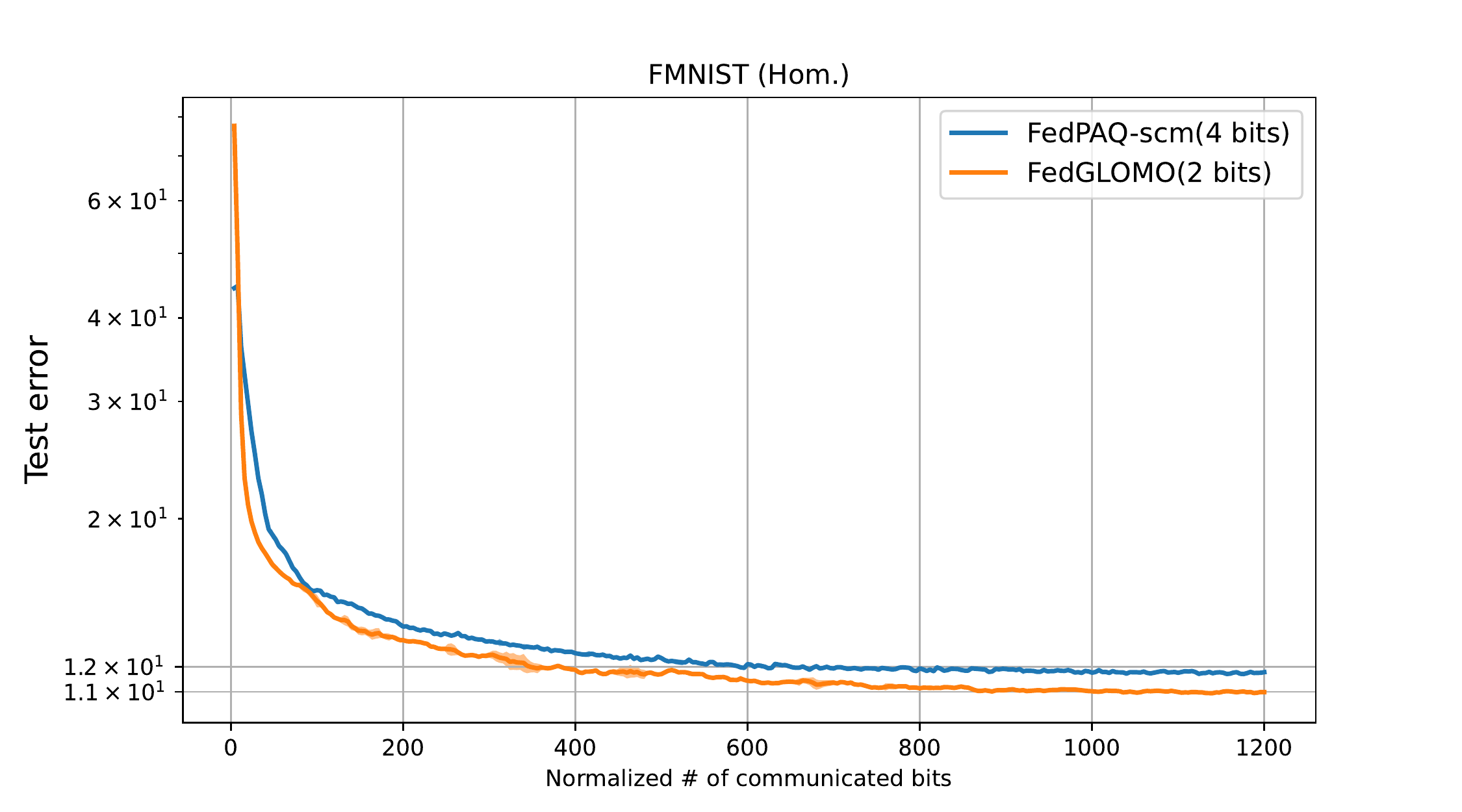}
	} 
\\
\subfloat[CIFAR-10 Train Loss]{
    \label{fig:sm2_c}
	\includegraphics[width=0.45\textwidth]{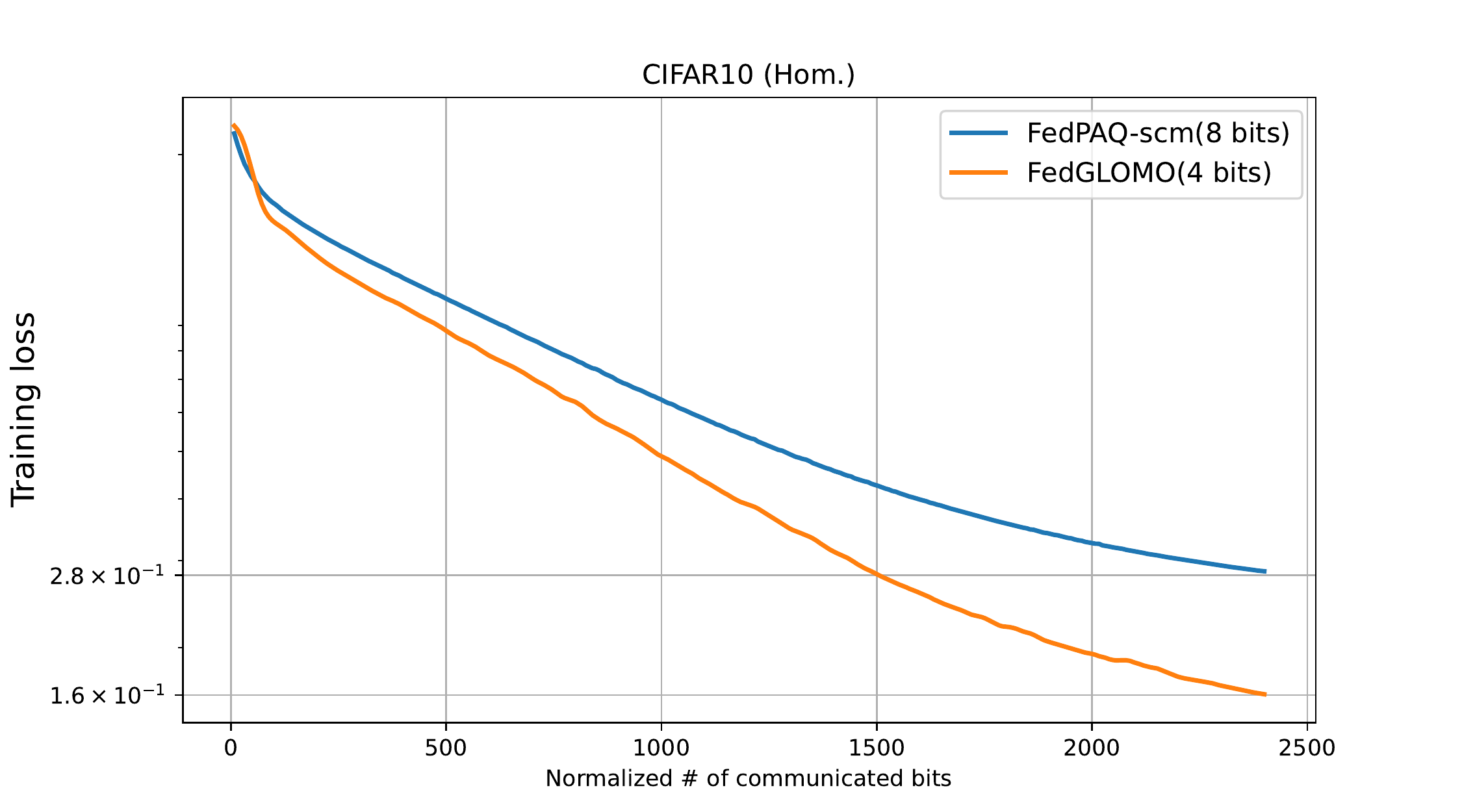}
	} 
\subfloat[CIFAR-10 Test Error]{
    \label{fig:sm2_d}
	\includegraphics[width=0.45\textwidth]{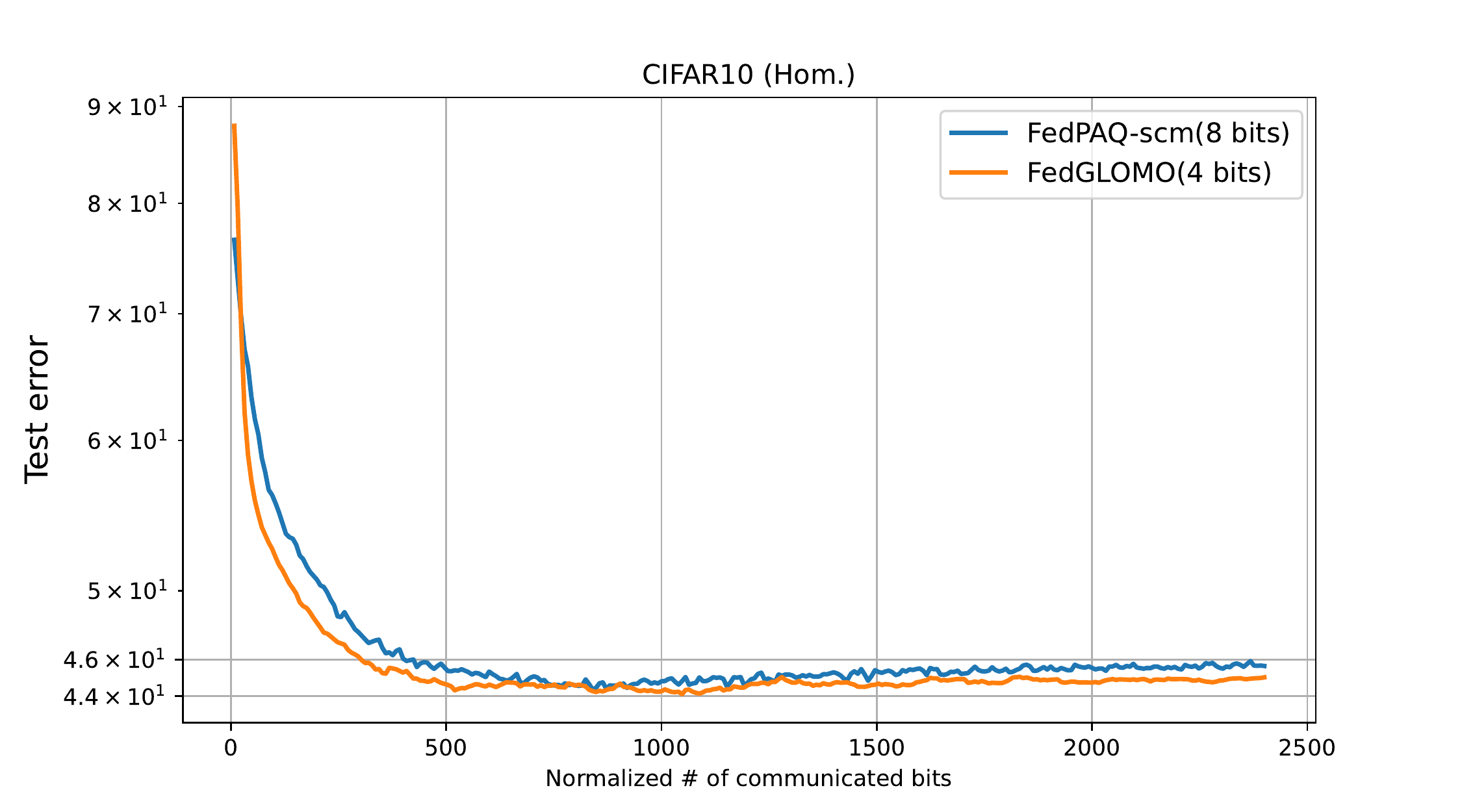}
	} 
\caption{ \textbf{Homogeneous case:} Same setting as \Cref{fig:sm-1} but in the homogeneous case. Once again, \texttt{FedGLOMO} outperforms \texttt{FedPAQ}-scm. Here, \texttt{FedPAQ}-scm has a smoother performance compared to \Cref{fig:sm-1} due to the homogeneous data distribution.}
\label{fig:sm-2}
\end{figure*}

\section{\texttt{FedLOMO}: A Simpler Version of \texttt{FedGLOMO}}
\label{sec:lomo}
Now we consider a simpler version of \texttt{FedGLOMO}, which we call \texttt{FedLOMO}, that applies only local momentum in the client updates and does simple averaging at the server (like \texttt{FedAvg}), i.e., there is no global momentum (and hence the name of this variant does not have a \enquote{\texttt{G}}). \texttt{FedLOMO} is summarized in \Cref{alg:1} and \ref{alg:1-local}. Notice that the momentum application occurs in {line \ref{line:mom}} of \Cref{alg:1-local}. 

As mentioned in the main paper, \texttt{FedLOMO} does not achieve the optimal convergence rate for smooth non-convex functions due to the absence of global momentum; see \Cref{fl-thm3} and the subsequent remarks. 

Just like the results of \texttt{FedGLOMO}, we do not use the BCD assumption (i.e., \cref{eq:bcd}) to derive the results of \texttt{FedLOMO}.

\begin{algorithm}[!htb]
	\caption{\texttt{FedLOMO} - Server Update}
	\label{alg:1}
	\begin{algorithmic}[1]
		\STATE {\bfseries Input:} 
		Initial point $\bm{w}_0$, \# of rounds of communication $K$, period $E$, learning rates  $\{\eta_{k}\}_{k=0}^{K-1}$, per-client batch size $b$, and  global batch size $r$. $Q_D$ is the quantization operator.
		\vspace{0.1 cm}
		\FOR{$k =0,\dots, K-1$}
		\vspace{0.1 cm}
		\STATE Server chooses a set $\mathcal{S}_k$ of $r$ clients uniformly at random without replacement 
		and sends $\bm{w}_k$ to them.
		\vspace{0.1 cm}
		\FOR{client $i \in \mathcal{S}_k$}
		\vspace{0.1 cm}
		\STATE Set $\bm{w}_{k,0}^{(i)} = \bm{w}_k$ and run \Cref{alg:1-local} for client $i$.
		\vspace{0.1 cm}
		\ENDFOR
		\vspace{0.1 cm}
		\STATE Update $\bm{w}_{k+1} = \bm{w}_{k} + \frac{1}{r}\sum_{i \in \mathcal{S}_k}Q_{D}({\bm{w}_{k,E}^{(i)} - \bm{w}_{k}})$.
		\label{line:fedavg}
		\vspace{0.1 cm}
		\ENDFOR
	\end{algorithmic}
\end{algorithm}

\begin{algorithm}[!htb]
	\caption{\texttt{FedLOMO} - {Client Update}}
	\label{alg:1-local}
	\begin{algorithmic}[1]
		\FOR{$\tau = 0,\ldots,E-1$}
		\vspace{0.1 cm}
		\IF{$\tau = 0$}
		\vspace{0.1 cm}
		\STATE $\bm{v}_{k,\tau}^{(i)} = \nabla f_i(\bm{w}_{k,\tau}^{(i)})$. 
		\label{line:full-grad}
		\vspace{0.1 cm}
		\ELSE
		\vspace{0.1 cm}
		\STATE Pick a random batch of $b$ samples in client $i$, say $\mathcal{B}_{k,\tau}^{(i)}$. 
		Compute the stochastic gradients of $f_i$ at $\bm{w}_{k,\tau}^{(i)}$ and $\bm{w}_{k,\tau-1}^{(i)}$ over $\mathcal{B}_{k,\tau}^{(i)}$ viz. $\widetilde{\nabla} f_i(\bm{w}_{k,\tau}^{(i)};\mathcal{B}_{k,\tau}^{(i)})$ and $\widetilde{\nabla} f_i(\bm{w}_{k,\tau-1}^{(i)};\mathcal{B}_{k,\tau}^{(i)})$, respectively.
		\vspace{0.2 cm}
		\STATE \label{line:mom} 
		Update 
		$\bm{v}_{k,\tau}^{(i)} = \widetilde{\nabla} f_i(\bm{w}_{k,\tau}^{(i)};\mathcal{B}_{k,\tau}^{(i)}) + \big(\bm{v}_{k,\tau-1}^{(i)}  - \widetilde{\nabla} f_i(\bm{w}_{k,\tau-1}^{(i)};\mathcal{B}_{k,\tau}^{(i)})\big)$.
		\text{// {\color{blue} \texttt{(Local Momentum)}}} 
		\vspace{0.1 cm}
		\ENDIF
		\vspace{0.1 cm}
		\STATE Update $\bm{w}_{k,\tau+1}^{(i)} = \bm{w}_{k,\tau}^{(i)} - \eta_{k}\bm{v}_{k,\tau}^{(i)}$.
		\vspace{0.1 cm}
		\ENDFOR
		\vspace{0.1 cm}
		\STATE Send $Q_{D}({\bm{w}_{k,E}^{(i)} - \bm{w}_{k}})$ to the server.
		\vspace{0.1 cm}
	\end{algorithmic}
\end{algorithm}
\subsection{Main Result for {\texttt{FedLOMO}}}
\label{sec:result:lomo}
Now, we present the convergence result of \texttt{FedLOMO} for the smooth non-convex case in \Cref{fl-thm3}. Its proof is in \Cref{sec-pf-1}. Here, we assume that \Cref{as-het} holds for \texttt{FedLOMO}; we restate it below.
\begin{assumption}[\textbf{\Cref{as-het} for \texttt{FedLOMO}}]
\label{as-het2}
Suppose all clients participate, i.e. $r=n$, in the $(k+1)^{\text{st}}$ round of \texttt{FedLOMO} (Alg. \ref{alg:1} and \ref{alg:1-local}). 
Let $\bm{w}_{k,\tau}^{(i)}$ be the $i^{\text{th}}$ client's local parameter at the $(\tau+1)^{\text{st}}$ local step of the $(k+1)^{\text{st}}$ round of \texttt{FedLOMO}, for $i \in [n]$. Define $\widetilde{\bm{e}}_{k,\tau}^{(i)} \triangleq \nabla f_i(\bm{w}_{k,\tau}^{(i)}) - \nabla f_i(\overline{\bm{w}}_{k,\tau})$, where $\overline{\bm{w}}_{k,\tau} \triangleq \frac{1}{n}\sum_{i \in [n]} \bm{w}_{k,\tau}^{(i)}$.
Then for some $\alpha \ll n$: 
\[\mathbb{E}\Big[\Big\|\sum_{i \in [n]}\widetilde{\bm{e}}_{k, \tau}^{(i)}\Big\|^2\Big] \leq \alpha \sum_{i \in [n]} \mathbb{E}\Big[\Big\|\widetilde{\bm{e}}_{k,\tau}^{(i)}\Big\|^2\Big], \text{ } \forall \text{ } \tau \in [E].\]
\end{assumption}

\begin{theorem}[\textbf{Smooth non-convex case for \texttt{FedLOMO}}]
\label{fl-thm3}
Suppose Assumptions \ref{as1}, \ref{as-may15}, \ref{as5} and \ref{as-het2} (i.e., \Cref{as-het} for \texttt{FedLOMO} instead of \texttt{FedGLOMO}) hold. Define a distribution $\mathbb{P}$ for $k \in \{0,\ldots,K-1\}$ such that $\mathbb{P}(k) = \frac{(1+\zeta)^{(K-1-k)}}{\sum_{k=0}^{K-1}(1+\zeta)^k}$ where $\zeta$ will be defined later. Sample $k^{*}$ from $\mathbb{P}$. In \texttt{FedLOMO}, set $\eta_{k} = \frac{1}{8 L E \sqrt{B K}}$ where $B = \frac{q}{n} + \frac{4(1+q)(n-r)}{r(n-1)}$. Then for $K > \frac{1}{64 B^3} (\frac{1}{n}(\alpha + \frac{4}{E}))$:
\begin{flalign*}
    \nonumber
    & \mathbb{E}[\|\nabla f(\bm{w}_{k^{*}})\|^2] \leq \frac{64 \sqrt{B} L f(\bm{w}_0)}{K^{1/2}} \text{ with }\zeta := \frac{1}{4K} + \frac{1}{16 (B K)^{1.5}} \Big(\frac{1}{n}\Big(\alpha + \frac{4}{E}\Big)\Big). 
\end{flalign*}
So \texttt{FedLOMO} needs $K = \mathcal{O}(\frac{1}{r \epsilon^2})$ rounds of communication to achieve $\mathbb{E}[\|\nabla f(\bm{w}_{k^{*}})\|^2] \leq \epsilon$, for $\epsilon < \mathcal{O}(\frac{n B^2}{\alpha}) = \mathcal{O}(\frac{n/\alpha}{r^2})$.
\end{theorem}
We make some remarks to discuss implications of this result and establish connections to some claims made in the main paper.
\begin{remark}[\textbf{Worse
iteration complexity than \texttt{FedGLOMO}}]
Since we do not have any constraint on $E$ depending on $\epsilon$, $T = KE$ is also $\mathcal{O}(\frac{1}{r\epsilon^{2}})$ as per the above theorem. So the iteration complexity of \texttt{FedLOMO} is poorer than that of \texttt{FedGLOMO}. However, it is on a par with the results of \cite{haddadpour2020federated, koloskova2020unified, wang2019slowmo, karimireddy2019scaffold}.
\label{r3-1}
\end{remark}

\begin{remark}[\textbf{High variance of simple averaging at the server}]
At a high level, \texttt{FedLOMO} fails to attain the complexity of \texttt{FedGLOMO} because of the high variance of the \texttt{FedAvg}-like plain averaging step at the server. The high variance is itself due to the amplified effect of client-heterogeneity with multiple local updates. Without the application of some \textit{global variance-reduction} technique (like the one in \texttt{FedGLOMO}), the complexity cannot be improved.
More precisely, in \Cref{fl-thm3}, $B$ is a constant that is not $\mathcal{O}(\eta L E)$ in general, due to which \texttt{FedLOMO} does not achieve the improved convergence rate of $\mathcal{O}(K^{-2/3})$ that \texttt{FedGLOMO} attains; see the proof of \Cref{fl-thm3} in \Cref{sec-pf-1} for more details.
However, in the special case of no compression and full-device participation (i.e., $r=n$), $B$ is 0 which allows \texttt{FedLOMO} to also achieve $\mathcal{O}(K^{-2/3})$ convergence by choosing $\eta = \mathcal{O}(\frac{1}{LEK^{1/3}})$.
\label{r3-1-4}
\end{remark}

\section{Detailed Proofs}
\label{sec-res-pf}

\subsection{Detailed  Proof of the Result of \texttt{FedGLOMO}:}
\label{sec-pf-2}
\textbf{Some definitions used in the proofs}: 
\[\bm{\delta}_k^{(i)} \triangleq \mathbb{E}_{\mathcal{B}_{1}^{(i)},\ldots,\mathcal{B}_{E-1}^{(i)}}[\bm{w}_k - \bm{w}_{k,E}^{(i)}] \text{ for any $E-1$ batches $\{\mathcal{B}_{1}^{(i)},\ldots,\mathcal{B}_{E-1}^{(i)}\}$ in client $i$, and } \overline{\bm{\delta}}_k \triangleq \frac{1}{n}\sum_{i \in [n]}\bm{\delta}_k^{(i)}.\]
\[{g}_Q(\bm{w}_k;\mathcal{S}_k) \triangleq \frac{1}{r}\sum_{i \in \mathcal{S}_k}Q_D({\bm{w}_{k} - \bm{w}_{k,E}^{(i)}})\]
\[\Delta{g}_Q(\bm{w}_k,\bm{w}_{k-1};\mathcal{S}_k) \triangleq \frac{1}{r}\sum_{i \in \mathcal{S}_k}Q_D(({\bm{w}_{k} - \bm{w}_{k,E}^{(i)}}) - ({\bm{w}_{k-1} - \widehat{\bm{w}}_{k-1,E}^{(i)}}))\]
\[{g}(\bm{w}_k;\mathcal{S}_k) \triangleq \frac{1}{r}\sum_{i \in \mathcal{S}_k}({\bm{w}_{k} - \bm{w}_{k,E}^{(i)}}) = \mathbb{E}_{Q_D}[{g}_Q(\bm{w}_k;\mathcal{S}_k)]\]
\[\widehat{g}(\bm{w}_{k-1};\mathcal{S}_k) \triangleq \frac{1}{r}\sum_{i \in \mathcal{S}_k}({\bm{w}_{k-1} - \widehat{\bm{w}}_{k-1,E}^{(i)}})\]
\[\overline{\bm{w}}_{k,\tau} \triangleq \frac{1}{n}\sum_{i \in [n]} \bm{w}_{k,\tau}^{(i)} \text{ and } \overline{\bm{v}}_{k,\tau} \triangleq \frac{1}{n}\sum_{i \in [n]} \bm{v}_{k,\tau}^{(i)}\]
\[{\bm{e}}_{k,\tau}^{(i)} \triangleq \bm{v}_{k,\tau}^{(i)} - \nabla f_i(\bm{w}_{k, \tau}^{(i)}) \text{ and } \widetilde{\bm{e}}_{k,\tau}^{(i)} \triangleq \nabla f_i(\bm{w}_{k, \tau}^{(i)}) - \nabla f_i(\overline{\bm{w}}_{k,\tau})\]
\\
\\
\noindent \textbf{Proof of \Cref{nov4-thm1}:}
\begin{proof}
We set $\eta_k = \eta$ and $\beta_k = \beta$ $\forall$ $k \in \{0,\ldots,K-1\}$.
\\
Then using \Cref{nov-1-lem0}, we have that:
\begin{multline}
    \label{eq:nov-4-thm1-1}
    \mathbb{E}[f(\bm{w}_{k'})] \leq 
    f(\bm{w}_{0}) -\frac{\eta E}{4}\sum_{k=0}^{k'-1}\mathbb{E}[\|\nabla f(\bm{w}_{k})\|^2] 
    + \frac{16 \eta^3 L^2 E^2 (\alpha E + 4) }{n^2} \sum_{k=0}^{k'-1}\sum_{i \in [n]}\mathbb{E}[\|\nabla f_i(\bm{w}_k)\|^2]
    \\
    + \frac{5}{4 \eta E \beta }{\mathbb{E}[\|\bm{u}_{0} - \overline{\bm{\delta}}_{0}\|^2]} 
    + 
    160\eta E \beta \Big(\frac{q}{n^2} + \frac{(1+q)}{r(n-1)}\Big(1 - \frac{r}{n}\Big)\Big)
    \sum_{k=0}^{k'-1}\sum_{i \in [n]} \mathbb{E}[\|\nabla f_i(\bm{w}_k)\|^2],
\end{multline}
for $4\eta L E^2 \leq 1$ and $\beta \geq \frac{80 e^2 (1+q) \eta^2 L^2 E^2 (E+1)^2}{(1 - 4\eta L E)}$. 
\\
Suppose we use full batch sizes for the local updates as well as the server update at $k = 0$ (the latter means $r=n$ only for $k=0$). Then, $\bm{u}_{0} = \overline{\bm{\delta}}_{0}$ above. Also, since the $f_i$'s are $L$-smooth, using \Cref{lem1-oct20}, we have that: 
\[\sum_{i \in [n]} \mathbb{E}[\|\nabla f_i(\bm{w}_k)\|^2 \leq \sum_{i \in [n]} 2L(\mathbb{E}[f_i(\bm{w}_k)] - f_i^{*}) \leq 2n L \mathbb{E}[f(\bm{w}_k)] - 2L \sum_{i \in [n]} f_i^{*} \leq 2n L \mathbb{E}[f(\bm{w}_k)].\]
The last step above follows because the $f_i^{*}$'s are non-negative. This trick allows us to circumvent the need for the bounded client dissimilarity assumption.

Using these in (\ref{eq:nov-4-thm1-1}), we get:
\begin{multline}
    \label{eq:may15-1}
    \mathbb{E}[f(\bm{w}_{k'})] \leq 
    f(\bm{w}_{0}) -\frac{\eta E}{4}\sum_{k=0}^{k'-1}\mathbb{E}[\|\nabla f(\bm{w}_{k})\|^2] 
    \\
    + \underbrace{64 \eta L E \Big(\frac{\eta^2 L^2 E (\alpha E + 4)}{n} + 5 \beta \Big(\frac{q}{n} + \frac{(1+q)(n-r)}{r(n-1)}\Big)\Big)}_{=\gamma}\sum_{k=0}^{k'-1} \mathbb{E}[f(\bm{w}_{k})].
\end{multline}
Using (\ref{eq:may15-1}) recursively, we get:
\begin{flalign}
    \sum_{k=0}^{k'-1}\mathbb{E}[f(\bm{w}_{k})] & \leq k' f(\bm{w}_{0}) - \frac{\eta E}{4}\sum_{k=0}^{k'-2} (k'-1-k)\mathbb{E}[\|\nabla f(\bm{w}_{k})\|^2] + \gamma \sum_{k=0}^{k'-2} (k'-1-k) \mathbb{E}[f(\bm{w}_{k})]
    \\
    \label{eq:may15-2}
    & \leq k' f(\bm{w}_{0}) - \frac{\eta E}{4}\sum_{k=0}^{k'-1} \mathbb{E}[\|\nabla f(\bm{w}_{k})\|^2] + \gamma k' \sum_{k=0}^{k'-1} \mathbb{E}[f(\bm{w}_{k})].
\end{flalign}
Let us now ensure that $\gamma k' \leq \frac{1}{2}$ for all $k' \in \{1,\ldots,K\}$, in which case we can simplify (\ref{eq:may15-2}) to:
\begin{equation}
    \label{eq:may15-3}
    \sum_{k=0}^{k'-1}\mathbb{E}[f(\bm{w}_{k})] \leq 2k' f(\bm{w}_{0}) - \frac{\eta E}{2}\sum_{k=0}^{k'-1} \mathbb{E}[\|\nabla f(\bm{w}_{k})\|^2] \leq 2k' f(\bm{w}_{0}).
\end{equation}
Now:
\begin{equation}
    \gamma k' \leq \gamma K = 64 \eta L E \Big(\frac{\eta^2 L^2 E (\alpha E + 4)}{n} + 5 \beta \Big(\frac{q}{n} + \frac{(1+q)(n-r)}{r(n-1)}\Big)\Big)K.
\end{equation}

{Now if we set $8 \eta L E^2 \leq 1$, then it can be verified that $\beta = 160e^2 (1+q) \eta^2 L^2 E^2 (E+1)^2$ is a valid choice. Using this above, we get that:
\begin{equation}
    \gamma k' \leq \gamma K = \underbrace{64 \eta^3 L^3 E^3 K \Big(\frac{1}{n}\Big(\alpha + \frac{4}{E}\Big) + 800 e^2 (1+q) (E+1)^2 \Big(\frac{q}{n} + \frac{(1+q)(n-r)}{r(n-1)}\Big)\Big)}_{\text{(A)}}.
\end{equation}
Setting $\eta = \frac{1}{6 L E K^{1/3} (\frac{1}{n}(\alpha + \frac{4}{E}) + 800 e^2 (1+q) (E+1)^2 (\frac{q}{n} + \frac{(1+q)(n-r)}{r (n-1)}))^{1/3}}$, we have (A) $ < \frac{1}{2}$. But we must also have
\begin{equation}
    \label{eq:may15-4}
    8 \eta L E^2 = \frac{4 E}{3 K^{1/3} \big(\frac{1}{n}\big(\alpha + \frac{4}{E}\big) + 800 e^2 (1+q) (E+1)^2 \big(\frac{q}{n} + \frac{(1+q)(n-r)}{r (n-1)}\big)\big)^{1/3}} \leq 1. 
\end{equation}
This holds for $K^{1/3} (E+1) \geq \frac{1}{1200e^2(1+q) \big(\frac{q}{n} + \frac{(1+q)(n-r)}{r (n-1)}\big)}$.

Further $\beta$ must be smaller than 1, so
\begin{equation}
    \beta = 160 e^2 (1+q) \eta^2 L^2 E^2 (E+1)^2 = \frac{160 e^2 (1+q) (E+1)^2}{36 K^{2/3} \big(\big(\frac{1}{n}\big(\alpha + \frac{4}{E}\big) + 800 e^2 (1+q) (E+1)^2 \big(\frac{q}{n} + \frac{(1+q)(n-r)}{r (n-1)}\big)\big)^{2/3}} < 1.
\end{equation}
This holds for $E+1 \leq \frac{\sqrt{1+q} (n-r)}{3 r(n-1)} K$.

Now using (\ref{eq:may15-3}) in (\ref{eq:may15-1}) with $k' = K$ and our choice of $\beta = 160 e^2 (1+q) \eta^2 L^2 E^2 (E+1)^2$ and $\eta = \frac{1}{6 L E K^{1/3} (\frac{1}{n}(\alpha + \frac{4}{E}) + 800 e^2 (1+q) (E+1)^2 (\frac{q}{n} + \frac{(1+q)(n-r)}{r (n-1)}))^{1/3}}$, we get:
\begin{multline}
    \label{eq:may15-5}
    \mathbb{E}[f(\bm{w}_{K})] \leq 
    f(\bm{w}_{0}) -\frac{\eta E}{4}\sum_{k=0}^{K-1}\mathbb{E}[\|\nabla f(\bm{w}_{k})\|^2] 
    \\
    + {128 \eta^3 L^3 E^3 \Big(\frac{1}{n}\Big(\alpha + \frac{4}{E}\Big) + 800 e^2 (1+q)(E+1)^2 \Big(\frac{q}{n} + \frac{(1+q)(n-r)}{r(n-1)}\Big)\Big)}K f(\bm{w}_0).
\end{multline}
Rearranging the above a bit and using the fact that $f(\bm{w}_K) \geq 0$, we get:
\small
\begin{equation}
    \label{eq:may15-6}
    \frac{1}{K} \sum_{k=0}^{K-1}\mathbb{E}[\|\nabla f(\bm{w}_{k})\|^2] \leq \frac{4 f(\bm{w}_{0})}{\eta E K} + 512 \eta^2 L^3 E^2 \Big(\frac{1}{n}\Big(\alpha + \frac{4}{E}\Big) + 800 e^2 (1+q) (E+1)^2 \Big(\frac{q}{n} + \frac{(1+q)(n-r)}{r(n-1)}\Big)\Big) f(\bm{w}_0).
\end{equation}
\normalsize
Substituting the value of $\eta$ above, we get:
\begin{equation}
    \label{eq:may15-7}
    \frac{1}{K} \sum_{k=0}^{K-1}\mathbb{E}[\|\nabla f(\bm{w}_{k})\|^2] \leq \frac{39 L f(\bm{w}_{0})}{K^{2/3}} \Big({\frac{1}{n}\Big(\alpha + \frac{4}{E}\Big)} + 800 e^2 (1+q) (E+1)^2 \Big(\frac{q}{n} + \frac{(1+q)(n-r)}{r(n-1)}\Big)\Big)^{1/3}.
\end{equation}
This concludes the proof.
}
\end{proof}

\noindent \textbf{Lemmas used in the proof of \Cref{nov4-thm1}:}

\begin{lemma}
\label{nov-1-lem0}
Suppose $4\eta L E^2 \leq 1$ and $\beta \geq \frac{80 e^2 (1+q) \eta^2 L^2 E^2 (E+1)^2}{(1 - 4\eta L E)}$. Then for any $k' \in \{1,\ldots,K\}$, we have:
\begin{multline*}
    \mathbb{E}[f(\bm{w}_{k'})] \leq 
    f(\bm{w}_{0}) -\frac{\eta E}{4}\sum_{k=0}^{k'-1}\mathbb{E}[\|\nabla f(\bm{w}_{k})\|^2] 
    + {\frac{16 \eta^3 L^2 E^2 (\alpha E + 4)}{n^2}} \sum_{k=0}^{k'-1}\sum_{i \in [n]}\mathbb{E}[\|\nabla f_i(\bm{w}_k)\|^2]
    \\
    + \frac{5}{4 \eta E \beta }{\mathbb{E}[\|\bm{u}_{0} - \overline{\bm{\delta}}_{0}\|^2]} 
    + 
    160\eta E \beta \Big(\frac{q}{n^2} + \frac{(1+q)}{r(n-1)}\Big(1 - \frac{r}{n}\Big)\Big)
    \sum_{k=0}^{k'-1}\sum_{i \in [n]} \mathbb{E}[\|\nabla f_i(\bm{w}_k)\|^2].
\end{multline*}
\end{lemma}
\begin{proof}
Per the previous definitions:
\begin{equation}
    \label{eq:nov-1-thm1-1}
    \bm{u}_k = \beta {g}_Q(\bm{w}_k;\mathcal{S}_k) + (1-\beta)\bm{u}_{k-1} + (1-\beta)\Delta{g}_Q(\bm{w}_k,\bm{w}_{k-1};\mathcal{S}_k)
\end{equation}
By $L$-smoothness of $f$, we have for $k \geq 1$:
\begin{flalign}
    \nonumber
    \mathbb{E}[f(\bm{w}_{k+1})] & \leq \mathbb{E}[f(\bm{w}_{k})] + \mathbb{E}[\langle \nabla f(\bm{w}_{k}), \bm{w}_{k+1} - \bm{w}_{k} \rangle] + \frac{L}{2}\mathbb{E}[\|\underbrace{\bm{w}_{k+1} - \bm{w}_{k}}_{=-\bm{u}_k}\|^2]
    \\
    \label{eq:nov-1-thm1-1-0}
    & = \mathbb{E}[f(\bm{w}_{k})] + \underbrace{\mathbb{E}[\langle \nabla f(\bm{w}_{k}), -\bm{u}_{k}\rangle]}_\text{(I*)} 
    + \underbrace{\frac{1}{8\eta E}\mathbb{E}[\|\bm{u}_{k}\|^2]}_\text{(II*)}
    - \Big(\frac{1}{8\eta E} - \frac{L}{2}\Big)\mathbb{E}[\|\bm{w}_{k+1} - \bm{w}_{k}\|^2].
\end{flalign}
Let us analyze (I*) first. 
\begin{flalign}
    \label{eq:nov6-1}
    \mathbb{E}[\langle \nabla f(\bm{w}_{k}), -\bm{u}_{k}\rangle] & = \mathbb{E}[\langle \nabla f(\bm{w}_{k}), -{g}(\bm{w}_k;\mathcal{S}_k) - (1 - \beta)(\bm{u}_{k-1} - \widehat{g}(\bm{w}_{k-1};\mathcal{S}_k)) \rangle]
    \\
    \nonumber
    & = \mathbb{E}[\langle \nabla f(\bm{w}_{k}), -{g}(\bm{w}_k;\mathcal{S}_k)] - (1 - \beta)\mathbb{E}[\langle \nabla f(\bm{w}_{k}), \bm{u}_{k-1} - \widehat{g}(\bm{w}_{k-1};\mathcal{S}_k) \rangle]
    \\
    \nonumber
    & = \underbrace{\mathbb{E}[\langle \nabla f(\bm{w}_{k}), \frac{1}{n}\sum_{i \in [n]}({\bm{w}_{k,E}^{(i)}} - \bm{w}_{k}) \rangle]}_\text{(III*)} + \underbrace{(1-\beta)\mathbb{E}[\langle - \nabla f(\bm{w}_{k}), \bm{u}_{k-1} - \overline{\bm{\delta}}_{k-1}
    \rangle]}_\text{(IV*)}
\end{flalign}
(\ref{eq:nov6-1}) follows by taking expectation with respect to $Q_D$. (III*) is obtained by taking expectation with respect to $\mathcal{S}_k$ above. (IV*) is obtained by taking expectation with respect to $\{\mathcal{B}_{k,1}^{(i)},\ldots,\mathcal{B}_{k,E-1}^{(i)}\}_{i=1}^{n}$ and $\mathcal{S}_k$ above.

From \Cref{lem1-may11}, for $\eta < \frac{1}{L}$ and $E < \frac{1}{4}\text{min}\Big(\frac{1}{\eta L}, \frac{1}{\eta^2 L^2} - \frac{1}{\eta L}\Big)$, we can bound (III*) as:
\small
\begin{multline}
    \label{eq:nov-1-thm1-2}
    \text{(III*)} \leq -\frac{\eta E}{2}\mathbb{E}[\|\nabla f(\bm{w}_{k})\|^2] - \frac{\eta}{2}{\Big(1 - \eta^2 L^2 E^2 \Big)}\sum_{\tau=0}^{E-1}\mathbb{E}[\|\overline{\bm{v}}_{k,\tau}\|^2] 
    + \frac{16 \eta^3 L^2 E^2 (\alpha E + 4)}{n^2}\sum_{i \in [n]}\mathbb{E}[\|\nabla f_i(\bm{w}_k)\|^2].
\end{multline}
\normalsize
Note that for $\eta L \ll 1$ (which is going to be the case eventually), we can combine all the above constraints on $\eta$ and $E$ into $4 \eta L E < 1$.

As for (IV*):
\begin{flalign}
    \text{(IV*)}
    & \leq (1-\beta)\mathbb{E}\big[\|\nabla f(\bm{w}_{k})\| \|\bm{u}_{k-1} - \overline{\bm{\delta}}_{k-1}\|\big]
    \\
    \label{eq:nov-1-thm1-3-i}
    & \leq \frac{(1-\beta)}{2}\Big(\frac{\eta E}{2(1-\beta)}{\mathbb{E}[\|\nabla f(\bm{w}_{k})\|^2]} + \frac{2(1-\beta)\mathbb{E}[\|\bm{u}_{k-1} - \overline{\bm{\delta}}_{k-1}\|^2]}{\eta E}\Big) 
    \\
    \label{eq:nov-1-thm1-3}
    & = \frac{\eta E}{4}\mathbb{E}[\|\nabla f(\bm{w}_{k})\|^2] + \frac{(1-\beta)^2}{\eta E}\mathbb{E}[\|\bm{u}_{k-1} - \overline{\bm{\delta}}_{k-1}\|^2].
\end{flalign}
(\ref{eq:nov-1-thm1-3-i}) above follows by the AM-GM inequality.
\\
Adding (\ref{eq:nov-1-thm1-2}) and (\ref{eq:nov-1-thm1-3}), we get:
\begin{multline}
    \label{eq:nov-1-thm1-4}
    \text{(I*)} \leq -\frac{\eta E}{4}\mathbb{E}[\|\nabla f(\bm{w}_{k})\|^2] -\frac{\eta}{2}(1 - \eta^2  L^2 E^2)\sum_{\tau=0}^{E-1}\mathbb{E}[\|\overline{\bm{v}}_{k,\tau}\|^2] + \frac{(1-\beta)^2}{\eta E}\mathbb{E}[\|\bm{u}_{k-1} - \overline{\bm{\delta}}_{k-1}\|^2]
    \\
    + \frac{16 \eta^3 L^2 E^2 (\alpha E + 4)}{n^2}\sum_{i \in [n]}\mathbb{E}[\|\nabla f_i(\bm{w}_k)\|^2].
\end{multline}
Now, let us analyze (II*). We have:
\begin{flalign}
    \label{eq:nov-1-thm1-4-2}
    \mathbb{E}[\|\bm{u}_k\|^2] \leq 2 \mathbb{E}[\|\overline{\bm{\delta}}_k\|^2] + 2 \mathbb{E}[\|\bm{u}_k - \overline{\bm{\delta}}_k\|^2]
\end{flalign}
Notice that:
\begin{equation}
    \label{eq:nov-1-thm1-4-2-0}
    \overline{\bm{\delta}}_k = \mathbb{E}_{\{\mathcal{B}_{k,1}^{(i)},\ldots,\mathcal{B}_{k,E-1}^{(i)}\}_{i=1}^{n}}\Big[\frac{1}{n}\sum_{i \in [n]}(\bm{w}_k - \bm{w}_{k,E}^{(i)})\Big] = \mathbb{E}_{\{\mathcal{B}_{k,1}^{(i)},\ldots,\mathcal{B}_{k,E-1}^{(i)}\}_{i=1}^{n}}[\sum_{\tau=0}^{E-1} \eta \overline{\bm{v}}_{k,\tau}].
\end{equation}
Thus:
\begin{equation}
    \label{eq:nov-1-thm1-5}
    \mathbb{E}[\|\overline{\bm{\delta}}_k\|^2] \leq  \eta^2 \mathbb{E}_{}\Big[\Big\|\sum_{\tau=0}^{E-1} \overline{\bm{v}}_{k,\tau}\Big\|^2\Big] \leq E \eta^2 \sum_{\tau=0}^{E-1} \mathbb{E}_{}[\| \overline{\bm{v}}_{k,\tau}\|^2].
\end{equation}
The expectation above is with respect to all the randomness in the algorithm so far.
\\
Using (\ref{eq:nov-1-thm1-5}) and the result of \Cref{nov-1-lem2} in (\ref{eq:nov-1-thm1-4-2}) with $2 \eta L E^2 \leq 1$, we have that:
\begin{multline}
    \label{eq:nov-1-thm1-6}
    \mathbb{E}[\|\bm{u}_k\|^2] \leq 2 E \eta^2 \sum_{\tau=0}^{E-1} \mathbb{E}_{}[\| \overline{\bm{v}}_{k,\tau}\|^2] + 2 \Big\{(1-\beta)^2\mathbb{E}[\|\bm{u}_{k-1} - \overline{\bm{\delta}}_{k-1}\|^2] + 2 \beta^2 \mathbb{E}[\|{g}_Q(\bm{w}_k;\mathcal{S}_k) - \overline{\bm{\delta}}_k\|^2] 
    \\
    + 8 e^2 (1+q) (1-\beta)^2 \eta^2 L^2 E^2 (E+1)^2 \mathbb{E}[\|\bm{w}_{k} - \bm{w}_{k-1}\|^2]\Big\}.
\end{multline}
Recalling that (II*) = $\frac{1}{8\eta E} \mathbb{E}[\|\bm{u}_k\|^2]$, we get:
\begin{multline}
    \label{eq:nov-1-thm1-8}
    \text{(II*)} \leq \frac{\eta}{4} \sum_{\tau=0}^{E-1} \mathbb{E}_{}[\| \overline{\bm{v}}_{k,\tau}\|^2] + \frac{1}{4\eta E} \Big\{(1-\beta)^2\mathbb{E}[\|\bm{u}_{k-1} - \overline{\bm{\delta}}_{k-1}\|^2] 
    + 2 \beta^2 \mathbb{E}[\|{g}_Q(\bm{w}_k;\mathcal{S}_k) - \overline{\bm{\delta}}_k\|^2] 
    \\
    + 8 e^2 (1+q) (1-\beta)^2 \eta^2 L^2 E^2 (E+1)^2 \mathbb{E}[\|\bm{w}_{k} - \bm{w}_{k-1}\|^2]\Big\}.
\end{multline}
Adding (\ref{eq:nov-1-thm1-4}) and (\ref{eq:nov-1-thm1-8}):
\begin{multline}
    \label{eq:nov-1-thm1-9}
    \text{(I*)} + \text{(II*)} \leq -\frac{\eta E}{4}\mathbb{E}[\|\nabla f(\bm{w}_{k})\|^2] -\frac{\eta}{2}\underbrace{\Big(1 - \eta^2  L^2 E^2 - \frac{1}{2}\Big)}_\text{$>0$ for $4 \eta L E \leq 1$}\sum_{\tau=0}^{E-1}\mathbb{E}[\|\overline{\bm{v}}_{k,\tau}\|^2] 
    \\
    + \frac{16 \eta^3 L^2 E^2}{n^2}
    (\alpha E + 4)
    \sum_{i \in [n]}\mathbb{E}[\|\nabla f_i(\bm{w}_k)\|^2]
    + \frac{5(1-\beta)^2}{4\eta E}\underbrace{\mathbb{E}[\|\bm{u}_{k-1} - \overline{\bm{\delta}}_{k-1}\|^2]}_\text{from \Cref{nov-1-lem2}}
    \\
    + \frac{\beta^2}{2 \eta E} \mathbb{E}[\|{g}_Q(\bm{w}_k;\mathcal{S}_k) - \overline{\bm{\delta}}_k\|^2] 
    + 2 e^2 (1+q) (1-\beta)^2 \eta L^2 E (E+1)^2 \mathbb{E}[\|\bm{w}_{k} - \bm{w}_{k-1}\|^2].
\end{multline}
Therefore, using 
\Cref{nov-1-lem2} 
recursively, we get:
\begin{multline}
    \label{eq:nov-1-thm1-11}
    \text{(I*)} + \text{(II*)} \leq -\frac{\eta E}{4}\mathbb{E}[\|\nabla f(\bm{w}_{k})\|^2] 
    + \frac{16 \eta^3 L^2 E^2 (\alpha E + 4) }{n^2}\sum_{i \in [n]}\mathbb{E}[\|\nabla f_i(\bm{w}_k)\|^2]
    \\
    + \frac{5(1-\beta)^{2k}}{4\eta E}{\mathbb{E}[\|\bm{u}_{0} - \overline{\bm{\delta}}_{0}\|^2]} 
    + \frac{5\beta^2}{2\eta E}
    \sum_{l=1}^{k}(1-\beta)^{2(k-l)}\underbrace{\mathbb{E}[\|{g}_Q(\bm{w}_l;\mathcal{S}_l) - \overline{\bm{\delta}}_l\|^2]}_\text{(V*)}
    \\
    + 10 e^2 (1+q) \eta L^2 E (E+1)^2 \sum_{l=1}^{k}(1-\beta)^{2(k-l+1)}\mathbb{E}[\|\bm{w}_{l} - \bm{w}_{l-1}\|^2].
\end{multline}
Using \Cref{lem-may11-n1}, we get:
\begin{equation}
    \text{(V*)} \leq 4 \eta^2 E\Big(\frac{q}{n^2} + \frac{(1+q)}{r(n-1)}\Big(1 - \frac{r}{n}\Big)\Big)\sum_{i \in [n]}\sum_{\tau=0}^{E-1}\mathbb{E}[\|\bm{v}_{k,\tau}^{(i)}\|^2].
\end{equation}

Putting this back in (\ref{eq:nov-1-thm1-11}), we get:
\begin{multline}
    \label{eq:nov-1-thm1-12}
    \text{(I*)} + \text{(II*)} \leq -\frac{\eta E}{4}\mathbb{E}[\|\nabla f(\bm{w}_{k})\|^2]
    + \frac{16 \eta^3 L^2 E^2 (\alpha E + 4) }{n^2} \sum_{i \in [n]}\mathbb{E}[\|\nabla f_i(\bm{w}_k)\|^2]
    \\
    + \frac{5(1-\beta)^{2k}}{4\eta E}{\mathbb{E}[\|\bm{u}_{0} - \overline{\bm{\delta}}_{0}\|^2]}
    + 
    10\eta \beta^2 \Big(\frac{q}{n^2} + \frac{(1+q)}{r(n-1)}\Big(1 - \frac{r}{n}\Big)\Big)
    \sum_{l=1}^{k}(1-\beta)^{2(k-l)}\sum_{i \in [n]}\sum_{\tau=0}^{E-1}\mathbb{E}[\|\bm{v}_{l,\tau}^{(i)}\|^2]
    \\
    + 10 e^2 (1+q) \eta L^2 E (E+1)^2 \sum_{l=1}^{k}(1-\beta)^{2(k-l+1)}\mathbb{E}[\|\bm{w}_{l} - \bm{w}_{l-1}\|^2].
\end{multline}
Next, using (\ref{eq:nov-1-thm1-12}) in (\ref{eq:nov-1-thm1-1-0}), we get that:
\begin{multline}
    \label{eq:nov-1-thm1-13}
    \mathbb{E}[f(\bm{w}_{k+1})] \leq 
    \mathbb{E}[f(\bm{w}_{k})] -\frac{\eta E}{4}\mathbb{E}[\|\nabla f(\bm{w}_{k})\|^2] 
    + \frac{16 \eta^3 L^2 E^2 (\alpha E + 4) }{n^2} \sum_{i \in [n]} \mathbb{E}[\|\nabla f_i(\bm{w}_k)\|^2]
    \\
    + \frac{5(1-\beta)^{2k}}{4\eta E}{\mathbb{E}[\|\bm{u}_{0} - \overline{\bm{\delta}}_{0}\|^2]} 
    + 
    10\eta \beta^2 \Big(\frac{q}{n^2} + \frac{(1+q)}{r(n-1)}\Big(1 - \frac{r}{n}\Big)\Big)
    \sum_{l=1}^{k}(1-\beta)^{2(k-l)}\sum_{i \in [n]}\sum_{\tau=0}^{E-1}\mathbb{E}[\|\bm{v}_{l,\tau}^{(i)}\|^2]
    \\
    + 10 e^2 (1+q) \eta L^2 E (E+1)^2 \sum_{l=1}^{k}(1-\beta)^{2(k-l+1)}\mathbb{E}[\|\bm{w}_{l} - \bm{w}_{l-1}\|^2]
    - \Big(\frac{1}{8\eta E} - \frac{L}{2}\Big)\mathbb{E}[\|\bm{w}_{k+1} - \bm{w}_{k}\|^2].
\end{multline}
Summing the above from $k=0$ through $(k'-1)$ for any $k' \in \{1,\ldots,K\}$, we get:
\begin{multline}
    \label{eq:nov-1-thm1-14}
    \mathbb{E}[f(\bm{w}_{k'})] \leq 
    f(\bm{w}_{0}) -\frac{\eta E}{4}\sum_{k=0}^{k'-1}\mathbb{E}[\|\nabla f(\bm{w}_{k})\|^2] 
    + \frac{16 \eta^3 L^2 E^2 (\alpha E + 4) }{n^2} \sum_{k=0}^{k'-1}\sum_{i \in [n]} \mathbb{E}[\|\nabla f_i(\bm{w}_k)\|^2]
    \\
    + \sum_{l=0}^{\infty}\frac{5(1-\beta)^{2l}}{4\eta E}{\mathbb{E}[\|\bm{u}_{0} - \overline{\bm{\delta}}_{0}\|^2]} 
    + 
    10 \eta \beta^2 \Big(\frac{q}{n^2} + \frac{(1+q)}{r(n-1)}\Big(1 - \frac{r}{n}\Big)\Big)
    \sum_{k=0}^{k'-1}\sum_{i \in [n]}\sum_{\tau=0}^{E-1}\mathbb{E}[\|\bm{v}_{k,\tau}^{(i)}\|^2]\sum_{l=0}^{\infty}{(1-\beta)^{2l}}
    \\
    + 10 e^2 (1+q) \eta L^2 E (E+1)^2 (1-\beta)^2
    \sum_{k=1}^{k'-1}\mathbb{E}[\|\bm{w}_{k} - \bm{w}_{k-1}\|^2]
    \sum_{l=0}^{\infty}(1-\beta)^{2l} 
    - \Big(\frac{1}{8\eta E} - \frac{L}{2}\Big)\sum_{k=0}^{k'-1}\mathbb{E}[\|\bm{w}_{k+1} - \bm{w}_{k}\|^2].
\end{multline}
Simplifying the above by noting that $\sum_{l=0}^{\infty}(1-\beta)^{2l} \leq \sum_{l=0}^{\infty}(1-\beta)^{l} = 1/\beta$, we get:
\begin{multline}
    \label{eq:nov-1-thm1-15}
    \mathbb{E}[f(\bm{w}_{k'})] \leq 
    f(\bm{w}_{0}) -\frac{\eta E}{4}\sum_{k=0}^{k'-1}\mathbb{E}[\|\nabla f(\bm{w}_{k})\|^2] 
    + \frac{16 \eta^3 L^2 E^2 (\alpha E + 4) }{n^2} \sum_{k=0}^{k'-1}\sum_{i \in [n]} \mathbb{E}[\|\nabla f_i(\bm{w}_k)\|^2]
    \\
    + \frac{5}{4 \eta E \beta }{\mathbb{E}[\|\bm{u}_{0} - \overline{\bm{\delta}}_{0}\|^2]} 
    + 
    10\eta \beta \Big(\frac{q}{n^2} + \frac{(1+q)}{r(n-1)}\Big(1 - \frac{r}{n}\Big)\Big)
    \sum_{k=0}^{k'-1}\sum_{i \in [n]}\sum_{\tau=0}^{E-1}\mathbb{E}[\|\bm{v}_{k,\tau}^{(i)}\|^2]
    \\
    + \underbrace{\frac{10 e^2 (1+q) \eta L^2 E (E+1)^2}{\beta} 
    \sum_{k=1}^{k'-1}\mathbb{E}[\|\bm{w}_{k} - \bm{w}_{k-1}\|^2] - \frac{(1-4\eta L E)}{8\eta E}\sum_{k=0}^{k'-1}\mathbb{E}[\|\bm{w}_{k+1} - \bm{w}_{k}\|^2]}_\text{(VI*) -- want this to be $\leq$ 0}
\end{multline}
We want (VI*) to be $\leq 0$. For this, we must have:
\begin{equation}
    \label{eq:nov-1-thm1-16}
    \beta \geq \frac{80 e^2 (1+q) \eta^2 L^2 E^2 (E+1)^2}{(1 - 4\eta L E)}.
\end{equation}
Note that the denominator above is positive since we already have a constraint of $4\eta L E \leq 1$. 

With $\beta$ satisfying the above constraint, and using the result of \Cref{fl-lem-new1} for $\sum_{\tau=0}^{E-1} \mathbb{E}[\|\bm{v}_{k,\tau}^{(i)}\|^2]$, we get:
\begin{multline}
    \label{eq:nov-1-thm1-17}
    \mathbb{E}[f(\bm{w}_{k'})] \leq 
    f(\bm{w}_{0}) -\frac{\eta E}{4}\sum_{k=0}^{k'-1}\mathbb{E}[\|\nabla f(\bm{w}_{k})\|^2] 
    + \frac{16 \eta^3 L^2 E^2 (\alpha E + 4) }{n^2} \sum_{k=0}^{k'-1}\sum_{i \in [n]}\mathbb{E}[\|\nabla f_i(\bm{w}_k)\|^2]
    \\
    + \frac{5}{4 \eta E \beta }{\mathbb{E}[\|\bm{u}_{0} - \overline{\bm{\delta}}_{0}\|^2]} 
    + 
    160\eta E \beta \Big(\frac{q}{n^2} + \frac{(1+q)}{r(n-1)}\Big(1 - \frac{r}{n}\Big)\Big)
    \sum_{k=0}^{k'-1}\sum_{i \in [n]} \mathbb{E}[\|\nabla f_i(\bm{w}_k)\|^2].
\end{multline}
Finally, note that we have two constraints namely: 
$4 \eta L E \leq 1$ and $2 \eta L E^2 \leq 1$. We can merge these constraints into $4 \eta L E^2 \leq 1$ for $E \geq 1$ (which is the case).

This gives us the desired result.
\end{proof}

\begin{lemma}
\label{lem1-may11}
For $\eta < \frac{1}{L}$ and $E < \frac{1}{4}\text{min}\Big(\frac{1}{\eta L}, \frac{1}{\eta^2 L^2} - \frac{1}{\eta L}\Big)$, (III*) in the proof of \Cref{nov-1-lem0} can be bounded as:
\begin{multline*}
    \text{(III*) } = \mathbb{E}[\langle \nabla f(\bm{w}_{k}), \frac{1}{n}\sum_{i \in [n]}(\bm{w}_{k,E}^{(i)} - \bm{w}_{k})\rangle] \leq -\frac{\eta E}{2}\mathbb{E}[\|\nabla f(\bm{w}_{k})\|^2] - \frac{\eta}{2}{\Big(1 - \eta^2 L^2 E^2 \Big)}\sum_{\tau=0}^{E-1}\mathbb{E}[\|\overline{\bm{v}}_{k,\tau}\|^2] 
    \\
    + {\frac{16 \eta^3 L^2 E^2 (\alpha E + 4)}{n^2}\sum_{i \in [n]}\|\nabla f_i(\bm{w}_k)\|^2}.
\end{multline*}
\end{lemma}

\begin{proof}
$\text{(III*)} = \mathbb{E}[\langle \nabla f(\bm{w}_{k}), \frac{1}{n}\sum_{i \in [n]}(\bm{w}_{k,E}^{(i)} - \bm{w}_{k})\rangle]$. Then:
\small
\begin{flalign}
    \nonumber
    \text{(III*)} & = \mathbb{E}[\langle \nabla f(\bm{w}_{k}), -\frac{1}{{n}}\sum_{i \in [n]} \sum_{\tau=0}^{E-1}\eta \bm{v}_{k,\tau}^{(i)}\rangle]
    \\
    \nonumber
    & = -{\eta}\sum_{\tau=0}^{E-1} \mathbb{E}[\langle \nabla f(\bm{w}_{k}), \underbrace{\frac{1}{n}\sum_{i \in [n]} \bm{v}_{k,\tau}^{(i)}}_{=\overline{\bm{v}}_{k,\tau}}\rangle]
    \\
    \label{eq:may11-1}
    & = \sum_{\tau=0}^{E-1}\Big\{-\frac{\eta}{2}\mathbb{E}[\|\nabla f(\bm{w}_{k})\|^2] -\frac{\eta}{2}\mathbb{E}[\|\overline{\bm{v}}_{k,\tau}\|^2] + \frac{\eta}{2}\mathbb{E}[\|\nabla f(\bm{w}_{k}) - \overline{\bm{v}}_{k,\tau}\|^2]\Big\}
    \\
    \nonumber
    & = \sum_{\tau=0}^{E-1}\Big\{-\frac{\eta}{2}\mathbb{E}[\|\nabla f(\bm{w}_{k})\|^2] -\frac{\eta}{2}\mathbb{E}[\|\overline{\bm{v}}_{k,\tau}\|^2] + \frac{\eta}{2}\mathbb{E}[\|\nabla f(\bm{w}_{k}) - \nabla f(\overline{\bm{w}}_{k,\tau}) + \nabla f(\overline{\bm{w}}_{k,\tau}) - \overline{\bm{v}}_{k,\tau}\|^2]\Big\}
    \\
    \label{eq:may11-2}
    & \leq \sum_{\tau=0}^{E-1}\Big\{-\frac{\eta}{2}\mathbb{E}[\|\nabla f(\bm{w}_{k})\|^2] -\frac{\eta}{2}\mathbb{E}[\|\overline{\bm{v}}_{k,\tau}\|^2] + {\eta}\mathbb{E}[\underbrace{\|\nabla f(\bm{w}_{k}) - \nabla f(\overline{\bm{w}}_{k,\tau})\|}_{\leq L \|\bm{w}_{k} - \overline{\bm{w}}_{k,\tau}\|}]^2 +  \eta \mathbb{E}[\|\nabla f(\overline{\bm{w}}_{k,\tau}) - \overline{\bm{v}}_{k,\tau}\|^2]\Big\}
    \\
    \label{eq:may11-3}
    & \leq \sum_{\tau=0}^{E-1}\Big\{-\frac{\eta}{2}\mathbb{E}[\|\nabla f(\bm{w}_{k})\|^2] -\frac{\eta}{2}\mathbb{E}[\|\overline{\bm{v}}_{k,\tau}\|^2] + {\eta L^2}\mathbb{E}[\|\bm{w}_{k} - \overline{\bm{w}}_{k,\tau}\|]^2 +  \eta \mathbb{E}[\|\nabla f(\overline{\bm{w}}_{k,\tau}) - \overline{\bm{v}}_{k,\tau}\|^2]\Big\}
\end{flalign}
\normalsize
(\ref{eq:may11-1}) above follows by using the fact that for any two vectors $\bm{a}$ and $\bm{b}$, $\langle \bm{a}, \bm{b} \rangle = \frac{1}{2}(\|\bm{a}\|^2 + \|\bm{b}\|^2 - \|\bm{a}-\bm{b}\|^2)$. Also, (\ref{eq:may11-2}) follows from the fact that for any two vectors $\bm{a}$ and $\bm{b}$, $\|\bm{a} + \bm{b}\|^2 \leq 2\|\bm{a}\|^2 + 2\|\bm{b}\|^2$.
\\
Per definitions, observe that:
\begin{equation}
    \label{eq:may11-4}
    \overline{\bm{w}}_{k,\tau+1} = \overline{\bm{w}}_{k,\tau} - \eta \overline{\bm{v}}_{k,\tau}.
\end{equation}
From this, we have that $\bm{w}_{k} - \overline{\bm{w}}_{k,\tau} = \eta \sum_{t=0}^{\tau-1}\overline{\bm{v}}_{k,t}$. Hence, $\|\bm{w}_{k} - \overline{\bm{w}}_{k,\tau}\|^2 = \eta^2 \|\sum_{t=0}^{\tau-1}\overline{\bm{v}}_{k,t}\|^2 \leq \eta^2 \tau \sum_{t=0}^{\tau-1} \|\overline{\bm{v}}_{k,t}\|^2$ -- this follows from the fact that for any $p > 1$ vectors $\{\bm{u}_1,\ldots,\bm{u}_p\}$, $\|\sum_{i=1}^p \bm{u}_i\|^2 \leq p \sum_{i=1}^p \|\bm{u}_i\|^2$. Using all this in (\ref{eq:may11-3}), we get:
\small
\begin{flalign}
    \nonumber
    \text{(III*)} & \leq -\frac{\eta E}{2}\mathbb{E}[\|\nabla f(\bm{w}_{k})\|^2] +  \sum_{\tau=0}^{E-1}\Big\{-\frac{\eta}{2}\mathbb{E}[\|\overline{\bm{v}}_{k,\tau}\|^2] + {\eta^3  L^2}\tau \sum_{t=0}^{\tau-1} \mathbb{E}[\|\overline{\bm{v}}_{k,t}\|^2] +  \eta \mathbb{E}[\|\nabla f(\overline{\bm{w}}_{k,\tau}) - \overline{\bm{v}}_{k,\tau}\|^2]\Big\}
    \\
    \label{eq:may11-5}
    & \leq -\frac{\eta E}{2}\mathbb{E}[\|\nabla f(\bm{w}_{k})\|^2] -\frac{\eta}{2}\sum_{\tau=0}^{E-1}\mathbb{E}[\|\overline{\bm{v}}_{k,\tau}\|^2] + \frac{\eta^3  L^2 E^2}{2}\sum_{\tau=0}^{E-1}\mathbb{E}[\|\overline{\bm{v}}_{k,\tau}\|^2] + \eta \underbrace{\sum_{\tau=0}^{E-1}\mathbb{E}[\|\nabla f(\overline{\bm{w}}_{k,\tau}) - \overline{\bm{v}}_{k,\tau}\|^2]}_\text{from \Cref{fl-lem2}}
\end{flalign}
\normalsize
Using \Cref{fl-lem2} to bound the last term above gives us:
\begin{equation}
    \text{(III*)} \leq -\frac{\eta E}{2}\mathbb{E}[\|\nabla f(\bm{w}_{k})\|^2] - \frac{\eta}{2}{\Big(1 - \eta^2 L^2 E^2 \Big)}\sum_{\tau=0}^{E-1}\mathbb{E}[\|\overline{\bm{v}}_{k,\tau}\|^2] + \frac{16 \eta^3 L^2 E^2 (\alpha E + 4)}{n^2}\sum_{i \in [n]}\|\nabla f_i(\bm{w}_k)\|^2.
\end{equation}
This gives us the desired result.
\end{proof}

\begin{lemma}
\label{fl-lem2}
For $\eta < \frac{1}{L}$ and $E < \frac{1}{4}\text{min}\Big(\frac{1}{\eta L}, \frac{1}{\eta^2 L^2} - \frac{1}{\eta L}\Big)$, we have:
\small
\begin{equation*}
    \sum_{\tau=0}^{E-1} \mathbb{E}[\|\overline{\bm{v}}_{k,\tau} - \nabla f(\overline{\bm{w}}_{k,\tau})\|^2] 
    \leq
    {\frac{16 \eta^2 L^2 E^2 (\alpha E + 4)}{n^2}}\sum_{i \in [n]}\|\nabla f_i(\bm{w}_k)\|^2,
\end{equation*}
\normalsize
where the expectation is with respect to the randomness due to $\{\mathcal{B}_{k,1}^{(i)},\ldots,\mathcal{B}_{k,E-1}^{(i)}\}_{i=1}^{n}$.
\end{lemma}
\begin{proof}
Let $\overline{\bm{e}}_{k,\tau} = \overline{\bm{v}}_{k,\tau} - \nabla f(\overline{\bm{w}}_{k,\tau})$. Then:
\begin{flalign}
    \nonumber
    \|\overline{\bm{e}}_{k,\tau}\|^2 & =  \|\overline{\bm{v}}_{k,\tau} - \nabla f(\overline{\bm{w}}_{k,\tau})\|^2
    \\
    \nonumber
    & = \Big\|\frac{1}{n}\sum_{i \in [n]}(\bm{v}_{k,\tau}^{(i)} - \nabla f_i(\overline{\bm{w}}_{k,\tau}))\Big\|^2
    \\
    \nonumber
    & = 
    \Big\|\frac{1}{n}\sum_{i \in [n]}(\bm{e}_{k,\tau}^{(i)} + \widetilde{\bm{e}}_{k,\tau}^{(i)})\Big\|^2
    \\
    \label{eq:fl-4}
    & \leq \frac{2}{n^2}\Big\|\sum_{i \in [n]}\bm{e}_{k,\tau}^{(i)}\Big\|^2 + \frac{2}{n^2}\Big\|\sum_{i \in [n]}\widetilde{\bm{e}}_{k,\tau}^{(i)}\Big\|^2
\end{flalign}
So:
\begin{equation}
    \label{eq:fl-5}
    \mathbb{E}[\|\overline{\bm{e}}_{k,\tau}\|^2] \leq
    \frac{2}{n^2}\mathbb{E}\Big[\Big\|\sum_{i \in [n]}\bm{e}_{k,\tau}^{(i)}\Big\|^2\Big] + 
    \frac{2}{n^2}\mathbb{E}\Big[\Big\|\sum_{i \in [n]}\widetilde{\bm{e}}_{k,\tau}^{(i)}\Big\|^2\Big]
\end{equation}
But:
\[\mathbb{E}\Big[\Big\|\sum_{i \in [n]}\bm{e}_{k,\tau}^{(i)}\Big\|^2\Big] = \sum_{i \in [n]}\mathbb{E}\Big[\Big\|\bm{e}_{k,\tau}^{(i)}\Big\|^2\Big] + \sum_{i \ne j:i,j \in [n]}\langle \mathbb{E}[\bm{e}_{k,\tau}^{(i)}], \mathbb{E}[\bm{e}_{k,\tau}^{(j)}] \rangle\]
In the cross-term above, we can take expectations individually as $\{\mathcal{B}_{k,1}^{(i)},\ldots,\mathcal{B}_{k,E-1}^{(i)}\}$ and $\{\mathcal{B}_{k,1}^{(j)},\ldots,\mathcal{B}_{k,E-1}^{(j)}\}$ are independent for $i \ne j$. Next, from \Cref{fl-lem0}, $\mathbb{E}[\bm{e}_{k,\tau}^{(i)}] = \vec{0}$ $\forall$ $i,k,\tau$. Hence:
\[\mathbb{E}\Big[\Big\|\sum_{i \in [n]}\bm{e}_{k,\tau}^{(i)}\Big\|^2\Big] = \sum_{i \in [n]}\mathbb{E}\Big[\Big\|\bm{e}_{k,\tau}^{(i)}\Big\|^2\Big].\]
Using the above result and {\Cref{as-het}}
in (\ref{eq:fl-5}), we get that:
\begin{equation}
    \label{eq:fl-6}
    \mathbb{E}[\|\overline{\bm{e}}_{k,\tau}\|^2] \leq
    \frac{2}{n^2}\sum_{i \in [n]}\mathbb{E}[\|\bm{e}_{k,\tau}^{(i)}\|^2] + 
    \frac{2 \alpha}{n^2}\sum_{i \in [n]}\mathbb{E}[\|\widetilde{\bm{e}}_{k,\tau}^{(i)}\|^2].
\end{equation}
Now:
\begin{flalign}
    \nonumber
    \mathbb{E}\Big[\Big\|\widetilde{\bm{e}}_{k,\tau}^{(i)}\Big\|^2\Big] & = \mathbb{E}[\|\nabla f_i(\bm{w}_{k, \tau}^{(i)}) - \nabla f_i(\overline{\bm{w}}_{k,\tau})\|^2]
    \\
    \nonumber
    & = L^2 \mathbb{E}[\|\bm{w}_{k, \tau}^{(i)} - \overline{\bm{w}}_{k,\tau}\|^2]
    \\
    \nonumber
    & \leq L^2 \mathbb{E}[\|(\bm{w}_{k, 0}^{(i)} - \eta \sum_{t=0}^{\tau-1} \bm{v}_{k,t}^{(i)}) - (\overline{\bm{w}}_{k,0} - \eta \sum_{t=0}^{\tau-1} \overline{\bm{v}}_{k,t})\|^2]
\end{flalign}
But since $\bm{w}_{k,0}^{(i)} = \bm{w}_{k}$ $\forall$ $i$, we have $\overline{\bm{w}}_{k,0} = \bm{w}_{k}$. Hence:
\begin{flalign}
    \nonumber
    \mathbb{E}\Big[\Big\|\widetilde{\bm{e}}_{k,\tau}^{(i)}\Big\|^2\Big] & = \eta^2 L^2 \mathbb{E}[\| \sum_{t=0}^{\tau-1} \overline{\bm{v}}_{k,t} - \sum_{t=0}^{\tau-1} \bm{v}_{k,t}^{(i)}\|^2] 
    \\
    \nonumber
    & \leq \eta^2 L^2 \tau \sum_{t=0}^{\tau-1} \mathbb{E}[\|\overline{\bm{v}}_{k,t} - \bm{v}_{k,t}^{(i)}\|^2]
    \\
    \nonumber
    & = \eta^2 L^2 \tau \sum_{t=0}^{\tau-1} \mathbb{E}[\|\overline{\bm{v}}_{k,t}\|^2 + \|\bm{v}_{k,t}^{(i)}\|^2 - 2\langle \overline{\bm{v}}_{k,t}, \bm{v}_{k,t}^{(i)} \rangle]
\end{flalign}
Substituting the above in (\ref{eq:fl-6}), we get:
\small
\begin{flalign}
    \nonumber
    \mathbb{E}[\|\overline{\bm{e}}_{k,\tau}\|^2] & \leq
    \frac{2}{n^2}\sum_{i \in [n]}\mathbb{E}[\|\bm{e}_{k,\tau}^{(i)}\|^2] + 
    \frac{2 \alpha}{n^2}\sum_{i \in [n]}\eta^2 L^2 \tau \sum_{t=0}^{\tau-1} \mathbb{E}[\|\overline{\bm{v}}_{k,t}\|^2 + \|\bm{v}_{k,t}^{(i)}\|^2 - 2\langle \overline{\bm{v}}_{k,t}, \bm{v}_{k,t}^{(i)} \rangle]
    \\
    \label{eq:fl-7}
    & = \frac{2}{n^2}\sum_{i \in [n]}\mathbb{E}[\|\bm{e}_{k,\tau}^{(i)}\|^2] + \frac{2 \alpha \eta^2 L^2 \tau}{n^2}\sum_{t=0}^{\tau-1}\{n\mathbb{E}[\|\overline{\bm{v}}_{k,t}\|^2] + \sum_{i \in [n]}\mathbb{E}[\|\bm{v}_{k,t}^{(i)}\|^2] - 2 \langle \overline{\bm{v}}_{k,t}, \sum_{i \in [n]} \bm{v}_{k,t}^{(i)} \rangle \}
    \\
    \label{eq:fl-8}
    & = \frac{2}{n^2}\sum_{i \in [n]}\mathbb{E}[\|\bm{e}_{k,\tau}^{(i)}\|^2] + \frac{2 \alpha \eta^2 L^2 \tau}{n^2} \sum_{t=0}^{\tau-1}\sum_{i \in [n]}\Big(\mathbb{E}[\|\bm{v}_{k,t}^{(i)}\|^2] {- \mathbb{E}[\|\overline{\bm{v}}_{k,t}\|^2]}\Big).
    \\
    \label{eq:oct13-lem2-1}
    & \leq \frac{2}{n^2}\sum_{i \in [n]}\mathbb{E}[\|\bm{e}_{k,\tau}^{(i)}\|^2] + \frac{2 \alpha \eta^2 L^2 \tau}{n^2} \sum_{t=0}^{\tau-1}\sum_{i \in [n]}\mathbb{E}[\|\bm{v}_{k,t}^{(i)}\|^2].
\end{flalign}
\normalsize
To get (\ref{eq:fl-8}) from (\ref{eq:fl-7}), we use the fact $\sum_{i \in [n]} \bm{v}_{k,t}^{(i)} = n\overline{\bm{v}}_{k,t}$. Now summing up (\ref{eq:oct13-lem2-1}) from $\tau=0$ through to $E-1$, we get:
\begin{flalign}
    \label{eq:new-ref-2}
    \sum_{\tau=0}^{E-1}\mathbb{E}[\|\overline{\bm{e}}_{k,\tau}\|^2] & \leq \frac{2}{n^2}\sum_{i \in [n]}\underbrace{\sum_{\tau=0}^{E-1}\mathbb{E}[\|\bm{e}_{k,\tau}^{(i)}\|^2]}_\text{from \Cref{fl-lem-new2}} + \frac{2 \alpha \eta^2 L^2 E^2}{2n^2}\sum_{i \in [n]}\underbrace{\sum_{\tau=0}^{E-1}\mathbb{E}[\|\bm{v}_{k,\tau}^{(i)}\|^2]}_\text{from \Cref{fl-lem-new1}}.
\end{flalign}
Now using \Cref{fl-lem-new2} and \Cref{fl-lem-new1} above with $\eta < \frac{1}{L}$ and $E < \frac{1}{4}\text{min}\Big(\frac{1}{\eta L}, \frac{1}{\eta^2 L^2} - \frac{1}{\eta L}\Big)$, we get:
\begin{flalign*}
    \sum_{\tau=0}^{E-1}\mathbb{E}[\|\overline{\bm{e}}_{k,\tau}\|^2] & \leq \frac{2}{n^2}\sum_{i \in [n]} 32 E^2 \eta^2 L^2 \|\nabla f_i(\bm{w}_k)\|^2
    \\
    & + \frac{\alpha \eta^2 L^2 E^2}{n^2}\sum_{i \in [n]}16 E\|\nabla f_i(\bm{w}_{k})\|^2.
\end{flalign*}
This gives us the desired result.
\end{proof}

\begin{lemma}
\label{fl-lem0}
$\mathbb{E}_{\mathcal{B}_{k,1}^{(i)},\ldots,\mathcal{B}_{k,\tau}^{(i)}}[\bm{e}_{k,\tau}^{(i)}] = \vec{0}$ $\forall$ 
$k \in \{0,\ldots,K-1\}, \tau \in \{1,\ldots,E-1\}$.
\end{lemma}
\begin{proof}
Note that: 
\[\bm{e}_{k,0}^{(i)} = \bm{v}_{k,0}^{(i)} - \nabla f_i(\bm{w}_{k,0}^{(i)}) = \vec{0}.\]
For $\tau > 0$:
\small
\begin{flalign*}
    &\mathbb{E}_{\mathcal{B}_{k,1}^{(i)},\ldots,\mathcal{B}_{k,\tau}^{(i)}}[\bm{e}_{k,\tau}^{(i)}]  = \mathbb{E}_{\mathcal{B}_{k,1}^{(i)},\ldots,\mathcal{B}_{k,\tau}^{(i)}}[\bm{v}_{k,\tau}^{(i)} - \nabla f_i(\bm{w}_{k,\tau}^{(i)})]
    \\
    & = \mathbb{E}_{\mathcal{B}_{k,1}^{(i)},\ldots,\mathcal{B}_{k,\tau}^{(i)}}[\widetilde{\nabla} f_i(\bm{w}_{k,\tau}^{(i)};\mathcal{B}_{k,\tau}^{(i)}) + (\bm{v}_{k,\tau-1}^{(i)} - \widetilde{\nabla} f_i(\bm{w}_{k,\tau-1}^{(i)};\mathcal{B}_{k,\tau}^{(i)})) - \nabla f_i(\bm{w}_{k,\tau}^{(i)})]
    \\
    & = \mathbb{E}_{\mathcal{B}_{k,1}^{(i)},\ldots,\mathcal{B}_{k,\tau-1}^{(i)}}[\mathbb{E}_{\mathcal{B}_{k,\tau}^{(i)}}[\widetilde{\nabla} f_i(\bm{w}_{k,\tau}^{(i)};\mathcal{B}_{k,\tau}^{(i)}) + (\bm{v}_{k,\tau-1}^{(i)} - \widetilde{\nabla} f_i(\bm{w}_{k,\tau-1}^{(i)};\mathcal{B}_{k,\tau}^{(i)})) - \nabla f_i(\bm{w}_{k,\tau}^{(i)})|\mathcal{B}_{k,1}^{(i)},\ldots,\mathcal{B}_{k,\tau-1}^{(i)}]]
    \\
    & = \mathbb{E}_{\mathcal{B}_{k,1}^{(i)},\ldots,\mathcal{B}_{k,\tau-1}^{(i)}}[(\bm{v}_{k,\tau-1}^{(i)} - \nabla f_i(\bm{w}_{k,\tau-1}^{(i)}))]
    \\
    & = \mathbb{E}_{\mathcal{B}_{k,1}^{(i)},\ldots,\mathcal{B}_{k,\tau-1}^{(i)}}[\bm{e}_{k,\tau-1}^{(i)}].
\end{flalign*}
\normalsize
Doing this recursively, we get:
\begin{equation}
    \label{eq:fl-01}
    \mathbb{E}_{\mathcal{B}_{k,1}^{(i)},\ldots,\mathcal{B}_{k,\tau}^{(i)}}[\bm{e}_{k,\tau}^{(i)}] = 
    \bm{e}_{k,0}^{(i)} = \vec{0}.
\end{equation}
Note that this result also holds if we use full gradients at $\tau=0$.
\end{proof}

\begin{lemma}
\label{fl-lem-new1}
For $\eta < \frac{1}{L}$ and $E < \frac{1}{4}\text{min}\Big(\frac{1}{\eta L}, \frac{1}{\eta^2 L^2} - \frac{1}{\eta L}\Big)$, we have:
\[\sum_{\tau=0}^{E-1} \mathbb{E}[\|\bm{v}_{k,\tau}^{(i)}\|^2] \leq 16 E\|\nabla f_i(\bm{w}_{k})\|^2.\]
Note that in this lemma, the expectation is with respect to the randomness only due to $\{\mathcal{B}_{k,1}^{(i)},\ldots,\mathcal{B}_{k,E-1}^{(i)}\}_{i=1}^{n}$.
\end{lemma}
\begin{proof}
First, recall that ${\bm{e}}_{k,\tau}^{(i)} = \bm{v}_{k,\tau}^{(i)} - \nabla f_i(\bm{w}_{k,\tau}^{(i)})$. {Note that $\bm{e}_{k,0}^{(i)} = \vec{0}$, as we are using clients' full gradients at $\tau=0$.}
We have:
\begin{equation}
    \label{may11-2-1}
    \mathbb{E}[\|\bm{v}_{k,\tau}^{(i)}\|^2] \leq 2\mathbb{E}[\|\bm{e}_{k,\tau}^{(i)}\|^2] + 2\mathbb{E}[\|\nabla f_i(\bm{w}_{k,\tau}^{(i)})\|^2].
\end{equation}

Using Lemma 2.1 of \cite{liu2020optimal} with $\beta = 0$, we have:
\begin{flalign}
    \nonumber
    \mathbb{E}[\|{\bm{e}}_{k,\tau}^{(i)}\|^2] & \leq
    \mathbb{E}[\|{\bm{e}}_{k,0}^{(i)}\|^2] + 2L^2\sum_{t=0}^{\tau-1}\mathbb{E}[\|\bm{w}_{k,t+1}^{(i)} - \bm{w}_{k,t}^{(i)}\|^2]
    \\
    \label{eq:fl-13}
    & \leq 2L^2\sum_{t=0}^{\tau-1}\mathbb{E}[\|\bm{w}_{k,t+1}^{(i)} - \bm{w}_{k,t}^{(i)}\|^2].
\end{flalign}
{The last step follows because $\bm{e}_{k,0}^{(i)} = \vec{0}$.}
\\
Summing the above from $\tau = 0$ through to $E-1$, we get:
\small
\begin{flalign}
    \nonumber
    \sum_{\tau=0}^{E-1}\mathbb{E}[\|{\bm{e}}_{k,\tau}^{(i)}\|^2] & \leq 
    2L^2\sum_{\tau=0}^{E-1}\sum_{t=0}^{\tau-1}\mathbb{E}[\|\bm{w}_{k,t+1}^{(i)} - \bm{w}_{k,t}^{(i)}\|^2]
    \\
    \label{eq:fl-14}
    & \leq {2 E L^2}\sum_{\tau=0}^{E-2}\mathbb{E}[\|\bm{w}_{k,\tau+1}^{(i)} - \bm{w}_{k,\tau}^{(i)}\|^2].
\end{flalign}
\normalsize
Next, 
re-arranging equation (11) in Lemma 2.2 of \cite{liu2020optimal} (observe that in our case, $G_{\eta}(.)$ is simply the gradient), we get:
\begin{equation}
    \label{eq:fl-15}
    \mathbb{E}[\|\nabla f_i(\bm{w}_{k,\tau}^{(i)})\|^2] \leq \frac{2}{\eta}\mathbb{E}[f_i(\bm{w}_{k,\tau}^{(i)}) - f_i(\bm{w}_{k,\tau+1}^{(i)})] -\frac{1}{\eta^2}(1 - \eta{L})\mathbb{E}[\|\bm{w}_{k,\tau+1}^{(i)} - \bm{w}_{k,\tau}^{(i)}\|^2] + \mathbb{E}[\|\bm{e}_{k,\tau}^{(i)}\|^2]
\end{equation}
Summing (\ref{eq:fl-15}) from $\tau=0$ to $E-1$ and using (\ref{eq:fl-14}), we get:
\small
\begin{multline}
    \label{eq:fl-16}
    \sum_{\tau=0}^{E-1}\mathbb{E}[\|\nabla f_i(\bm{w}_{k,\tau}^{(i)})\|^2] \leq \frac{2}{\eta}(f_i(\bm{w}_{k}) - \mathbb{E}[f_i(\bm{w}_{k,E}^{(i)})]) 
    -\frac{(1-\eta L)}{\eta^2}\sum_{\tau=0}^{E-1}\mathbb{E}[\|\bm{w}_{k,\tau+1}^{(i)} - \bm{w}_{k,\tau}^{(i)}\|^2] 
    \\
    + {2 E L^2}\sum_{\tau=0}^{E-2}\mathbb{E}[\|\bm{w}_{k,\tau+1}^{(i)} - \bm{w}_{k,\tau}^{(i)}\|^2].
\end{multline}
\normalsize
Next, summing (\ref{eq:fl-14}) and (\ref{eq:fl-16}) gives us:
\begin{multline}
    \label{eq:fl-17}
    \sum_{\tau=0}^{E-1}\{\mathbb{E}[\|{\bm{e}}_{k,\tau}^{(i)}\|^2]+\mathbb{E}[\|\nabla f_i(\bm{w}_{k,\tau}^{(i)})\|^2]\} \leq \frac{2}{\eta}(f_i(\bm{w}_{k}) - \mathbb{E}[f_i(\bm{w}_{k,E}^{(i)})])
    \\
    - \underbrace{\Big(\frac{1-\eta L}{\eta^2}\Big)}_\text{$> 0$ for $\eta < \frac{1}{L}$}\mathbb{E}[\|\bm{w}_{k,E}^{(i)} - \bm{w}_{k,E-1}^{(i)}\|^2] - \underbrace{\Big(\frac{(1-\eta L)}{\eta^2} - {4 E L^2}\Big)}_\text{$> 0$ for $E < \frac{(1-\eta L)}{4\eta^2 L^2}$}\sum_{\tau=0}^{E-2}\mathbb{E}[\|\bm{w}_{k,\tau+1}^{(i)} - \bm{w}_{k,\tau}^{(i)}\|^2].
\end{multline}
So if we have $\eta < \frac{1}{L}$ and 
$E < \frac{1}{4}(\frac{1}{\eta^2 L^2} - \frac{1}{\eta L})$
, we get:
\small
\begin{equation}
    \label{eq:fl-18}
    \sum_{\tau=0}^{E-1}\{\mathbb{E}[\|{\bm{e}}_{k,\tau}^{(i)}\|^2]+\mathbb{E}[\|\nabla f_i(\bm{w}_{k,\tau}^{(i)})\|^2]\} \leq \frac{2}{\eta}(f_i(\bm{w}_{k}) - \mathbb{E}[f_i(\bm{w}_{k,E}^{(i)})]).
\end{equation}
\normalsize
Now from \Cref{fl-lem3}, for $E < \frac{1}{4}\text{min}\Big(\frac{1}{\eta L}, \frac{1}{\eta^2 L^2} - \frac{1}{\eta L}\Big)$, we have that:
\begin{equation}
\label{eq:fl-add-19}
f_i(\bm{w}_{k}) - \mathbb{E}[f_i(\bm{w}_{k,E}^{(i)})] \leq 4 \eta E\|\nabla f_i(\bm{w}_{k})\|^2.
\end{equation}
Putting (\ref{eq:fl-add-19}) in (\ref{eq:fl-18}) and then using it (\ref{may11-2-1}) gives us the desired result.
\end{proof}

\begin{lemma}
\label{fl-lem3}
For $\eta < \frac{1}{L}$ and $E < \frac{1}{4}\text{min}\Big(\frac{1}{\eta L}, \frac{1}{\eta^2 L^2} - \frac{1}{\eta L}\Big)$, we have:
\[f_i(\bm{w}_{k}) - \mathbb{E}[f_i(\bm{w}_{k,E}^{(i)})] \leq 4\eta E\|\nabla f_i(\bm{w}_{k})\|^2.\]
The expectation above is with respect to the randomness only due to $\{\mathcal{B}_{k,1}^{(i)},\ldots,\mathcal{B}_{k,E-1}^{(i)}\}_{i=1}^{n}$.
\end{lemma}
\begin{proof}
By $L$-smoothness of each $f_i$, we have:
\[f_i(\bm{w}_{k,E}^{(i)}) \geq f_i(\bm{w}_{k}) + \langle \nabla f_i(\bm{w}_{k}), \bm{w}_{k,E}^{(i)} - \bm{w}_{k} \rangle - \frac{L}{2}\|\bm{w}_{k,E}^{(i)} - \bm{w}_{k}\|^2\]
\begin{flalign*}
    \implies f_i(\bm{w}_{k}) - f_i(\bm{w}_{k,E}^{(i)}) & \leq \langle \nabla f_i(\bm{w}_{k}), \bm{w}_{k} - \bm{w}_{k,E}^{(i)} \rangle + \frac{L}{2}\|\bm{w}_{k,E}^{(i)} - \bm{w}_{k}\|^2
    \\
    & \leq \underbrace{\frac{\alpha'}{2}\|\nabla f_i(\bm{w}_{k})\|^2 + \frac{1}{2\alpha'} \|\bm{w}_{k,E}^{(i)} - \bm{w}_{k}\|^2}_\text{follows by Young's inequality} + \frac{L}{2}\|\bm{w}_{k,E}^{(i)} - \bm{w}_{k}\|^2 \text{ for } \alpha' > 0.
\end{flalign*}
Recall that $\bm{w}_{k,E}^{(i)} - \bm{w}_{k} = \eta \sum_{\tau=0}^{E-1}\bm{v}_{k,\tau}^{(i)}$. Hence taking expectation above with $\alpha' = 2\eta E$, we get that:
\begin{flalign}
    f_i(\bm{w}_{k}) - \mathbb{E}[f_i(\bm{w}_{k,E}^{(i)})] & \leq 
    \eta E\|\nabla f_i(\bm{w}_{k})\|^2 +\eta^2 E \Big(\frac{1}{4\eta E} + \frac{L}{2}\Big)\sum_{\tau=0}^{E-1}\mathbb{E}[\|\bm{v}_{k,\tau}^{(i)}\|^2]
    \\
    \label{eq:fl-1-1}
    & \leq \eta E\|\nabla f_i(\bm{w}_{k})\|^2 + \frac{3\eta}{8}\sum_{\tau=0}^{E-1}\mathbb{E}[\|\bm{v}_{k,\tau}^{(i)}\|^2].
\end{flalign}
(\ref{eq:fl-1-1}) follows from the fact that $\eta L E < \frac{1}{4}$. Next, from the proof of \Cref{fl-lem-new1}, for $E < \frac{(1-\eta L)}{4\eta^2 L^2}$: 
\[\sum_{\tau=0}^{E-1}\mathbb{E}[\|\bm{v}_{k,\tau}^{(i)}\|^2] 
\leq \frac{2}{\eta}(f_i(\bm{w}_{k}) - \mathbb{E}[f_i(\bm{w}_{k,E}^{(i)})]).\]
Putting this in (\ref{eq:fl-1-1}), we get:
\[f_i(\bm{w}_{k}) - \mathbb{E}[f_i(\bm{w}_{k,E}^{(i)})] \leq \eta E\|\nabla f_i(\bm{w}_{k})\|^2 +\frac{3}{4}(f_i(\bm{w}_{k}) - \mathbb{E}[f_i(\bm{w}_{k,E}^{(i)})].\]
\begin{equation}
    \label{eq:fl-1-2}
     \implies f_i(\bm{w}_{k}) - \mathbb{E}[f_i(\bm{w}_{k,E}^{(i)})] \leq 4\eta E\|\nabla f_i(\bm{w}_{k})\|^2.
\end{equation}
\end{proof}

\begin{lemma}
\label{fl-lem-new2}
For $\eta < \frac{1}{L}$ 
and $E < \frac{1}{4}\text{min}\Big( \frac{1}{\eta L}, \frac{1}{\eta^2 L^2} - \frac{1}{\eta L}\Big)$, we have:
\[\sum_{\tau=0}^{E-1}\mathbb{E}[\|{\bm{e}}_{k,\tau}^{(i)}\|^2] \leq 32 E^2 \eta^2 L^2 \|\nabla f_i(\bm{w}_k)\|^2.\]
The expectation above is with respect to the randomness only due to $\{\mathcal{B}_{k,1}^{(i)},\ldots,\mathcal{B}_{k,E-1}^{(i)}\}_{i=1}^{n}$.
\end{lemma}
\begin{proof}
Note that in \Cref{fl-lem-new1}, we have already bounded $\sum_{\tau=0}^{E-1}\mathbb{E}[\|{\bm{e}}_{k,\tau}^{(i)}\|^2]$ (see (\ref{eq:fl-14})) -- but here we expand it more for use in \Cref{fl-lem2}.
\\
First, from (\ref{eq:fl-14}), we have:
\begin{flalign*}
    \sum_{\tau=0}^{E-1}\mathbb{E}[\|{\bm{e}}_{k,\tau}^{(i)}\|^2] \leq {2 E L^2}\sum_{\tau=0}^{E-2}\mathbb{E}[\|\bm{w}_{k,\tau+1}^{(i)} - \bm{w}_{k,\tau}^{(i)}\|^2].
\end{flalign*}
Next, using the fact that $\bm{w}_{k,\tau+1}^{(i)} = \bm{w}_{k,\tau}^{(i)} - \eta\bm{v}_{k,\tau}^{(i)}$, we get:
\begin{flalign}
    \label{eq:fl-lem-new2-1}
    \nonumber
    \sum_{\tau=0}^{E-1}\mathbb{E}[\|{\bm{e}}_{k,\tau}^{(i)}\|^2] & \leq {2 E \eta^2 L^2}\sum_{\tau=0}^{E-2}\mathbb{E}[\|\bm{v}_{k,\tau}^{(i)}\|^2]
    \leq {2 E \eta^2 L^2}\underbrace{\sum_{\tau=0}^{E-1}\mathbb{E}[\|\bm{v}_{k,\tau}^{(i)}\|^2]}_\text{from \Cref{fl-lem-new1}}
    \leq {2 E \eta^2 L^2}(16 E\|\nabla f_i(\bm{w}_{k})\|^2).
\end{flalign}
This gives us the desired result.
\end{proof}

\begin{lemma}
\label{nov-1-lem2}
Suppose $2 \eta L E^2 \leq 1$. Then:
\begin{multline*}
    \mathbb{E}[\|\bm{u}_k - \overline{\bm{\delta}}_k\|^2] \leq (1-\beta)^2\mathbb{E}[\|\bm{u}_{k-1} - \overline{\bm{\delta}}_{k-1}\|^2] + 2 \beta^2 \mathbb{E}[\|{g}_Q(\bm{w}_k;\mathcal{S}_k) - \overline{\bm{\delta}}_k\|^2] 
    \\
    + 8 e^2 (1+q) (1-\beta)^2 \eta^2 L^2 E^2 (E+1)^2 \mathbb{E}[\|\bm{w}_{k} - \bm{w}_{k-1}\|^2]. 
\end{multline*}
\end{lemma}
\begin{proof}
First, note that for each $i \in [n]$, $\mathbb{E}_{\mathcal{B}_{k,1}^{(i)},\ldots,\mathcal{B}_{k,E-1}^{(i)}}[\bm{w}_{k} - \widehat{\bm{w}}_{k,E}^{(i)}] = \bm{\delta}_{k}^{(i)}$. So:
\begin{equation}
    \label{eq:nov-1-lem1-2}
    \mathbb{E}_{\mathcal{S}_k,\{\mathcal{B}_{k,1}^{(i)},\ldots,\mathcal{B}_{k,E-1}^{(i)}\}_{i=1}^{n}}[{g}(\bm{w}_k;\mathcal{S}_k)] = \overline{\bm{\delta}}_k.
\end{equation}
Similarly, for each $i \in [n]$, $\mathbb{E}_{\mathcal{B}_{k,1}^{(i)},\ldots,\mathcal{B}_{k,E-1}^{(i)}}[\bm{w}_{k-1} - \widehat{\bm{w}}_{k-1,E}^{(i)}] = 
\bm{\delta}_{k-1}^{(i)}$. Hence:
\begin{equation}
    \label{eq:nov-1-lem1-3}
    \mathbb{E}_{\mathcal{S}_k,\{\mathcal{B}_{k,1}^{(i)},\ldots,\mathcal{B}_{k,E-1}^{(i)}\}_{i=1}^{n}
    }[\widehat{g}(\bm{w}_{k-1};\mathcal{S}_k)] = \overline{\bm{\delta}}_{k-1}.
\end{equation}
We have:
\small
\begin{flalign}
    \nonumber
    \mathbb{E}&[\|\bm{u}_k - \overline{\bm{\delta}}_k\|^2]  = 
    \mathbb{E}[\|\beta {g}_Q(\bm{w}_k;\mathcal{S}_k) + (1-\beta)\bm{u}_{k-1} + (1-\beta)\Delta{g}_Q(\bm{w}_k,\bm{w}_{k-1};\mathcal{S}_k) - \overline{\bm{\delta}}_k\|^2]
    \\
    \nonumber
    & = 
    \mathbb{E}[\|(1 - \beta)(\bm{u}_{k-1} - \overline{\bm{\delta}}_{k-1}) + \beta{g}_Q(\bm{w}_k;\mathcal{S}_k) - \overline{\bm{\delta}}_k + (1 - \beta)(\overline{\bm{\delta}}_{k-1} + \Delta{g}_Q(\bm{w}_k,\bm{w}_{k-1};\mathcal{S}_k))\|^2]
    \\
    \label{eq:nov-1-lem2-1}
    & = (1-\beta)^2\mathbb{E}[\|\bm{u}_{k-1} - \overline{\bm{\delta}}_{k-1}\|^2] + 
    \mathbb{E}[\|\beta{g}_Q(\bm{w}_k;\mathcal{S}_k) - \overline{\bm{\delta}}_k + (1 - \beta)(\overline{\bm{\delta}}_{k-1} + \Delta{g}_Q(\bm{w}_k,\bm{w}_{k-1};\mathcal{S}_k))\|^2]
\end{flalign}
\normalsize
The cross-term in (\ref{eq:nov-1-lem2-1}) vanishes by taking expectation with respect to $Q_D$ and $\mathcal{S}_k$. Next:
\small
\begin{flalign}
    \nonumber
    & 
    \mathbb{E}[\|\beta{g}_Q(\bm{w}_k;\mathcal{S}_k) - \overline{\bm{\delta}}_k + (1 - \beta)(\overline{\bm{\delta}}_{k-1} + \Delta{g}_Q(\bm{w}_k,\bm{w}_{k-1};\mathcal{S}_k))\|^2]
    \\
    \nonumber
    & = 
    \mathbb{E}[\|\beta({g}_Q(\bm{w}_k;\mathcal{S}_k) - \overline{\bm{\delta}}_k) + (1 - \beta)(\overline{\bm{\delta}}_{k-1} + \Delta{g}_Q(\bm{w}_k,\bm{w}_{k-1};\mathcal{S}_k)) - \overline{\bm{\delta}}_k)\|^2]
    \\
    \label{eq:nov-1-lem2-2}
    & \leq 
    2 \beta^2 \mathbb{E}[\|{g}_Q(\bm{w}_k;\mathcal{S}_k) - \overline{\bm{\delta}}_k\|^2] + 2 (1-\beta)^2\mathbb{E}[\|\overline{\bm{\delta}}_{k-1} + \Delta{g}_Q(\bm{w}_k,\bm{w}_{k-1};\mathcal{S}_k) - \overline{\bm{\delta}}_k\|^2]
\end{flalign}
\normalsize
Next, note that:
\small
\begin{flalign}
    \nonumber
    & \mathbb{E}[\|\overline{\bm{\delta}}_{k-1} + \Delta{g}_Q(\bm{w}_k,\bm{w}_{k-1};\mathcal{S}_k) - \overline{\bm{\delta}}_k\|^2]
    \\
    \nonumber
    & = \mathbb{E}[\|\Delta{g}_Q(\bm{w}_k,\bm{w}_{k-1};\mathcal{S}_k)\|^2] + \mathbb{E}[\|\overline{\bm{\delta}}_k - \overline{\bm{\delta}}_{k-1}\|^2] - 2 \mathbb{E}[\langle \Delta{g}_Q(\bm{w}_k,\bm{w}_{k-1};\mathcal{S}_k), \overline{\bm{\delta}}_k - \overline{\bm{\delta}}_{k-1} \rangle]
    \\
    \label{eq:nov-1-lem2-3}
    & = 
    \mathbb{E}[\|\Delta{g}_Q(\bm{w}_k,\bm{w}_{k-1};\mathcal{S}_k)\|^2] + \mathbb{E}[\|\overline{\bm{\delta}}_k - \overline{\bm{\delta}}_{k-1}\|^2] - 2\mathbb{E}[\|\overline{\bm{\delta}}_k - \overline{\bm{\delta}}_{k-1}\|^2]
    \\
    \label{eq:nov-1-lem2-4}
    & \leq \mathbb{E}[\|\Delta{g}_Q(\bm{w}_k,\bm{w}_{k-1};\mathcal{S}_k)\|^2].
\end{flalign}
\normalsize
(\ref{eq:nov-1-lem2-3}) follows by first taking expectation with respect to $Q_D$ and then using (\ref{eq:nov-1-lem1-2}) and (\ref{eq:nov-1-lem1-3}). 
\\
Further:
\begin{flalign}
    \nonumber
    \mathbb{E}[\|\Delta{g}_Q(\bm{w}_k,\bm{w}_{k-1};\mathcal{S}_k)\|^2] & = \mathbb{E}\Big[\Big\|\frac{1}{r}\sum_{i \in \mathcal{S}_k}Q_D(({\bm{w}_{k} - \bm{w}_{k,E}^{(i)}}) - ({\bm{w}_{k-1} - \widehat{\bm{w}}_{k-1,E}^{(i)}}))\Big\|^2\Big]
    \\
    \nonumber
    & \leq \mathbb{E}_{\mathcal{S}_k}\Big[\frac{r}{r^2} \sum_{i \in \mathcal{S}_k} \mathbb{E}\Big[\| Q_D(({\bm{w}_{k} - \bm{w}_{k,E}^{(i)}}) - ({\bm{w}_{k-1} - \widehat{\bm{w}}_{k-1,E}^{(i)}}))\|^2\Big]\Big]
    \\
    \label{nov6-2}
    & \leq \mathbb{E}_{\mathcal{S}_k}\Big[\frac{1}{r} \sum_{i \in \mathcal{S}_k} (1+q) \mathbb{E}\Big[\|({\bm{w}_{k} - \bm{w}_{k,E}^{(i)}}) - ({\bm{w}_{k-1} - \widehat{\bm{w}}_{k-1,E}^{(i)}})\|^2\Big]\Big]
    \\
    \label{eq:may10-3}
    & = \frac{1}{n} \sum_{i \in [n]} (1+q) \mathbb{E}\Big[\|({\bm{w}_{k} - \bm{w}_{k,E}^{(i)}}) - ({\bm{w}_{k-1} - \widehat{\bm{w}}_{k-1,E}^{(i)}})\|^2\Big].
\end{flalign}
(\ref{nov6-2}) follows from \Cref{as5} on the variance of $Q_D$. Further, using \Cref{nov-1-lem3}, we get
\begin{equation}
    \label{nov6-3}
    \mathbb{E}[\|({\bm{w}_{k} - \bm{w}_{k,E}^{(i)}}) - ({\bm{w}_{k-1} - \widehat{\bm{w}}_{k-1,E}^{(i)}})\|^2] \leq
    4 e^2 \eta^2 L^2 E^2 (E+1)^2 \mathbb{E}[\|\bm{w}_{k} - \bm{w}_{k-1}\|^2],
\end{equation}
for $2 \eta L E^2 \leq 1$. 

Using this in (\ref{eq:may10-3}):
\begin{flalign}
    \label{nov6-4}
    \mathbb{E}[\|\Delta{g}_Q(\bm{w}_k,\bm{w}_{k-1};\mathcal{S}_k)\|^2] & \leq 
    4 e^2 (1+q) \eta^2 L^2 E^2 (E+1)^2 \mathbb{E}[\|\bm{w}_{k} - \bm{w}_{k-1}\|^2].
\end{flalign}
Now using (\ref{nov6-4}) in (\ref{eq:nov-1-lem2-4}) and then using it in (\ref{eq:nov-1-lem2-2}), we get:
\begin{multline}
    \label{eq:nov-1-lem2-5}
    \mathbb{E}[\|\beta{g}_Q(\bm{w}_k;\mathcal{S}_k) - \overline{\bm{\delta}}_k + (1 - \beta)(\overline{\bm{\delta}}_{k-1} + \Delta{g}_Q(\bm{w}_k,\bm{w}_{k-1};\mathcal{S}_k))\|^2]
    \\
    \leq 
    2 \beta^2 \mathbb{E}[\|{g}_Q(\bm{w}_k;\mathcal{S}_k) - \overline{\bm{\delta}}_k\|^2] +
    8 e^2 (1+q) (1-\beta)^2 \eta^2 L^2 E^2 (E+1)^2 \mathbb{E}[\|\bm{w}_{k} - \bm{w}_{k-1}\|^2].
\end{multline}
Finally, putting (\ref{eq:nov-1-lem2-5}) back in (\ref{eq:nov-1-lem2-1}) gives us the desired result. 
\end{proof}

\begin{lemma}
\label{nov-1-lem3}
Suppose $2 \eta L E^2 \leq 1$. Then $\forall$ $k \geq 0$ and $i \in [n]$, we have:
\begin{equation*}
    \mathbb{E}[\|({\bm{w}_{k} - \bm{w}_{k,E}^{(i)}}) - ({\bm{w}_{k-1} - \widehat{\bm{w}}_{k-1,E}^{(i)}})\|] \leq
    2 e (\eta L E (E+1)) \|\bm{w}_{k} - \bm{w}_{k-1}\|.
\end{equation*}
\end{lemma}
\begin{proof}
We have for any $i \in [n]$:
\begin{flalign}
    \nonumber
    \|({\bm{w}_{k} - \bm{w}_{k,E}^{(i)}}) - ({\bm{w}_{k-1} - \widehat{\bm{w}}_{k-1,E}^{(i)}})\| & = \Big\|\sum_{\tau=0}^{E-1}\eta \bm{v}_{k,\tau}^{(i)} - \sum_{\tau=0}^{E-1}\eta \widehat{\bm{v}}_{k-1,\tau}^{(i)}\Big\|
    \\
    \label{eq:nov-1-lem3-1}
    & \leq 
    \sum_{\tau=0}^{E-1} \eta \|\bm{v}_{k,\tau}^{(i)} - \widehat{\bm{v}}_{k-1,\tau}^{(i)}\|.
\end{flalign}
The last step follows by the triangle inequality.
\\
Next, we have:
\begin{multline*}
    \|\bm{v}_{k,\tau}^{(i)} - \widehat{\bm{v}}_{k-1,\tau}^{(i)}\| = \|\{\widetilde{\nabla} f_i(\bm{w}_{k,\tau}^{(i)};\mathcal{B}_{k,\tau}^{(i)}) + (\bm{v}_{k,\tau-1}^{(i)} - \widetilde{\nabla} f_i(\bm{w}_{k,\tau-1}^{(i)};\mathcal{B}_{k,\tau}^{(i)}))\} 
    \\
    - \{\widetilde{\nabla} f_i(\widehat{\bm{w}}_{k-1,\tau}^{(i)};\mathcal{B}_{k,\tau}^{(i)}) 
    + (\widehat{\bm{v}}_{k-1,\tau-1}^{(i)}
    - \widetilde{\nabla} f_i(\widehat{\bm{w}}_{k-1,\tau-1}^{(i)};\mathcal{B}_{k,\tau}^{(i)}))\}\|
\end{multline*}
Note that $\mathcal{B}_{k,\tau}^{(i)}$ can be the full batch too. 

Re-arranging the above, using the triangle inequality and the smoothness of the stochastic gradients, we get:
\begin{equation}
    \label{eq:nov-1-lem3-2}
    \|\bm{v}_{k,\tau}^{(i)} - \widehat{\bm{v}}_{k-1,\tau}^{(i)}\| \leq L\|\bm{w}_{k,\tau}^{(i)} - \widehat{\bm{w}}_{k-1,\tau}^{(i)}\| +
    \|\bm{v}_{k,\tau-1}^{(i)} - \widehat{\bm{v}}_{k-1,\tau-1}^{(i)}\| +
    L\|\bm{w}_{k,\tau-1}^{(i)} - \widehat{\bm{w}}_{k-1,\tau-1}^{(i)}\|.
\end{equation}
Unfolding the above recursion, we get:
\begin{equation}
    \label{eq:nov-1-lem3-3}
    \|\bm{v}_{k,\tau}^{(i)} - \widehat{\bm{v}}_{k-1,\tau}^{(i)}\| \leq 2L\sum_{t=0}^{\tau} \|\bm{w}_{k,t}^{(i)} - \widehat{\bm{w}}_{k-1,t}^{(i)}\|.
\end{equation}
Just as a sanity check for (\ref{eq:nov-1-lem3-3}), observe that $\|\bm{v}_{k,0}^{(i)} - \widehat{\bm{v}}_{k-1,0}^{(i)}\| = \|\nabla f_i(\bm{w}_k) - \nabla f_i(\bm{w}_{k-1})\| \leq L \|\bm{w}_k - \bm{w}_{k-1}\|$. Next:
\begin{flalign}
    \nonumber
    \|\bm{w}_{k,\tau+1}^{(i)} - \widehat{\bm{w}}_{k-1,\tau+1}^{(i)}\|
    & = \|\bm{w}_{k,\tau}^{(i)} - \widehat{\bm{w}}_{k-1,\tau}^{(i)} - \eta (\bm{v}_{k,\tau}^{(i)} - \widehat{\bm{v}}_{k-1,\tau}^{(i)})\|
    \\
    \nonumber
    & \leq \|\bm{w}_{k,\tau}^{(i)} - \widehat{\bm{w}}_{k-1,\tau}^{(i)}\| + \eta \|\bm{v}_{k,\tau}^{(i)} - \widehat{\bm{v}}_{k-1,\tau}^{(i)}\|
    \\
    \nonumber
    & \leq \|\bm{w}_{k,\tau}^{(i)} - \widehat{\bm{w}}_{k-1,\tau}^{(i)}\| + 2 \eta L \sum_{t=0}^{\tau} \|\bm{w}_{k,t}^{(i)} - \widehat{\bm{w}}_{k-1,t}^{(i)}\|.
\end{flalign}
The last step follows by using (\ref{eq:nov-1-lem3-3}). Thus:
\begin{equation}
    \label{eq:nov-1-lem3-4}
    \|\bm{w}_{k,\tau}^{(i)} - \widehat{\bm{w}}_{k-1,\tau}^{(i)}\| \leq \|\bm{w}_{k,\tau-1}^{(i)} - \widehat{\bm{w}}_{k-1,\tau-1}^{(i)}\| + 2 \eta L \sum_{t=0}^{\tau-1} \|\bm{w}_{k,t}^{(i)} - \widehat{\bm{w}}_{k-1,t}^{(i)}\|.
\end{equation}
Based on (\ref{eq:nov-1-lem3-4}), we claim that:
\begin{equation}
    \label{eq:nov-1-lem3-5}
    \|\bm{w}_{k,\tau}^{(i)} - \widehat{\bm{w}}_{k-1,\tau}^{(i)}\| \leq (1 + 2 \eta L E)^{\tau}\|\bm{w}_k - \bm{w}_{k-1}\|.
\end{equation}

We prove this by induction. Let us first examine the base case of $\tau = 1$. We have:
\begin{flalign*}
    \|\bm{w}_{k,1}^{(i)} - \widehat{\bm{w}}_{k-1,1}^{(i)}\| & = \|\bm{w}_{k} - {\bm{w}}_{k-1} - \eta (\bm{v}_{k,0}^{(i)} - \widehat{\bm{v}}_{k-1,0}^{(i)})\|
    \\
    & = \|\bm{w}_{k} - {\bm{w}}_{k-1} - \eta(\nabla f_i(\bm{w}_{k}) - \nabla f_i({\bm{w}}_{k-1}))\|
    \\
    & \leq \|\bm{w}_{k} - {\bm{w}}_{k-1}\| + \eta L \|\bm{w}_{k} - {\bm{w}}_{k-1}\|
    \\
    & \leq (1 + 2 \eta L E)^{1}\|\bm{w}_{k} - {\bm{w}}_{k-1}\|.
\end{flalign*}
For ease of notation, let us define ${d}_{k} \triangleq \|\bm{w}_{k} - \bm{w}_{k-1}\|$. Now suppose the claim is true for $\tau \leq t$. Then using (\ref{eq:nov-1-lem3-4}), we have for $\tau= t+1$:
\begin{flalign}
    \nonumber
    \|\bm{w}_{k,t+1}^{(i)} - \widehat{\bm{w}}_{k-1,t+1}^{(i)}\|
    & \leq  \Big\{(1 + 2 \eta L E)^{t} + 2\eta L \sum_{t_2=0}^{t}(1 + 2 \eta L E)^{t_2}\Big\}{d}_{k}
    \\
    \nonumber
    & \leq \Big\{(1 + 2 \eta L E)^{t} + 2 \eta L (t+1) (1 + 2 \eta L E)^{t} \Big\}{d}_{k}
    \\
    \label{eq:nov-1-lem3-6}
    & \leq (1 + 2 \eta L E)^{t} (1 + 2 \eta L (t+1)) {d}_{k} \leq (1 + 2 \eta L E)^{t+1} {d}_{k}.
\end{flalign}
This proves our claim.
\\
Now, using our claim, i.e., (\ref{eq:nov-1-lem3-5}) in (\ref{eq:nov-1-lem3-3}), we get:
\begin{flalign}
    \label{eq:nov-1-lem3-8}
    \|\bm{v}_{k,\tau}^{(i)} - \widehat{\bm{v}}_{k-1,\tau}^{(i)}\| & \leq 2L\sum_{t=0}^{\tau}(1 + 2 \eta L E)^{t}\|\bm{w}_{k} - \bm{w}_{k-1}\| \leq 2 L (\tau+1) (1 + 2 \eta L E)^{\tau}\|\bm{w}_{k} - \bm{w}_{k-1}\|.
\end{flalign}
Note that this bound is independent of $i$.
\\
Finally, using (\ref{eq:nov-1-lem3-8}) in (\ref{eq:nov-1-lem3-1}), we get:
\small
\begin{flalign}
    \nonumber
    \|({\bm{w}_{k} - \bm{w}_{k,E}^{(i)}}) - ({\bm{w}_{k-1} - \widehat{\bm{w}}_{k-1,E}^{(i)}})\|
    & \leq \sum_{\tau=0}^{E-1} \eta \|\bm{v}_{k,\tau}^{(i)} - \widehat{\bm{v}}_{k-1,\tau}^{(i)}\|
    \\
    \nonumber
    & \leq \sum_{\tau=0}^{E-1} 2 \eta L (\tau + 1) (1 + 2 \eta L E)^{\tau}\|\bm{w}_{k} - \bm{w}_{k-1}\|
    \\
    \nonumber
    & \leq 2 \eta L E (E+1) (1 + 2 \eta L E)^{E}\|\bm{w}_{k} - \bm{w}_{k-1}\|
    \\
    \label{eq:nov-1-lem3-9}
    & \leq 2 \eta L E (E+1) e^{2 \eta L E^2} \|\bm{w}_{k} - \bm{w}_{k-1}\|.
\end{flalign}
\normalsize
The last step follows from the fact that $1+z \leq e^z$ $\forall$ $z$.

Finally, setting $2 \eta L E^2 \leq 1$ gives us the desired result.
\end{proof}

\begin{lemma}
\label{lem-may11-n1}
(V*) in the proof of \Cref{nov-1-lem0} can be bounded as:
\begin{equation*}
    \text{(V*)} \leq 4 \eta^2 E\Big(\frac{q}{n^2} + \frac{(1+q)}{r(n-1)}\Big(1 - \frac{r}{n}\Big)\Big)\sum_{i \in [n]}\sum_{\tau=0}^{E-1}\mathbb{E}[\|\bm{v}_{k,\tau}^{(i)}\|^2].
\end{equation*}
\end{lemma}

\begin{proof}
We have (V*) = $\mathbb{E}[\|{g}_Q(\bm{w}_l;\mathcal{S}_l) - \overline{\bm{\delta}}_l\|^2]$. Note that:
\begin{multline}
    \label{eq:may11-n10}
    \mathbb{E}[\|{g}_Q(\bm{w}_l;\mathcal{S}_l) - \overline{\bm{\delta}}_l\|^2] \leq \eta^2 \underbrace{\mathbb{E}\Big[\Big\|\frac{1}{r}\sum_{i \in \mathcal{S}_l}\frac{Q_D({\bm{w}_{l} - \bm{w}_{l,E}^{(i)}})}{\eta} - \frac{1}{n}\sum_{i \in [n]}\frac{Q_D({\bm{w}_{l} - \bm{w}_{l,E}^{(i)}})}{\eta}\Big\|^2\Big]}_\text{(A)} \\
    +
    \eta^2\underbrace{\mathbb{E}\Big[\Big\|\frac{1}{n}\sum_{i \in [n]}\Big\{\frac{Q_D({\bm{w}_{l} - \bm{w}_{l,E}^{(i)}})}{\eta} - \frac{({\bm{w}_{l} - \bm{w}_{l,E}^{(i)}})}{\eta}\Big\}\Big\|^2\Big]}_\text{(B)}
\end{multline}

In (A), we take expectation with respect to $\mathcal{S}_k$ and $Q_D(.)$ -- for that, we use Lemma 4 of \cite{reisizadeh2020fedpaq}. Note that $\x_{k,\tau}^{(i)} - \x_k$ in their lemma corresponds to $({\bm{w}_{k,E}^{(i)} - \bm{w}_{k}})$ in our case. Specifically, using eqn. (59) and (60) in \cite{reisizadeh2020fedpaq} (they also have \Cref{as5}), we get:
\begin{flalign}
    \nonumber
    \text{(A)} & \leq \frac{1}{r(n-1)}\Big(1 - \frac{r}{n}\Big)4(1+q)\sum_{i \in [n]}\mathbb{E}[\|{\bm{w}_{k,E}^{(i)} - \bm{w}_{k}}\|^2] 
    \\
    \nonumber
    & = \frac{1}{r(n-1)}\Big(1 - \frac{r}{n}\Big) 4(1+q)\sum_{i \in [n]}\mathbb{E}[\|\sum_{\tau=0}^{E-1}\eta \bm{v}_{k,\tau}^{(i)}\|^2] 
    \\
    \label{eq:oct13-7}
    & \leq \frac{\eta^2}{r(n-1)}\Big(1 - \frac{r}{n}\Big) 4(1+q)E\sum_{i \in [n]}\sum_{\tau=0}^{E-1}\mathbb{E}[\|\bm{v}_{k,\tau}^{(i)}\|^2]
\end{flalign}
Next, we deal with (B). We have:
\begin{flalign}
    \nonumber
    \text{(B)} & = \mathbb{E}\Big[\mathbb{E}_{Q_D}\Big[\Big\|\frac{1}{n}\sum_{i \in [n]}\Big\{Q_{D}\Big({\bm{w}_{k,E}^{(i)} - \bm{w}_{k}}\Big) - \Big({\bm{w}_{k,E}^{(i)} - \bm{w}_{k}}\Big)\Big\}\Big\|^2\Big]\Big]
    \\
    \nonumber
    & \leq \frac{q}{n^2}\sum_{i \in [n]}\mathbb{E}\Big[\Big\|{\bm{w}_{k,E}^{(i)} - \bm{w}_{k}}\Big\|^2\Big]
    \\
    \label{eq:oct13-8}
    & \leq \frac{q E \eta^2 }{n^2}\sum_{i \in [n]}\sum_{\tau=0}^{E-1}\mathbb{E}[\|\bm{v}_{k,\tau}^{(i)}\|^2].
\end{flalign}
Now using (\ref{eq:oct13-7}) and (\ref{eq:oct13-8}) in (\ref{eq:may11-n10}), we get:
\begin{flalign}
    \nonumber
    \text{(V*)} & \leq \frac{\eta^2}{r(n-1)}\Big(1 - \frac{r}{n}\Big) 4 (1+q) E\sum_{i \in [n]}\sum_{\tau=0}^{E-1}\mathbb{E}[\|\bm{v}_{k,\tau}^{(i)}\|^2] + \frac{\eta^2 q E}{n^2}\sum_{i \in [n]}\sum_{\tau=0}^{E-1}\mathbb{E}[\|\bm{v}_{k,\tau}^{(i)}\|^2]
    \\
    \label{eq:nov-1-thm1-10}
    & \leq
    4 \eta^2 E\Big(\frac{q}{n^2} + \frac{(1+q)}{r(n-1)}\Big(1 - \frac{r}{n}\Big)\Big)\sum_{i \in [n]}\sum_{\tau=0}^{E-1}\mathbb{E}[\|\bm{v}_{k,\tau}^{(i)}\|^2].
\end{flalign}
This gives us the desired result.
\end{proof}

\begin{lemma}
\label{lem1-oct20}
For any $L$-smooth function $h(\bm{x})$, we have $\forall$ $\bm{x}$:
\[\|\nabla h(\bm{x})\|^2 \leq 2L(h(\bm{x}) - h^{*}) \text{ where } h^{*} = \min_{\bm{x}}h(\bm{x}).\]

\end{lemma}
\begin{proof}
For any $y$, we have that:
\begin{equation}
    \label{eq1-lem1-oct20}
    h^{*} \leq h(\bm{y}) \leq \underbrace{h(\bm{x}) + \langle \nabla h(\bm{x}), \bm{y} - \bm{x} \rangle + \frac{L}{2}\|\bm{y} - \bm{x}\|^2}_\text{$:= h_2(\bm{y})$}
\end{equation}
Setting $\nabla h_2(\bm{y}) = \vec{0}$
, we get that $\widehat{\bm{y}} = \bm{x} - \frac{1}{L} \nabla h(\bm{x})$ is the minimizer of $h_2(\bm{y})$ (which is a quadratic with respect to $\bm{y}$). Plugging this back in (\ref{eq1-lem1-oct20}) gives us:
\begin{equation}
    \label{eq2-lem1-oct20}
    h^{*} \leq {h(\bm{x}) + \Big \langle \nabla h(\bm{x}), -\frac{1}{L} \nabla h(\bm{x}) \Big  \rangle  + \frac{L}{2}\Big \|-\frac{1}{L} \nabla h(\bm{x})\Big \|^2} = h(\bm{x}) - \frac{1}{2L}\|\nabla h(\bm{x})\|^2.
\end{equation}
This gives us the desired result.
\end{proof}

\subsection{Detailed Proof of the Result of \texttt{FedLOMO}}
\label{sec-pf-1}
Let us redefine the quantities needed to prove the results of \texttt{FedLOMO}.
\[\overline{\bm{w}}_{k,\tau} \triangleq \frac{1}{n}\sum_{i \in [n]} \bm{w}_{k,\tau}^{(i)} \text{ and } \overline{\bm{v}}_{k,\tau} \triangleq \frac{1}{n}\sum_{i \in [n]} \bm{v}_{k,\tau}^{(i)}\]
\[{\bm{e}}_{k,\tau}^{(i)} \triangleq \bm{v}_{k,\tau}^{(i)} - \nabla f_i(\bm{w}_{k, \tau}^{(i)}) \text{ and } \widetilde{\bm{e}}_{k,\tau}^{(i)} \triangleq \nabla f_i(\bm{w}_{k, \tau}^{(i)}) - \nabla f_i(\overline{\bm{w}}_{k,\tau})\]

\noindent \textbf{Proof of \Cref{fl-thm3}}:
\begin{proof}
Let us set $\eta_k = \eta$.
\\
Using \Cref{oct-13-lem1}, with $\eta < \frac{1}{L}$ and $E < \frac{1}{4}\text{min}\Big(\frac{1}{\eta L}, \frac{1}{\eta^2 L^2} - \frac{1}{\eta L}\Big)$:
\begin{multline}
    \label{sept25-eq1}
    \mathbb{E}[f(\bm{w}_{k+1})] \leq \mathbb{E}[f(\bm{w}_{k})] -\frac{\eta E}{2}\mathbb{E}[\|\nabla f(\bm{w}_{k})\|^2] - \frac{\eta}{2}(1 - \eta^2  L^2 E^2 - \eta L E )\sum_{\tau=0}^{E-1}\mathbb{E}[\|\overline{\bm{v}}_{k,\tau}\|^2]
    \\
    + 16\eta L E^2\Big\{\frac{\eta^2 L (\alpha E + 4)}{n^2} + \frac{\eta}{2}\Big(\frac{q}{n^2} + \frac{4(1+q)}{r(n-1)}\Big(1 - \frac{r}{n}\Big) \Big)\Big\}\sum_{i \in [n]}{\mathbb{E}[\|\nabla f_i(\bm{w}_{k})\|^2]}.
\end{multline}
Note here that for $\eta < \frac{1}{2L}$, $\frac{1}{\eta L} < \frac{1}{\eta^2 L^2} - \frac{1}{\eta L}$ and so $E < \frac{1}{4\eta L}$ or $\eta L E < \frac{1}{4}$. Since $E>1$, we are just left with $\eta L E < \frac{1}{4}$. 

Next, we circumvent the need for 
the bounded client dissimilarity assumption by using the fact that each $f_i$ is $L$-smooth and so $\|\nabla f_i(\bm{w}_k)\|^2 \leq 2L (f_i(\bm{w}_k) - f_i^{*})$ using \Cref{lem1-oct20}. Hence:
\begin{equation}
    \label{sept25-eq2}
    \sum_{i \in [n]}{\mathbb{E}[\|\nabla f_i(\bm{w}_{k})\|^2]} \leq 2L\sum_{i \in [n]}\mathbb{E}[(f_i(\bm{w}_k) - f_i^{*})] = 2nL \mathbb{E}[(f(\bm{w}_k) - f^{*} + \Delta^{*})],
\end{equation}
where $\Delta^{*} := f^{*} - \frac{1}{n}\sum_{i=1}^n f_i^{*}$. 
Using all this in (\ref{sept25-eq1}), we get:
\begin{multline}
    \label{sept25-eq3}
    \mathbb{E}[f(\bm{w}_{k+1})] \leq \mathbb{E}[f(\bm{w}_{k})] -\frac{\eta E}{2}\mathbb{E}[\|\nabla f(\bm{w}_{k})\|^2] - \frac{\eta}{2}\underbrace{(1 - \eta^2  L^2 E^2 - \eta L E )}_\text{$> 0$ for $\eta L E < \frac{1}{4}$}\sum_{\tau=0}^{E-1}\mathbb{E}[\|\overline{\bm{v}}_{k,\tau}\|^2]
    \\
    + 32 \eta L^2 E^2\Big\{{\eta^2 L}{\Big( \frac{\alpha E +4}{n}\Big)}
    + \frac{\eta}{2}\underbrace{\Big(\frac{q}{n} + \frac{4(1+q)(n-r)}{r(n-1)} \Big)}_{:=B}\Big\}\mathbb{E}[(f(\bm{w}_k) - f^{*} + \Delta^{*})].
\end{multline}
Note that $(1 - \eta^2  L^2 E^2 - \eta L E ) > \frac{11}{16}$ for $\eta L E < \frac{1}{4}$.
Further, $-f^{*} + \Delta^{*} = -f^{*} + f^{*} - \frac{1}{n}\sum_{i=1}^n f_i^{*} = - \frac{1}{n}\sum_{i=1}^n f_i^{*}$; hence, we can ignore the corresponding term when the $f_i^{*}$'s are non-negative. Re-writing the above equation, we get:
\begin{flalign}
    \nonumber
    \mathbb{E}[f(\bm{w}_{k+1})] & \leq \mathbb{E}[f(\bm{w}_{k})] -\frac{\eta E}{2}\mathbb{E}[\|\nabla f(\bm{w}_{k})\|^2] 
    + 32\eta L^2 E^2\Big\{{\eta^2 L}\Big(\frac{\alpha E +4}{n}\Big) + \frac{\eta B}{2}\Big\}\mathbb{E}[f(\bm{w}_k)]
    \\
    & \leq \mathbb{E}[f(\bm{w}_{k})]\Big\{1+\underbrace{\Big(\frac{32 \eta^3 L^3 E^3}{n}\Big(\alpha + \frac{4}{E}\Big) + 16 B\eta^2 L^2 E^2\Big)}_{=\zeta}\Big\} - \frac{\eta E}{2}\mathbb{E}[\|\nabla f(\bm{w}_{k})\|^2]. 
\end{flalign}
Let us denote $\frac{32 \eta^3 L^3 E^3}{n}\Big(\alpha + \frac{4}{E}\Big) + 16 B\eta^2 L^2 E^2$ as $\zeta$ for brevity. 
\\
Unfolding the above recursion from $k=0$ through $K-1$, we get:
\begin{flalign}
    \label{thm3-eq2}
    \mathbb{E}[f(\bm{w}_{K})] \leq f(\bm{w}_{0})(1+\zeta)^K - \frac{\eta E}{2}\sum_{k=0}^{K-1}(1+\zeta)^{(K-1-k)}\mathbb{E}[\|\nabla f(\bm{w}_{k})\|^2]. 
\end{flalign}
Re-arranging the above, we get:
\begin{flalign}
    \label{thm3-eq3}
    \sum_{k=0}^{K-1}p_k \mathbb{E}[\|\nabla f(\bm{w}_{k})\|^2] \leq \frac{2}{\eta E} \frac{f(\bm{w}_0)(1+\zeta)^K}{\sum_{k=0}^{K-1}(1+\zeta)^k}, \text{ where } p_k = \frac{(1+\zeta)^{(K-1-k)}}{\sum_{k=0}^{K-1}(1+\zeta)^k}.
\end{flalign}
Notice that $p_k$ defines a distribution over $k$. Hence, the LHS is $\mathbb{E}_{k \sim \mathbb{P}(k)}[\mathbb{E}[\|\nabla f(\bm{w}_{k})\|^2]]$ with $\mathbb{P}(k) = p_k$. Incorporating this and simplifying further, we get:
\begin{flalign}
    \label{thm3-eq4}
    \mathbb{E}_{k \sim \mathbb{P}(k)}[\mathbb{E}[\|\nabla f(\bm{w}_{k})\|^2]] \leq \frac{2}{\eta E} \Big\{\frac{f(\bm{w}_0)\zeta}{1 - (1+\zeta)^{-K}}\Big\}, \text{ where } \mathbb{P}(k) = \frac{(1+\zeta)^{(K-1-k)}}{\sum_{k=0}^{K-1}(1+\zeta)^k}.
\end{flalign}
Also note that: $(1+\zeta)^{-K} < 1 - \zeta K + {\zeta^2}\frac{K(K+1)}{2} < 1 - \zeta K + {\zeta^2}K^2$. Hence, $1 - (1+\zeta)^{-K} > \zeta K (1 - \zeta K)$. Using this in (\ref{thm3-eq4}), we have for $\zeta K < 1$:
\begin{flalign}
    \label{thm3-eq5}
    \mathbb{E}_{k \sim \mathbb{P}(k)}[\mathbb{E}[\|\nabla f(\bm{w}_{k})\|^2]] \leq \frac{2 f(\bm{w}_0)}{\underbrace{\eta E K (1 - \zeta K)}_{=d(\eta)}}, \text{ where } \mathbb{P}(k) = \frac{(1+\zeta)^{(K-1-k)}}{\sum_{k=0}^{K-1}(1+\zeta)^k}.
\end{flalign}
Plugging in the value of $\zeta$ in (\ref{thm3-eq5}), the denominator, $d(\eta) = \eta E K \Big(1 - 16 \eta^2 L^2 E^2 \Big(\frac{2\eta L E}{n}\Big(\alpha + \frac{4}{E}\Big) + B\Big)K\Big)$. 

Before going ahead, we would like to highlight that the reason \texttt{FedLOMO} does not achieve the optimal rate is because $B$ is a constant that is not $\mathcal{O}(\eta L E)$ in general; if we were to consider the special case of no compression and full-device participation (i.e., $r=n$), then $B$ would be 0 which would allow \texttt{FedLOMO} to achieve the optimal rate.

Let us choose $\eta = \frac{1}{8 L E \sqrt{B K}}$. Note that:
\begin{equation}
    \label{may14-1}
    \eta L E \leq \frac{1}{4} \text{ for } K \geq \frac{1}{4B}.
\end{equation}
Thus, for sufficiently large $K$, this choice of $\eta$ is valid. Also:
\begin{equation}
    \label{may14-2}
    \zeta K = \frac{1}{4} + \frac{1}{16 B^{1.5}\sqrt{K}}\Big(\frac{1}{n}\Big(\alpha + \frac{4}{E}\Big)\Big) < \frac{3}{4} \text{ for } K > \frac{1}{64 B^3} \Big(\frac{1}{n}\Big(\alpha + \frac{4}{E}\Big)\Big).
\end{equation}
So for $K \geq \frac{1}{64 B^3} (\frac{1}{n}(\alpha + \frac{4}{E}))$,
\begin{equation}
    \label{may14-3}
    d(\eta) = \eta E K (1 - \zeta K) \geq \frac{\sqrt{K}}{8L\sqrt{B}}(1 - \frac{3}{4}) = \frac{\sqrt{K}}{32 L\sqrt{B}}.
\end{equation}
Plugging this in (\ref{thm3-eq5}), we get:
\small
\begin{flalign}
    \nonumber
    \mathbb{E}_{k \sim \mathbb{P}(k)}[\mathbb{E}[\|\nabla f(\bm{w}_{k})\|^2]] \leq & \frac{64 \sqrt{B} L f(\bm{w}_0)}{K^{1/2}}, \text{ where } \mathbb{P}(k) = \frac{(1+\zeta)^{(K-1-k)}}{\sum_{k=0}^{K-1}(1+\zeta)^k} \text{ for } k \in \{0,\ldots,K-1\},
    \\
    \label{thm3-eq12}
    & \zeta = \frac{1}{4K} + \frac{1}{16 B^{1.5}K^{1.5}} \Big(\frac{1}{n}\Big(\alpha + \frac{4}{E}\Big)\Big) \text{ and } B = \frac{q}{n} + \frac{4(1+q)(n-r)}{r(n-1)}.
\end{flalign}
\normalsize
This concludes the proof. 
\end{proof}

\noindent \textbf{Key lemma used in the proof of \Cref{fl-thm3}}:
\begin{lemma}
\label{oct-13-lem1}
For $\eta_k = \eta$ where $\eta < \frac{1}{L}$ and $E < \frac{1}{4}\text{min}\Big(\frac{1}{\eta L}, \frac{1}{\eta^2 L^2} - \frac{1}{\eta L}\Big)$ in \texttt{FedLOMO}, we have:
\begin{multline*}
    \mathbb{E}[f(\bm{w}_{k+1})] \leq \mathbb{E}[f(\bm{w}_{k})] -\frac{\eta E}{2}\mathbb{E}[\|\nabla f(\bm{w}_{k})\|^2] - \frac{\eta}{2}(1 - \eta^2  L^2 E^2 - \eta L E )\sum_{\tau=0}^{E-1}\mathbb{E}[\|\overline{\bm{v}}_{k,\tau}\|^2]
    \\
    + 16\eta L E^2\Big\{\frac{\eta^2 L (\alpha E + 4)}{n^2} + \frac{\eta}{2}\Big(\frac{q}{n^2} + \frac{4(1+q)}{r(n-1)}\Big(1 - \frac{r}{n}\Big) \Big)\Big\}\sum_{i \in [n]}{\mathbb{E}[\|\nabla f_i(\bm{w}_{k})\|^2]}.
\end{multline*}
\end{lemma}
\begin{proof}
By the $L$-smoothness of $f$, we have:
\small
\begin{equation}
    \label{eq:oct13-1}
    \mathbb{E}[f(\bm{w}_{k+1})] \leq \mathbb{E}[f(\bm{w}_{k})] + \underbrace{\mathbb{E}\Big[\Big\langle \nabla f(\bm{w}_{k}), \frac{1}{r}\sum_{i \in \mathcal{S}_k} Q_{D}({\bm{w}_{k,E}^{(i)} - \bm{w}_{k}})\Big\rangle \Big]}_\text{(I)} + \underbrace{\frac{L}{2}\mathbb{E}\Big[\Big\|\frac{1}{r}\sum_{i \in \mathcal{S}_k} Q_{D}({\bm{w}_{k,E}^{(i)} - \bm{w}_{k}})\Big\|^2\Big]}_\text{(II)}
\end{equation}
\normalsize
Let us analyze (I) first -- taking expectation with respect to $\mathcal{S}_k$ and $Q_D(.)$ (recall that $Q_D(.)$ is unbiased from \Cref{as5}), we get:
\begin{flalign}
    \nonumber
    \text{(I)} & = \mathbb{E}[\langle \nabla f(\bm{w}_{k}), \frac{1}{n}\sum_{i \in [n]}(\bm{w}_{k,E}^{(i)} - \bm{w}_{k})\rangle]
\end{flalign}
But this is the same as (III*) in the proof of \Cref{nov-1-lem0}; using \Cref{lem1-may11} and \Cref{as-het2} instead of \Cref{as-het}, we get:
\small
\begin{flalign}
    \label{eq:oct13-5}
    \text{(I)} \leq -\frac{\eta E}{2}\mathbb{E}[\|\nabla f(\bm{w}_{k})\|^2] -\frac{\eta}{2}(1 - \eta^2  L^2 E^2)\sum_{\tau=0}^{E-1}\mathbb{E}[\|\overline{\bm{v}}_{k,\tau}\|^2] + \frac{16 \eta^3 L^2 E^2 (\alpha E + 4)}{n^2}\sum_{i \in [n]}\mathbb{E}[\|\nabla f_i(\bm{w}_k)\|^2],
\end{flalign}
\normalsize
when $\eta < \frac{1}{L}$ and $E < \frac{1}{4}\text{min}\Big(\frac{1}{\eta L}, \frac{1}{\eta^2 L^2} - \frac{1}{\eta L}\Big)$.
\\
Let us now analyze (II). Recall that:
\begin{flalign}
    \nonumber
    \text{(II)} & = \frac{L}{2}\mathbb{E}\Big[\Big\|\frac{1}{r}\sum_{i \in \mathcal{S}_k} Q_{D}({\bm{w}_{k,E}^{(i)} - \bm{w}_{k}})\Big\|^2\Big].
\end{flalign}
Observe that:
\[\mathbb{E}_{\mathcal{S}_k}\Big[\frac{1}{r}\sum_{i \in \mathcal{S}_k} Q_{D}({\bm{w}_{k,E}^{(i)} - \bm{w}_{k}})\Big] = \frac{1}{n}\sum_{i \in [n]} Q_{D}({\bm{w}_{k,E}^{(i)} - \bm{w}_{k}}).\]
Hence:
\begin{multline}
    \label{eq:oct13-6}
    \text{(II)} = \frac{L}{2}\Big\{\underbrace{\mathbb{E}\Big[\Big\|\frac{1}{n}\sum_{i \in [n]} Q_{D}({\bm{w}_{k,E}^{(i)} - \bm{w}_{k}})\Big\|^2\Big]}_\text{(III)} 
    \\
    +
    \underbrace{\mathbb{E}\Big[\Big\|\frac{1}{r}\sum_{i \in \mathcal{S}_k} Q_{D}({\bm{w}_{k,E}^{(i)} - \bm{w}_{k}}) - \frac{1}{n}\sum_{i \in [n]} Q_{D}({\bm{w}_{k,E}^{(i)} - \bm{w}_{k}})\Big\|^2\Big]}_\text{(IV)}
    \Big\}.
\end{multline}
Note that in (III), the expectation is without $\mathcal{S}_k$. In (IV), we take expectation with respect to $\mathcal{S}_k$ and $Q_D(.)$ -- for that, we use Lemma 4 of \cite{reisizadeh2020fedpaq}. Note that $\x_{k,\tau}^{(i)} - \x_k$ in their lemma corresponds to $({\bm{w}_{k,E}^{(i)} - \bm{w}_{k}})$ in our case. Specifically, using eqn. (59) and (60) in \cite{reisizadeh2020fedpaq} (they also have \Cref{as5}), we get:
\begin{flalign}
    \nonumber
    \text{(IV)} & \leq \frac{1}{r(n-1)}\Big(1 - \frac{r}{n}\Big)4(1+q)\sum_{i \in [n]}\mathbb{E}[\|{\bm{w}_{k,E}^{(i)} - \bm{w}_{k}}\|^2] 
    \\
    \nonumber
    & = \frac{1}{r(n-1)}\Big(1 - \frac{r}{n}\Big) 4(1+q)\sum_{i \in [n]}\mathbb{E}[\|\sum_{\tau=0}^{E-1}\eta \bm{v}_{k,\tau}^{(i)}\|^2] 
    \\
    \label{eq:oct13-7-1}
    & \leq \frac{\eta^2}{r(n-1)}\Big(1 - \frac{r}{n}\Big) 4(1+q)E\sum_{i \in [n]}\sum_{\tau=0}^{E-1}\mathbb{E}[\|\bm{v}_{k,\tau}^{(i)}\|^2]
\end{flalign}
Next, we deal with (III). Noting that $\mathbb{E}_{Q_D}\Big[\frac{1}{n}\sum_{i \in [n]} Q_{D}({\bm{w}_{k,E}^{(i)} - \bm{w}_{k}})\Big] = ({\overline{\bm{w}}_{k,E} - \bm{w}_{k}})$, we get:
\begin{flalign}
    \nonumber
    \text{(III)} & = \mathbb{E}[\|{\overline{\bm{w}}_{k,E} - \bm{w}_{k}}\|^2] + \mathbb{E}\Big[\mathbb{E}_{Q_D}\Big[\Big\|\frac{1}{n}\sum_{i \in [n]}\Big\{Q_{D}\Big({\bm{w}_{k,E}^{(i)} - \bm{w}_{k}}\Big) - \Big({\bm{w}_{k,E}^{(i)} - \bm{w}_{k}}\Big)\Big\}\Big\|^2\Big]\Big]
    \\
    \nonumber
    & \leq \mathbb{E}\Big[\Big\|\sum_{\tau=0}^{E-1}\eta \overline{\bm{v}}_{k,\tau}\Big\|^2\Big] + \frac{q}{n^2}\sum_{i \in [n]}\mathbb{E}\Big[\Big\|{\bm{w}_{k,E}^{(i)} - \bm{w}_{k}}\Big\|^2\Big]
    \\
    \label{eq:oct13-8-1}
    & \leq \eta^2 E\sum_{\tau=0}^{E-1}\mathbb{E}[\|\overline{\bm{v}}_{k,\tau}\|^2] + \frac{q E \eta^2 }{n^2}\sum_{i \in [n]}\sum_{\tau=0}^{E-1}\mathbb{E}[\|\bm{v}_{k,\tau}^{(i)}\|^2]
\end{flalign}
Now, using (\ref{eq:oct13-7-1}) and (\ref{eq:oct13-8-1}) in (\ref{eq:oct13-6}) gives us:
\begin{flalign}
    \nonumber
    \text{(II)} & \leq \frac{L E \eta^2}{2}\Big\{\sum_{\tau=0}^{E-1}\mathbb{E}[\|\overline{\bm{v}}_{k,\tau}\|^2] + \Big(\frac{q}{n^2} + \frac{4(1+q)}{r(n-1)}\Big(1 - \frac{r}{n}\Big) \Big)\sum_{i \in [n]}\underbrace{\sum_{\tau=0}^{E-1}\mathbb{E}[\|\bm{v}_{k,\tau}^{(i)}\|^2}_\text{from \Cref{fl-lem-new1}} \Big\}
    \\
    \label{eq:oct13-9}
    & \leq \frac{L E \eta^2}{2}\Big\{\sum_{\tau=0}^{E-1}\mathbb{E}[\|\overline{\bm{v}}_{k,\tau}\|^2] + \Big(\frac{q}{n^2} + \frac{4(1+q)}{r(n-1)}\Big(1 - \frac{r}{n}\Big) \Big)16 E \sum_{i \in [n]}{\mathbb{E}[\|\nabla f_i(\bm{w}_{k})\|^2]}\Big\}.
\end{flalign}
Therefore, using (\ref{eq:oct13-5}) and (\ref{eq:oct13-9}) in (\ref{eq:oct13-1}), we get:
\begin{multline*}
    \mathbb{E}[f(\bm{w}_{k+1})] \leq \mathbb{E}[f(\bm{w}_{k})] 
    \\
    -\frac{\eta E}{2}\mathbb{E}[\|\nabla f(\bm{w}_{k})\|^2] -\frac{\eta}{2}(1 - \eta^2  L^2 E^2)\sum_{\tau=0}^{E-1}\mathbb{E}[\|\overline{\bm{v}}_{k,\tau}\|^2] + \frac{16 \eta^3 L^2 E^2 (\alpha E + 4)}{n^2}\sum_{i \in [n]}\mathbb{E}[\|\nabla f_i(\bm{w}_k)\|^2]
    \\
    + \frac{L E \eta^2}{2}\Big\{\sum_{\tau=0}^{E-1}\mathbb{E}[\|\overline{\bm{v}}_{k,\tau}\|^2] + \Big(\frac{q}{n^2} + \frac{4(1+q)}{r(n-1)}\Big(1 - \frac{r}{n}\Big) \Big)16 E \sum_{i \in [n]}{\mathbb{E}[\|\nabla f_i(\bm{w}_{k})\|^2]}\Big\}
\end{multline*}
\begin{multline}
    \implies \mathbb{E}[f(\bm{w}_{k+1})] \leq \mathbb{E}[f(\bm{w}_{k})] -\frac{\eta E}{2}\mathbb{E}[\|\nabla f(\bm{w}_{k})\|^2] - \frac{\eta}{2}(1 - \eta^2  L^2 E^2 - \eta L E )\sum_{\tau=0}^{E-1}\mathbb{E}[\|\overline{\bm{v}}_{k,\tau}\|^2]
    \\
    + 16\eta L E^2\Big\{\frac{\eta^2 L (\alpha E + 4)}{n^2} + \frac{\eta}{2}\Big(\frac{q}{n^2} + \frac{4(1+q)}{r(n-1)}\Big(1 - \frac{r}{n}\Big) \Big)\Big\}\sum_{i \in [n]}{\mathbb{E}[\|\nabla f_i(\bm{w}_{k})\|^2]}
\end{multline}
This completes the proof.
\end{proof}

\section{Convergence of \texttt{FedAvg} under \texorpdfstring{\Cref{as-het}}{Lg}}
\label{sec:fed_avg_conv}
Here, we provide a convergence result for \texttt{FedAvg} (\Cref{alg:fed-avg}) in the absence of the bounded client dissimilarity assumption (i.e. \cref{eq:bcd}) and instead assuming that \Cref{as-het} holds for \texttt{FedAvg}; we restate it below for \texttt{FedAvg}.

\begin{assumption}[\textbf{\Cref{as-het} for \texttt{FedAvg}}]
\label{as-het3}
Suppose all clients participate, i.e. $r=n$, in the $(k+1)^{\text{st}}$ round of \texttt{FedAvg} (\Cref{alg:fed-avg}). 
Let $\bm{w}_{k,\tau}^{(i)}$ be the $i^{\text{th}}$ client's local parameter at the $(\tau+1)^{\text{st}}$ local step of the $(k+1)^{\text{st}}$ round of \texttt{FedAvg}, for $i \in [n]$. Define $\widetilde{\bm{e}}_{k,\tau}^{(i)} \triangleq \nabla f_i(\bm{w}_{k,\tau}^{(i)}) - \nabla f_i(\overline{\bm{w}}_{k,\tau})$, where $\overline{\bm{w}}_{k,\tau} \triangleq \frac{1}{n}\sum_{i \in [n]} \bm{w}_{k,\tau}^{(i)}$.
Then for some $\alpha \ll n$: 
\[\mathbb{E}\Big[\Big\|\sum_{i \in [n]}\widetilde{\bm{e}}_{k, \tau}^{(i)}\Big\|^2\Big] \leq \alpha \sum_{i \in [n]} \mathbb{E}\Big[\Big\|\widetilde{\bm{e}}_{k,\tau}^{(i)}\Big\|^2\Big], \text{ } \forall \text{ } \tau \in [E].\]
\end{assumption}
Again, in the worst case, this assumption will always hold with $\alpha = n$. Also, as discussed after \Cref{as-het}, we expect $\alpha$ to increase as the degree of client heterogeneity increases. 

Before presenting the convergence result, we show empirical proof that \Cref{as-het3} holds. For this, we compute and plot $\alpha$ (as we did in \Cref{sec:het-asm-expt}) for 8 and 4 bit \texttt{FedAvg} on CIFAR-10 and FMNIST, respectively; the results are in \Cref{fig:het1}.

\begin{figure}[!htb]
\centering 
\subfloat[CIFAR-10]{
    \label{fig:het_c}
	\includegraphics[width=0.45\textwidth]{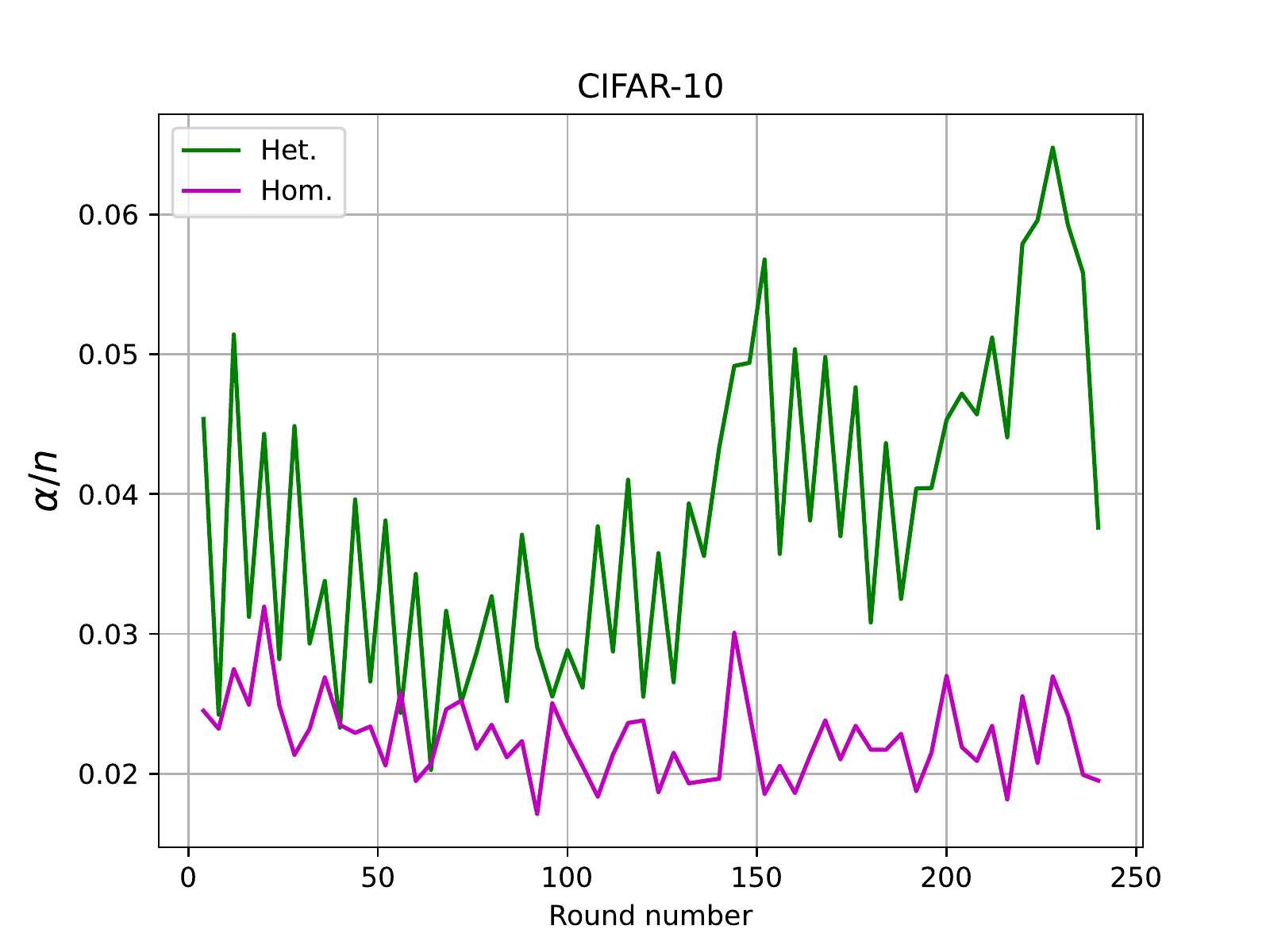}
	} 
\subfloat[FMNIST]{
    \label{fig:het_d}
	\includegraphics[width=0.45\textwidth]{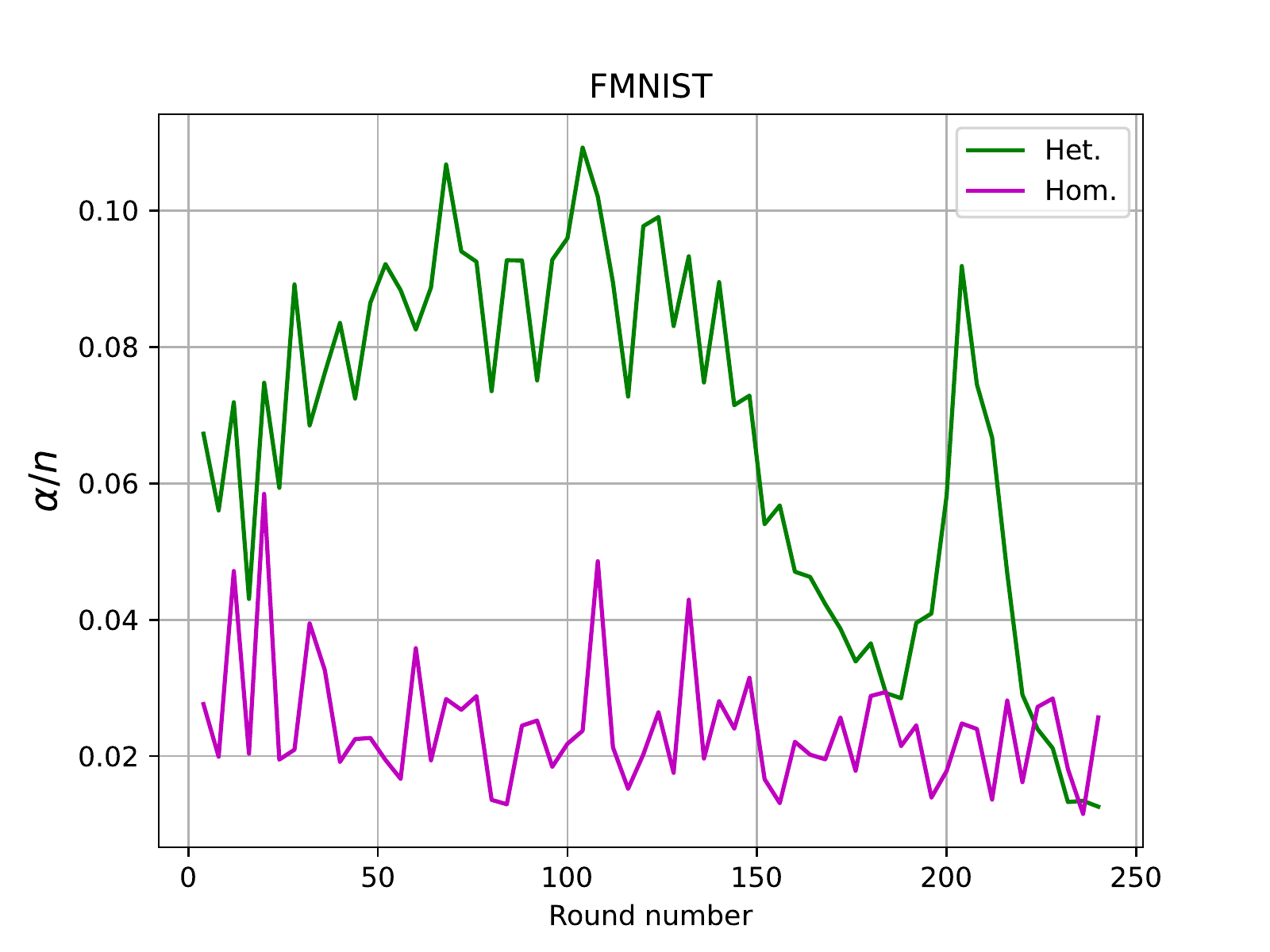}
	} 
\caption{Variation of $(\frac{\alpha}{n})$ over different rounds of $8$ and $4$ bit \texttt{FedAvg} for CIFAR-10 (Fig. \ref{fig:het_a}) and FMNIST (Fig. \ref{fig:het_b}) in the heterogeneous and homogeneous cases. In both cases, notice that $\alpha \ll n$ throughout training. {Also, as expected, observe that $(\frac{\alpha}{n})$ is higher for the heterogeneous case (except towards the end of training for FMNIST).}
}
\label{fig:het1}
\end{figure}

\begin{theorem}[\textbf{Smooth non-convex case for \texttt{FedAvg}}]
\label{thm-fedavg}
Suppose Assumptions \ref{as1}, \ref{as-may15} and \ref{as-het3} hold. Let $\sigma^2$ be the maximum variance of the local (client-level) stochastic gradients. In \texttt{FedAvg} (\Cref{alg:fed-avg}), set $\eta_{k} = \frac{1}{L E \sqrt{3 (\frac{(n-r)}{6r(n-1)} + \frac{4 \alpha}{9 n}) K }}$. 
Define a distribution $\mathbb{P}$ for $k \in \{0,\ldots,K-1\}$ such that $\mathbb{P}(k) = \frac{(1+\zeta)^{(K-1-k)}}{\sum_{k=0}^{K-1}(1+\zeta)^k}$ where $\zeta := \eta^2 L^2 E^2 \Big(\frac{(n-r)}{6r(n-1)} + \frac{4 \alpha}{9 n} \Big)$. Sample $k^{*}$ from $\mathbb{P}$. Then for $K \geq \frac{4}{3 (\frac{(n-r)}{6r(n-1)} + \frac{4 \alpha}{9 n})}$:
\begin{multline*}
    \mathbb{E}[\|\nabla f(\bm{w}_{k^{*}})\|^2] \leq \frac{3 L f(\bm{w}_0)}{\sqrt{K}} \sqrt{3\Big(\frac{(n-r)}{6r(n-1)} + \frac{4 \alpha}{9 n}\Big)} + \frac{1}{\sqrt{3 (\frac{(n-r)}{6r(n-1)} + \frac{4 \alpha}{9 n}) K }} \Big(\frac{1}{r E} + \frac{(n-r)}{3r(n-1)}\Big)\sigma^2
    \\
    + \frac{1}{{3 (\frac{n(n-r)}{6r(n-1)} + \frac{4 \alpha}{9}) K}}\Big(\frac{1}{E} + \frac{8\alpha}{9}\Big)\sigma^2.
\end{multline*}
So \texttt{FedAvg} needs $K = \mathcal{O}(\frac{1}{r \epsilon^2})$ rounds of communication to achieve $\mathbb{E}[\|\nabla f(\bm{w}_{k^{*}})\|^2] \leq \epsilon$ where $\epsilon < \mathcal{O}(\frac{1}{r})$.
\end{theorem}

Thus, we recover the same complexity for \texttt{FedAvg}/Local SGD (which is basically \texttt{FedAvg} with full-device participation) as \cite{karimireddy2019scaffold,koloskova2020unified,wang2019slowmo} -- but without the bounded client dissimilarity assumption.

\begin{proof}
Using \Cref{sep26-lem3}, for $\eta_k L E \leq \frac{1}{2}$, we can bound the per-round progress as:
\begin{multline}
    \label{eq:sept26-18}
    \mathbb{E}[f(\bm{w}_{k+1})] 
    \leq \mathbb{E}[f(\bm{w}_k)]
    - \frac{\eta_k E}{2} \mathbb{E}[\|\nabla f(\bm{w}_k)\|^2] + \eta_k^2 L E^2 \Big(\frac{(n-r)}{6r(n-1)} + \underbrace{\frac{8 \alpha \eta_k L E}{9 n}}_{\leq \frac{4 \alpha}{9 n}}\Big)\Big(\frac{1}{n}\sum_{i \in [n]} \mathbb{E}[\|\nabla {f}_i(\bm{w}_{k})\|^2]\Big) 
    \\
    + \frac{\eta_k^2 L E}{2} \Big(\frac{\eta_k L E}{n}\Big(1 + \frac{8\alpha E}{9}\Big) + \frac{1}{r} + \frac{(n-r)E}{3r(n-1)}\Big)\sigma^2.
\end{multline}
Now applying our earlier trick of using the $L$-smoothness and non-negativity of the $f_i$'s, we get:
\[\sum_{i \in [n]} \mathbb{E}[\|\nabla f_i(\bm{w}_k)\|^2 \leq \sum_{i \in [n]} 2L(\mathbb{E}[f_i(\bm{w}_k)] - f_i^{*}) \leq 2n L \mathbb{E}[f(\bm{w}_k)] - 2L \sum_{i \in [n]} f_i^{*} \leq 2n L \mathbb{E}[f(\bm{w}_k)].\]
Putting this in \cref{eq:sept26-18}, we get for a constant learning rate of $\eta_k = \eta$:
\begin{multline}
    \label{eq:sept26-18-1}
    \mathbb{E}[f(\bm{w}_{k+1})] 
    \leq 
    \Big(1 + \eta^2 L^2 E^2 \Big(\frac{(n-r)}{6r(n-1)} + \frac{4 \alpha}{9 n} \Big)\Big) \mathbb{E}[f(\bm{w}_k)]
    - \frac{\eta E}{2} \mathbb{E}[\|\nabla f(\bm{w}_k)\|^2]
    \\
    + \frac{\eta^2 L E}{2} \Big(\frac{\eta L E}{n}\Big(1 + \frac{8\alpha E}{9}\Big) + \frac{1}{r} + \frac{(n-r)E}{3r(n-1)}\Big)\sigma^2.
\end{multline}
For ease of notation, define $\zeta := \eta^2 L^2 E^2 \Big(\frac{(n-r)}{6r(n-1)} + \frac{4 \alpha}{9 n} \Big)$ and $\zeta_2 := \Big(\frac{\eta L E}{n}\Big(1 + \frac{8\alpha E}{9}\Big) + \frac{1}{r} + \frac{(n-r)E}{3r(n-1)}\Big)$. Then, unfolding the recursion of \cref{eq:sept26-18-1} from  $k=0$ through to $k=K-1$, we get:
\begin{multline}
    \label{eq:sept26-19}
    \mathbb{E}[f(\bm{w}_{K})] 
    \leq 
    (1 + \zeta)^K f(\bm{w}_0)
    - \frac{\eta E}{2} \sum_{k=0}^{K-1}(1+\zeta)^{(K-1-k)} \mathbb{E}[\|\nabla f(\bm{w}_k)\|^2]
    + \frac{\eta^2 L E}{2} \zeta_2 \sigma^2 \sum_{k=0}^{K-1}(1+\zeta)^{(K-1-k)}.
\end{multline}
Let us define $p_k := \frac{(1+\zeta)^{(K-1-k)}}{\sum_{k'=0}^{K-1}(1+\zeta)^{(K-1-k')}}$. Then, re-arranging \cref{eq:sept26-19} and using the fact that $\mathbb{E}[f(\bm{w}_{K})] \geq 0$, we get:
\begin{flalign}
    \sum_{k=0}^{K-1}p_k \mathbb{E}[\|\nabla f(\bm{w}_k)\|^2] & \leq \frac{2 (1 + \zeta)^K f(\bm{w}_0)}{\eta E \sum_{k'=0}^{K-1}(1+\zeta)^{k'}} + {\eta L}\zeta_2 \sigma^2 
    \\
    \label{eq:sep27-1}
    & = \frac{2 \zeta f(\bm{w}_0)}{\eta E
    (1 - (1+\zeta)^{-K})} + \eta L E \Big(\frac{\eta L}{n}\Big(1 + \frac{8\alpha E}{9}\Big) + \frac{1}{r E} + \frac{(n-r)}{3r(n-1)}\Big)\sigma^2,
\end{flalign}
where the last step follows by using the fact that $\sum_{k'=0}^{K-1}(1+\zeta)^{k'} = \frac{(1+\zeta)^{K} - 1}{\zeta}$ and plugging in the value of $\zeta_2$. Now as we did in the proof of \Cref{fl-thm3}:
\begin{equation*}
    (1+\zeta)^{-K} < 1 - \zeta K + {\zeta^2}\frac{K(K+1)}{2} < 1 - \zeta K + {\zeta^2}K^2 \implies 1 - (1+\zeta)^{-K} > \zeta K (1 - \zeta K).
\end{equation*}
Plugging this in \cref{eq:sep27-1}, we have for $\zeta K < 1$:
\begin{flalign}
    \label{eq:sep27-2}
    \sum_{k=0}^{K-1}p_k \mathbb{E}[\|\nabla f(\bm{w}_k)\|^2] & \leq \frac{2 f(\bm{w}_0)}{\eta E K (1 - \zeta K)} + \eta L E \Big(\frac{\eta L}{n}\Big(1 + \frac{8\alpha E}{9}\Big) + \frac{1}{r E} + \frac{(n-r)}{3r(n-1)}\Big)\sigma^2.
\end{flalign}
Let us pick $\eta = \frac{1}{L E \sqrt{3 (\frac{(n-r)}{6r(n-1)} + \frac{4 \alpha}{9 n}) K }}$. With this choice, we have $\zeta K = \frac{1}{3} < 1$.
Note that we also need to have $\eta L E \leq \frac{1}{2}$; this happens for $K \geq \frac{4}{3 (\frac{(n-r)}{6r(n-1)} + \frac{4 \alpha}{9 n})}$. Putting $\eta = \frac{1}{L E \sqrt{3 (\frac{(n-r)}{6r(n-1)} + \frac{4 \alpha}{9 n}) K }}$ in \cref{eq:sep27-2}, we get:
\begin{multline}
    \sum_{k=0}^{K-1}p_k \mathbb{E}[\|\nabla f(\bm{w}_k)\|^2] \leq \frac{3 L f(\bm{w}_0)}{\sqrt{K}} \sqrt{3\Big(\frac{(n-r)}{6r(n-1)} + \frac{4 \alpha}{9 n}\Big)} + \frac{1}{\sqrt{3 (\frac{(n-r)}{6r(n-1)} + \frac{4 \alpha}{9 n}) K }} \Big(\frac{1}{r E} + \frac{(n-r)}{3r(n-1)}\Big)\sigma^2
    \\
    + \frac{1}{{3 (\frac{n(n-r)}{6r(n-1)} + \frac{4 \alpha}{9}) K}}\Big(\frac{1}{E} + \frac{8\alpha}{9}\Big)\sigma^2.
\end{multline}
This finishes the proof.
\end{proof}

\begin{lemma}
\label{sep26-lem3}
For $\eta_k L E \leq \frac{1}{2}$, we have:
\begin{multline*}
    \mathbb{E}[f(\bm{w}_{k+1})] 
    \leq \mathbb{E}[f(\bm{w}_k)]
    - \frac{\eta_k E}{2} \mathbb{E}[\|\nabla f(\bm{w}_k)\|^2] + \eta_k^2 L E^2 \Big(\frac{(n-r)}{6r(n-1)} + \frac{8 \alpha \eta_k L E}{9 n}\Big)\Big(\frac{1}{n}\sum_{i \in [n]} \mathbb{E}[\|\nabla {f}_i(\bm{w}_{k})\|^2]\Big) 
    \\
    + \frac{\eta_k^2 L E}{2} \Big(\frac{\eta_k L E}{n}\Big(1 + \frac{8\alpha E}{9}\Big) + \frac{1}{r} + \frac{(n-r)E}{3r(n-1)}\Big)\sigma^2.
\end{multline*}
\end{lemma}
\begin{proof}
Define
\[
\widehat{\bm{u}}_{k,\tau}^{(i)} := \nabla \widetilde{f}_i(\bm{w}^{(i)}_{k, \tau}; \mathcal{B}^{(i)}_{k, \tau}) 
\text{, }
\widehat{\bm{u}}_{k,\tau} := \frac{1}{n}\sum_{i \in [n]} \widehat{\bm{u}}_{k,\tau}^{(i)}
\text{, }
\bm{u}_{k,\tau} := \frac{1}{n}\sum_{i \in [n]}\nabla f_i(\bm{w}^{(i)}_{k, \tau}) \text{, } \]
\[\overline{\bm{w}}_{k,\tau} := \frac{1}{n}\sum_{i \in [n]}\bm{w}^{(i)}_{k, \tau} \text{ and } \widetilde{\bm{e}}_{k,\tau}^{(i)} = \nabla f_i(\bm{w}^{(i)}_{k, \tau}) - \nabla f_i(\overline{\bm{w}}_{k,\tau}).\]
Then:
\begin{equation}
    \label{eq:feb28-1}
    \bm{w}_{k+1} = \bm{w}_k - \eta_k \sum_{\tau=0}^{E-1}\Big(\frac{1}{r}\sum_{i \in \mathcal{S}_k}\widehat{\bm{u}}_{k,\tau}^{(i)}\Big).
\end{equation}
\begin{equation}
    \label{eq:feb28-1-0}
    \overline{\bm{w}}_{k,\tau} = \bm{w}_k - \eta_k \sum_{t=0}^{\tau-1}\widehat{\bm{u}}_{k,t}.
\end{equation}
\begin{equation}
    \label{eq:feb28-2}
    \mathbb{E}_{\{\mathcal{B}^{(i)}_{k, \tau}\}_{i=1}^n}[\widehat{\bm{u}}_{k,\tau}] = \bm{u}_{k,\tau}.
\end{equation}
\begin{equation}
    \label{eq:feb28-3}
    \mathbb{E}\Big[\Big\|\sum_{t=0}^{\tau-1}\widehat{\bm{u}}_{k,t}\Big\|^2\Big] \leq \tau \sum_{t=0}^{\tau-1}\mathbb{E}[\|\bm{u}_{k,t}\|^2] + \frac{\tau \sigma^2}{n}.
\end{equation}
\begin{equation}
    \label{eq:feb28-3-1}
    \mathbb{E}\Big[\Big\|\sum_{t=0}^{\tau-1}\widehat{\bm{u}}_{k,t}^{(i)}\Big\|^2\Big] \leq \tau \sum_{t=0}^{\tau-1}\mathbb{E}[\|\nabla f_i(\bm{w}^{(i)}_{k, t})\|^2] + {\tau \sigma^2}.
\end{equation}
Recall that $\sigma^2$ is the maximum variance of the local (client-level) stochastic gradients.
In \cref{eq:feb28-3}, the expectation is w.r.t. $\{\mathcal{B}^{(i)}_{k, t}\}_{i=1, t=0}^{n, \tau-1}$ and it follows due to the independence of the noise in each local update of each client. 
Similarly, \cref{eq:feb28-3-1}, the expectation is w.r.t. $\{\mathcal{B}^{(i)}_{k, t}\}_{t=0}^{\tau-1}$ and it follows due to the independence of the noise in each local update.
\\
\\
Next, using the $L$-smoothness of $f$ and \cref{eq:feb28-1}, we get
\small
\begin{flalign}
    \label{eq:feb28-4-0-1}
    \mathbb{E}[f(\bm{w}_{k+1})] & \leq 
    \mathbb{E}[f(\bm{w}_k)] - \mathbb{E}\Big[ \Big\langle \nabla f(\bm{w}_k), \eta_k \sum_{\tau=0}^{E-1} \Big(\frac{1}{r}\sum_{i \in \mathcal{S}_k}\widehat{\bm{u}}_{k,\tau}^{(i)}\Big) \Big\rangle\Big] + \frac{L}{2}\mathbb{E}\Big[\Big\|\eta_k \sum_{\tau=0}^{E-1} \Big(\frac{1}{r}\sum_{i \in \mathcal{S}_k}\widehat{\bm{u}}_{k,\tau}^{(i)}\Big)\Big\|^2\Big]
    \\
    \label{eq:feb28-4-0}
    & = \mathbb{E}[f(\bm{w}_k)] - \mathbb{E}[ \langle \nabla f(\bm{w}_k), \sum_{\tau=0}^{E-1} \eta_k \widehat{\bm{u}}_{k,\tau} \rangle] + \frac{\eta_k^2 L}{2}\Big\{\frac{n(r-1)}{r(n-1)}\mathbb{E}[\|\sum_{\tau=0}^{E-1} \widehat{\bm{u}}_{k,\tau}\|^2] 
    \\
    \nonumber
    & + \frac{(n-r)}{r(n-1)}\Big(\frac{1}{n}\sum_{i \in [n]} \mathbb{E}[\|\sum_{\tau=0}^{E-1} \widehat{\bm{u}}_{k,\tau}^{(i)}\|^2]\Big)\Big\}
    \\
    \label{eq:feb28-4}
    & \leq \mathbb{E}[f(\bm{w}_k)] - \eta_k \mathbb{E}[\langle \nabla f(\bm{w}_k), 
    \sum_{\tau=0}^{E-1}
    {\bm{u}}_{k,\tau}\rangle] +
    \frac{\eta_k^2 L E}{2}\Big\{
    \frac{n(r-1)}{r(n-1)} \Big(\sum_{\tau=0}^{E-1}\mathbb{E}[\|{\bm{u}}_{k,\tau}\|^2] + \frac{\sigma^2}{n} \Big) 
    \\
    \nonumber
    & + \frac{(n-r)}{r(n-1)}
    \Big(\frac{1}{n}\sum_{i \in [n]} \sum_{\tau=0}^{E-1}\mathbb{E}[\|\nabla f_i(\bm{w}^{(i)}_{k, \tau})\|^2] + {\sigma^2}\Big)\Big\}
\end{flalign}
\normalsize
Note that \cref{eq:feb28-4-0} follows by taking expectation w.r.t. $\mathcal{S}_k$ in \cref{eq:feb28-4-0-1}, while \cref{eq:feb28-4} follows from \cref{eq:feb28-2}, \cref{eq:feb28-3} and \cref{eq:feb28-3-1}.
\\
\\
For any 2 vectors $\bm{a}$ and $\bm{b}$, we have that $\langle \bm{a}, \bm{b} \rangle = \frac{1}{2}(\|\bm{a}\|^2 + \|\bm{b}\|^2 - \|\bm{a} - \bm{b}\|^2)$. Using this:
\begin{flalign}
\label{eq:feb28-5}
\langle \nabla f(\bm{w}_k), \sum_{\tau=0}^{E-1} \bm{u}_{k,\tau} \rangle & = \sum_{\tau=0}^{E-1} \langle \nabla f(\bm{w}_k), \bm{u}_{k,\tau} \rangle = \frac{1}{2}\sum_{\tau=0}^{E-1}(\|\nabla f(\bm{w}_k)\|^2 + \|\bm{u}_{k,\tau} \|^2 - \|\nabla f(\bm{w}_k) - \bm{u}_{k,\tau}\|^2).
\end{flalign}
Putting this in \cref{eq:feb28-4}, we get:
\begin{multline}
    \label{eq:feb28-6}
    \mathbb{E}[f(\bm{w}_{k+1})] 
    \leq \mathbb{E}[f(\bm{w}_k)] - \frac{\eta_k E}{2} \mathbb{E}[\|\nabla f(\bm{w}_k)\|^2] - \frac{\eta_k}{2}
    \Big(1 - \eta_k L E \frac{n(r-1)}{r(n-1)}\Big)
    \sum_{\tau=0}^{E-1} \mathbb{E}[\|{\bm{u}}_{k,\tau}\|^2]
    \\
    + \frac{\eta_k}{2} \underbrace{\sum_{\tau=0}^{E-1} \mathbb{E}[\|\nabla f(\bm{w}_k) - \bm{u}_{k,\tau}\|^2]}_\text{(A)}
    + \frac{\eta_k^2 L E}{2 r} \sigma^2 + \frac{(n-r)}{r(n-1)} \frac{\eta_k^2 L E}{2}
    \underbrace{\Big(\frac{1}{n}\sum_{i \in [n]} \sum_{\tau=0}^{E-1}\mathbb{E}[\|\nabla f_i(\bm{w}^{(i)}_{k, \tau})\|^2]\Big)}_\text{(B)}.
\end{multline}
We upper bound (A) and (B) using \Cref{sept26-lem2} and \Cref{sept26-lem1}, respectively. Plugging in these bounds, we get:
\begin{multline}
    \label{eq:sep26-15}
    \mathbb{E}[f(\bm{w}_{k+1})] 
    \leq \mathbb{E}[f(\bm{w}_k)] - \frac{\eta_k E}{2} \mathbb{E}[\|\nabla f(\bm{w}_k)\|^2] - \frac{\eta_k}{2}
    \underbrace{\Big(1 - \eta_k L E \frac{n(r-1)}{r(n-1)} - \eta_k^2 L^2 E^2 \Big)}_\text{(C)}
    \sum_{\tau=0}^{E-1} \mathbb{E}[\|{\bm{u}}_{k,\tau}\|^2]
    \\
    + \eta_k^2 L E^2 \Big(\frac{(n-r)}{6r(n-1)} + \frac{8 \alpha \eta_k L E}{9 n}\Big)\Big(\frac{1}{n}\sum_{i \in [n]} \mathbb{E}[\|\nabla {f}_i(\bm{w}_{k})\|^2]\Big) + \frac{\eta_k^2 L E}{2} \Big(\frac{\eta_k L E}{n}\Big(1 + \frac{8\alpha E}{9}\Big) + \frac{1}{r} + \frac{(n-r)E}{3r(n-1)}\Big)\sigma^2,
\end{multline}
for $\eta_k L E \leq \frac{1}{2}$. Note that $\text{(C)} \geq 0$ for $\eta_k L E \leq \frac{1}{2}$. Thus, for $\eta_k L E \leq \frac{1}{2}$, we have:
\begin{multline}
    \label{eq:sep26-16}
    \mathbb{E}[f(\bm{w}_{k+1})] 
    \leq \mathbb{E}[f(\bm{w}_k)]
    - \frac{\eta_k E}{2} \mathbb{E}[\|\nabla f(\bm{w}_k)\|^2] + \eta_k^2 L E^2 \Big(\frac{(n-r)}{6r(n-1)} + \frac{8 \alpha \eta_k L E}{9 n}\Big)\Big(\frac{1}{n}\sum_{i \in [n]} \mathbb{E}[\|\nabla {f}_i(\bm{w}_{k})\|^2]\Big) 
    \\
    + \frac{\eta_k^2 L E}{2} \Big(\frac{\eta_k L E}{n}\Big(1 + \frac{8\alpha E}{9}\Big) + \frac{1}{r} + \frac{(n-r)E}{3r(n-1)}\Big)\sigma^2.
\end{multline}
\end{proof}

\begin{lemma}
\label{sept26-lem2}
For $\eta_k L E \leq \frac{1}{2}$:
\begin{equation*}
    \sum_{\tau=0}^{E-1} \mathbb{E}[\|\nabla f(\bm{w}_k) - \bm{u}_{k,\tau}\|^2] \leq \eta_k^2 L^2 E^2 \sum_{\tau=0}^{E-1}\mathbb{E}[\|\bm{u}_{k,\tau}\|^2] + \frac{16 \alpha \eta_k^2 L^2 E^3}{9 n^2} \sum_{i \in [n]} \mathbb{E}[\|\nabla {f}_i(\bm{w}_{k})\|^2] + \frac{\eta_k^2 L^2 E^2}{n}\Big(1 + \frac{8\alpha E}{9}\Big) \sigma^2.
\end{equation*}
\end{lemma}
\begin{proof}
We have:
\begin{flalign}
    \mathbb{E}[\|\nabla f(\bm{w}_k) - \bm{u}_{k,\tau}\|^2] & = \mathbb{E}[\|\nabla f(\bm{w}_k) - \nabla f(\overline{\bm{w}}_{k,\tau}) + \nabla f(\overline{\bm{w}}_{k,\tau}) - \bm{u}_{k,\tau}\|^2]
    \\
    & \leq 2\mathbb{E}[\|\nabla f(\bm{w}_k) - \nabla f(\overline{\bm{w}}_{k,\tau})\|^2] + 2\mathbb{E}[\|\nabla f(\overline{\bm{w}}_{k,\tau}) - \bm{u}_{k,\tau}\|^2]
    \\
    \label{eq:sep26-1}
    & \leq 2L^2\mathbb{E}[\|\bm{w}_k - \overline{\bm{w}}_{k,\tau}\|^2] + 2\mathbb{E}\Big[\Big\|\frac{1}{n}\sum_{i \in [n]} \underbrace{(\nabla f_i(\overline{\bm{w}}_{k,\tau})- \nabla f_i(\bm{w}^{(i)}_{k, \tau}))}_{=-\widetilde{\bm{e}}_{k,\tau}^{(i)}}\Big\|^2\Big]
    \\
    \label{eq:sep26-2}
    & \leq 2 \eta_k^2 L^2 \mathbb{E}\Big[\Big\|\sum_{t=0}^{\tau-1}\widehat{\bm{u}}_{k,t}\Big\|^2\Big] + \frac{2 \alpha}{n^2} \sum_{i \in [n]} \mathbb{E}[\|\widetilde{\bm{e}}_{k,\tau}^{(i)}\|^2]
    \\
    \label{eq:sep26-3}
    & \leq 2 \eta_k^2 L^2 \Big(\tau \sum_{t=0}^{\tau-1}\mathbb{E}[\|\bm{u}_{k,t}\|^2] + \frac{\tau \sigma^2}{n}\Big) + \frac{2 \alpha L^2}{n^2} \sum_{i \in [n]} \mathbb{E}[\|\bm{w}^{(i)}_{k, \tau} - \overline{\bm{w}}_{k,\tau}\|^2].
\end{flalign}
\Cref{eq:sep26-1} follows from the $L$-smoothness of $f$ and the definition of $\bm{u}_{k,\tau}$. 
\Cref{eq:sep26-2} follows from \cref{eq:feb28-1-0} and \Cref{as-het3}. \Cref{eq:sep26-3} follows from \cref{eq:feb28-3} and the $L$-smoothness of $f_i$.

But:
\begin{flalign}
    \sum_{i \in [n]} \mathbb{E}[\|\bm{w}_{k, \tau}^{(i)} - \overline{\bm{w}}_{k,\tau}\|^2]
    & = \sum_{i \in [n]} \mathbb{E}[\|(\bm{w}_{k, 0}^{(i)} - \eta_k \sum_{t=0}^{\tau-1} \widehat{\bm{u}}_{k,t}^{(i)}) - (\overline{\bm{w}}_{k,0} - \eta_k \sum_{t=0}^{\tau-1} \widehat{\bm{u}}_{k,t})\|^2]
    \\
    \label{eq:sept26-4}
    & = \eta_k^2 \sum_{i \in [n]} \mathbb{E}[\| \sum_{t=0}^{\tau-1} \widehat{\bm{u}}_{k,t} - \sum_{t=0}^{\tau-1} \widehat{\bm{u}}_{k,t}^{(i)}\|^2] 
    \\
    & \leq \eta_k^2 \tau \sum_{i \in [n]} \sum_{t=0}^{\tau-1} \mathbb{E}[\|\widehat{\bm{u}}_{k,t} - \widehat{\bm{u}}_{k,t}^{(i)}\|^2]
    \\
    \label{eq:sept26-5}
    & = \eta_k^2 \tau \sum_{t=0}^{\tau-1} \sum_{i \in [n]} \mathbb{E}[\|\widehat{\bm{u}}_{k,t}\|^2 + \|\widehat{\bm{u}}_{k,t}^{(i)}\|^2 - 2\langle \widehat{\bm{u}}_{k,t}, \widehat{\bm{u}}_{k,t}^{(i)} \rangle]
\end{flalign}
\Cref{eq:sept26-4} follows because  $\bm{w}_{k,0}^{(i)} = \bm{w}_{k}$ $\forall$ $i \in [n]$, due to which  $\overline{\bm{w}}_{k,0} = \bm{w}_{k}$. Next, using the fact that $\widehat{\bm{u}}_{k,\tau} = \frac{1}{n}\sum_{i \in [n]} \widehat{\bm{u}}_{k,\tau}^{(i)}$, we can simplify \cref{eq:sept26-5} to:
\begin{flalign}
    \sum_{i \in [n]} \mathbb{E}[\|\bm{w}_{k, \tau}^{(i)} - \overline{\bm{w}}_{k,\tau}\|^2]
    & \leq \eta_k^2 \tau \sum_{t=0}^{\tau-1} \sum_{i \in [n]} (\mathbb{E}[\|\widehat{\bm{u}}_{k,\tau}^{(i)}\|^2] - \mathbb{E}[\|\widehat{\bm{u}}_{k,t}\|^2])
    \\
    & \leq \eta_k^2 \tau \sum_{t=0}^{\tau-1} \sum_{i \in [n]} \mathbb{E}[\|{\widehat{\bm{u}}_{k,\tau}^{(i)}}\|^2]
    \\
    \label{eq:sept26-10}
    & \leq \eta_k^2 \tau \sum_{t=0}^{\tau-1} \sum_{i \in [n]} (\mathbb{E}[\|\nabla {f}_i(\bm{w}^{(i)}_{k, t})\|^2] + \sigma^2).
\end{flalign}
Next, using \Cref{sept26-lem1} for $\eta_k L E \leq \frac{1}{2}$ in \cref{eq:sept26-10}, we get:
\begin{flalign}
    \label{eq:sept26-11}
    \sum_{i \in [n]} \mathbb{E}[\|\bm{w}_{k, \tau}^{(i)} - \overline{\bm{w}}_{k,\tau}\|^2]
    \leq \frac{4 \eta_k^2 \tau^2}{3} \sum_{i \in [n]} (2\mathbb{E}[\|\nabla {f}_i(\bm{w}_{k})\|^2] + \sigma^2).
\end{flalign}
Plugging \cref{eq:sept26-11} back in \cref{eq:sep26-3}, we get:
\begin{flalign}
    \mathbb{E}[\|\nabla f(\bm{w}_k) - \bm{u}_{k,\tau}\|^2] & \leq 2 \eta_k^2 L^2 \Big(\tau \sum_{t=0}^{\tau-1}\mathbb{E}[\|\bm{u}_{k,t}\|^2] + \frac{\tau \sigma^2}{n}\Big) + \frac{8 \alpha \eta_k^2 L^2 \tau^2}{3 n^2} \sum_{i \in [n]} (2\mathbb{E}[\|\nabla {f}_i(\bm{w}_{k})\|^2] + \sigma^2)
    \\
    \label{eq:sept26-12}
    & = 2 \eta_k^2 L^2 \tau \sum_{t=0}^{\tau-1}\mathbb{E}[\|\bm{u}_{k,t}\|^2] + \frac{16 \alpha \eta_k^2 L^2 \tau^2}{3 n^2} \sum_{i \in [n]} \mathbb{E}[\|\nabla {f}_i(\bm{w}_{k})\|^2] + \frac{\eta_k^2 L^2 \tau \sigma^2}{n}\Big(2 + \frac{8\alpha}{3} \tau\Big).
\end{flalign}
Summing up \cref{eq:sept26-12} for $\tau \in \{0, \ldots, E-1\}$, we get:
\begin{equation}
    \sum_{\tau=0}^{E-1} \mathbb{E}[\|\nabla f(\bm{w}_k) - \bm{u}_{k,\tau}\|^2] \leq \eta_k^2 L^2 E^2 \sum_{\tau=0}^{E-1}\mathbb{E}[\|\bm{u}_{k,\tau}\|^2] + \frac{16 \alpha \eta_k^2 L^2 E^3}{9 n^2} \sum_{i \in [n]} \mathbb{E}[\|\nabla {f}_i(\bm{w}_{k})\|^2] + \frac{\eta_k^2 L^2 E^2}{n}\Big(1 + \frac{8\alpha E}{9}\Big) \sigma^2.
\end{equation}

\end{proof}

\begin{lemma}
\label{sept26-lem1}
For $\eta_k L E \leq \frac{1}{2}$, we have:
\begin{equation*}
    \sum_{t=0}^{\tau-1} \mathbb{E}[\|\nabla {f}_i(\bm{w}^{(i)}_{k, t})\|^2] \leq \frac{\tau}{3}(8 \mathbb{E}[\|\nabla {f}_i(\bm{w}_k)\|^2] + \sigma^2).
\end{equation*}
\end{lemma}
\begin{proof}
\begin{flalign}
    \nonumber
    \mathbb{E}[\|\nabla {f}_i(\bm{w}^{(i)}_{k, t})\|^2] & = \mathbb{E}[\|\nabla {f}_i(\bm{w}^{(i)}_{k, t}) - \nabla {f}_i(\bm{w}_k) + \nabla {f}_i(\bm{w}_k)\|^2]
    \\
    \nonumber
    & \leq 2 \mathbb{E}[\|\nabla {f}_i(\bm{w}_k)\|^2] + 2 \mathbb{E}[\|\nabla {f}_i(\bm{w}^{(i)}_{k, t}) - \nabla {f}_i(\bm{w}_k)\|^2]
    \\
    \label{eq:feb28-7}
    & \leq 2 \mathbb{E}[\|\nabla {f}_i(\bm{w}_k)\|^2] + 2L^2 \mathbb{E}[\|\bm{w}^{(i)}_{k, t} - \bm{w}_k\|^2].
\end{flalign}
But:   
\begin{flalign}
    \label{eq:feb28-8}
    \mathbb{E}[\|\bm{w}_k - \bm{w}^{(i)}_{k, t}\|^2] = \mathbb{E}\Big[\Big\|\eta_k \sum_{t'=0}^{t-1} \nabla \widetilde{f}_i(\bm{w}^{(i)}_{k, t'}; \mathcal{B}^{(i)}_{k, t'})\Big\|^2\Big]
    & \leq \eta_k^2 t  \sum_{t'=0}^{t-1} \mathbb{E}\Big[\|\nabla \widetilde{f}_i(\bm{w}^{(i)}_{k, t'}; \mathcal{B}^{(i)}_{k, t'})\|^2\Big]
    \leq \eta_k^2 t \sum_{t'=0}^{t-1} (\mathbb{E}[\|\nabla {f}_i(\bm{w}^{(i)}_{k, t'})\|^2] + \sigma^2).
\end{flalign}    
Putting this back in \cref{eq:feb28-7}, we get:
\begin{flalign}
    \label{eq:feb28-9}
    \mathbb{E}[\|\nabla {f}_i(\bm{w}^{(i)}_{k, t})\|^2] \leq 
    2\mathbb{E}[\|\nabla {f}_i(\bm{w}_k)\|^2] + 2 \eta_k^2 L^2 t \sum_{t'=0}^{t-1} (\mathbb{E}[\|\nabla {f}_i(\bm{w}^{(i)}_{k, t'})\|^2] + \sigma^2).
\end{flalign}
Now summing up \cref{eq:feb28-9} for all $t \in \{0,\ldots,\tau-1\}$, we get:
\begin{flalign}
    \nonumber
    \sum_{t=0}^{\tau-1} \mathbb{E}[\|\nabla {f}_i(\bm{w}^{(i)}_{k, t})\|^2] & \leq 2\tau (\mathbb{E}[\|\nabla {f}_i(\bm{w}_k)\|^2]) + 2 \eta_k^2 L^2 \sum_{t=0}^{\tau-1} \tau \sum_{t'=0}^{t-1} (\mathbb{E}[\|\nabla {f}_i(\bm{w}^{(i)}_{k, t'})\|^2] + \sigma^2)
    \\
    \label{eq:feb28-10}
    & \leq 2\tau(\mathbb{E}[\|\nabla {f}_i(\bm{w}_k)\|^2) + \eta_k^2 L^2 \tau^2 \sum_{t=0}^{\tau-1} (\mathbb{E}[\|\nabla {f}_i(\bm{w}^{(i)}_{k, t})\|^2] + \sigma^2).
\end{flalign}
Let us set $\eta_k L E \leq 1/2$. Then:
\begin{equation*}
    \sum_{t=0}^{\tau-1} \mathbb{E}[\|\nabla {f}_i(\bm{w}^{(i)}_{k, t})\|^2] \leq 2 \tau (\mathbb{E}[\|\nabla {f}_i(\bm{w}_k)\|^2) + \frac{1}{4} \sum_{t=0}^{\tau-1} \mathbb{E}[\|\nabla {f}_i(\bm{w}^{(i)}_{k, t})\|^2] + \frac{\sigma^2 \tau}{4}.
\end{equation*}
Simplifying, we get:
\begin{equation}
    \label{eq:feb28-11}
    \sum_{t=0}^{\tau-1} \mathbb{E}[\|\nabla {f}_i(\bm{w}^{(i)}_{k, t})\|^2] \leq \frac{\tau}{3}(8 \mathbb{E}[\|\nabla {f}_i(\bm{w}_k)\|^2] + \sigma^2).
\end{equation}
\end{proof}

\end{document}